\documentclass[10pt]{article}
\usepackage[T1]{fontenc}

\usepackage{amsmath, dsfont, bm, mathrsfs,makecell}

\usepackage{tikz, tikzscale}

\usepackage{mathtools}

\usepackage{appendix}
\usepackage{comment}
\usepackage{wrapfig}

\usepackage{graphicx}
\usepackage{amssymb}
\usepackage{natbib}
\usepackage{amsthm}
\usepackage{footmisc}
\newtheorem{theorem}{Theorem}

\newtheorem{lemma}[theorem]{Lemma}
\newtheorem{proposition}[theorem]{Proposition}
\newtheorem{definition}{Definition}
\newtheorem{assumption}{Assumption}

\theoremstyle{remark}
\newtheorem{remark}{Remark}

\theoremstyle{remark}
\newtheorem*{remark*}{Remark}

\theoremstyle{remark}
\newtheorem{condition}{Condition}

\usepackage[verbose=true,letterpaper]{geometry}
\AtBeginDocument{
  \newgeometry{
    textheight=9in,
    textwidth=6.5in,
    top=1in,
    headheight=14pt,
    headsep=25pt,
    footskip=30pt
  }
}

\usepackage{hyperref}
\hypersetup{
    colorlinks=true,
    linkcolor=blue,
    filecolor=blue,      
    urlcolor=blue,
    citecolor=blue,
}
\usepackage{breakurl}

\title{On the Generalization Power of the Overfitted Three-Layer Neural Tangent Kernel Model}

\newcommand*\samethanks[1][\value{footnote}]{\footnotemark[#1]}
\author{Peizhong Ju\thanks{School of Electrical and Computer Engineering, 
 Purdue University. Email: \texttt{\{jup,linx\}@purdue.edu}} \and Xiaojun Lin\samethanks \and Ness B. Shroff\thanks{Department of ECE and CSE, The Ohio State University. Email: \texttt{shroff.11@osu.edu}}}
\date{}

\newcommand{\defeq}{\mathrel{\mathop:}=}

\newcommand{\XX}{\mathbf{X}}
\newcommand{\Xii}{\XX_i}

\newcommand{\xsmall}{\bm{x}}

\newcommand{\xdensity}{\mu}
\newcommand{\vdensity}{\lambda}

\newcommand{\hsmall}{\bm{h}}
\newcommand{\hx}{\hsmall_{\Vzero,\xsmall}}
\newcommand{\Ht}{\mathbf{H}}

\newcommand{\PH}{\mathbf{P}}

\newcommand{\esmall}{\bm{\epsilon}}

\newcommand{\hatf}{\hat{\f}}

\newcommand{\f}{f}

\newcommand{\F}{\mathbf{F}}

\newcommand{\fg}{\f_g}
\newcommand{\fl}{\hatf^{\ell_2}}

\newcommand{\learnableSet}{\mathcal{F}^{\ell_2}_{(3)}}
\newcommand{\learnableSetTwo}{\mathcal{F}^{\ell_2}_{(2)}}
\newcommand{\learnableSetTwoNB}{\mathcal{F}^{\ell_2}_{(2),\text{NLB}}}
\newcommand{\learnableSetTwoBB}{\mathcal{F}^{\ell_2}_{(2),\text{BB}}}
\newcommand{\learnableSetTwoBias}{\mathcal{F}^{\ell_2}_{(2),b}}
\newcommand{\learnableSetTwoBiasA}{\mathcal{F}^{\ell_2}_{(2),b_1}}
\newcommand{\learnableSetTwoBiasB}{\mathcal{F}^{\ell_2}_{(2),b_2}}

\newcommand{\hf}{h^{(3)}}
\newcommand{\hft}{h^{(2)}}

\newcommand{\hftb}{h^{(2)}_{b}}
\newcommand{\bOne}{b_1}
\newcommand{\bTwo}{b_2}
\newcommand{\hftbA}{h^{(2)}_{b_1}}
\newcommand{\hftbB}{h^{(2)}_{b_2}}

\newcommand{\Vzero}{\mathbf{V}}

\newcommand{\Vzeroj}{\Vzero[j]}

\newcommand{\capr}{\mathcal{B}^r}
\newcommand{\capsr}{\mathcal{B}_{\bm{v}_*}^r}

\newcommand{\capsrh}{\mathcal{B}^{\hat{r}}}

\newcommand{\sd}{\mathcal{S}^{d-1}}

\newcommand{\identity}{\mathbf{I}}
\DeclareMathOperator*{\expectation}{\mathsf{E}}
\DeclareMathOperator*{\vari}{\mathsf{Var}}

\DeclareMathOperator*{\argmin}{arg\,min}
\DeclareMathOperator*{\prob}{\mathsf{Pr}}

\newcommand{\eig}{\mathsf{eig}}

\makeatletter
\newcommand*\bigcdot{\mathpalette\bigcdot@{.5}}
\newcommand*\bigcdot@[2]{\mathbin{\vcenter{\hbox{\scalebox{#2}{$\m@th#1\bullet$}}}}}
\makeatother

\newcommand{\hxRF}{\hx^{\text{RF}}}

\newcommand{\w}{w}

\newcommand{\indicator}[1]{\mathsf{1}_{\{#1\}}}

\newcommand{\Wzero}{\mathbf{W}_{0}}
\newcommand{\Wone}{\mathbf{W}_{1}}
\newcommand{\hh}[1]{{\bm{h}^{\text{Three}}_{\Vzero,\Wzero,#1}}}
\newcommand{\hhx}{\hh{\xsmall}}
\newcommand{\hhXi}{\hh{\XX_i}}
\newcommand{\hhXj}{\hh{\XX_j}}
\newcommand{\hhz}{\hh{\bm{z}}}

\newcommand{\hzRF}{\bm{h}^{\text{RF}}_{\Vzero,\bm{z}}}
\newcommand{\hnxRF}{\bm{h}^{\text{RF}}_{\Vzero,-\xsmall}}
\newcommand{\hXiRF}{\bm{h}^{\text{RF}}_{\Vzero,\XX_i}}
\newcommand{\hXjRF}{\bm{h}^{\text{RF}}_{\Vzero,\XX_j}}
\newcommand{\hnXiRF}{\bm{h}^{\text{RF}}_{\Vzero,-\XX_i}}

\newcommand{\Chzhx}{\mathcal{C}_{\hxRF,\hzRF}^{\Wzero}}

\newcommand{\DW}{\Delta \mathbf{W}}
\newcommand{\DWl}{\DW^{\ell_2}}
\newcommand{\DWs}{\DW^*}
\newcommand{\HHt}{\Ht}
\newcommand{\pseudoGT}{f^g_{\Vzero,\Wzero}}

\newcommand{\minhRFnorm}{\beta}

\newcommand{\Drhat}{D_{\hat{r}}}

\newcommand{\RFangle}{\theta^{\text{RF}}}
\newcommand{\RFangleij}{\theta^{\text{RF}}_{i,j}}
\newcommand{\RFanglemin}{\RFangle_{\min}}
\newcommand{\RFnorm}{\beta}
\newcommand{\RFnormij}{\beta_{i,j}}

\newcommand{\Hinf}{\mathbf{H}^{\infty}}
\newcommand{\Hinfnormalized}{\tilde{\mathbf{H}}^{\infty}}
\newcommand{\spp}{\mathcal{S}^{p_1-1}}
\newcommand{\wdensity}{\gamma}
\newcommand{\FWVg}{\mathbf{F}_{\Vzero,\Wzero}^g}

\newcommand{\KThree}{K^{\text{Three}}}
\newcommand{\KTwo}{K^{\text{Two}}}
\newcommand{\KRF}{K^{\text{RF}}}

\newcommand{\Jnppdq}{J(n,p_1,p_2,d,q)}
\newcommand{\Qpd}{Q(p_1,d)}
\newcommand{\Cdq}{C(n,d,q)}

\newcommand{\xGeneralBias}{\xsmall_{b}}

\newcommand{\taylorCoeff}{u}
\newcommand{\harmonicCoeff}{c}

\newcommand{\qnTwo}{qn^2\sqrt{2p_1d}+q^2n^3d+\frac{qn^2p_1}{\sqrt{p_2}}}
\newcommand{\qnThree}{qn^2\sqrt{\frac{2d}{p_1}}+\frac{q^2n^3d}{p_1}+\frac{qn^2}{\sqrt{p_2}}}

\newcommand{\myCeil}[1]{\left\lceil #1 \right\rceil}

\begin{document}

\maketitle

\begin{abstract}
% In this paper, we study the generalization performance of overparameterized 3-layer NTK models. We show that, for a specific set of ground-truth functions (which we refer to as the ``learnable set''), the test error of the overfitted 3-layer NTK is upper bounded by an expression that decreases with the number of  neurons of the two hidden layers. Compared with existing results of 2-layer NTK models, our upper bound suggests that the descent behavior in the overparameterized region of 3-layer NTK shows both similarity and differences. Specifically, similar to 2-layer NTK, even when the number of neurons approaches infinity, the test error of 3-layer NTK decreases to a small limiting value, which further decreases with the number of training samples. However, different from 2-layer NTK where there exists only one hidden-layer, the 3-layer NTK involves interactions between two hidden-layers. Our upper bound reveals that, between the two hidden-layers, the test error descends faster with respect to the number of neurons in the second hidden-layer (the one closer to the output) than with respect to that in the first hidden-layer (the one closer to the input). In terms of the learnable set, we show that the learnable set of 3-layer NTK without bias is no smaller than that of 2-layer NTK models with various choices of bias in the neurons. However, in terms of the actual generalization performance, our results suggest that 3-layer NTK is much less sensitive to the choices of bias than 2-layer NTK, especially when the input dimension is large.

In this paper, we study the generalization performance of overparameterized 3-layer NTK models. We show that, for a specific set of ground-truth functions (which we refer to as the ``learnable set''), the test error of the overfitted 3-layer NTK is upper bounded by an expression that decreases with the number of  neurons of the two hidden layers. 
Different from 2-layer NTK where there exists only one hidden-layer, the 3-layer NTK involves interactions between two hidden-layers. 
Our upper bound reveals that, between the two hidden-layers, the test error descends faster with respect to the number of neurons in the second hidden-layer (the one closer to the output) than with respect to that in the first hidden-layer (the one closer to the input). 
We also show that the learnable set of 3-layer NTK without bias is no smaller than that of 2-layer NTK models with various choices of bias in the neurons. 
However, in terms of the actual generalization performance, our results suggest that 3-layer NTK is much less sensitive to the choices of bias than 2-layer NTK, especially when the input dimension is large.

\end{abstract}
\section{Introduction}

Neural tangent kernel (NTK) models \citep{jacot2018neural} have been recently studied as an important intermediate steps to understand the exceptional generalization power of overparameterized deep neural networks (DNNs). Deep neural networks (DNNs) usually have so many parameters that they can perfectly fit all train data, yet they still have good generalization performance \citep{zhang2016understanding,advani2017high}. This seems contradicting to the classical wisdom of ``bias-variance-tradeoff'' in the statistical machine learning methods \citep{bishop2006pattern, hastie2009elements, stein1956inadmissibility, james1992estimation,lecun1991second,tikhonov1943stability}. To understand this distinct behavior of DNNs, a resent line of work studies the so-called ``double-descent'' phenomenon, beginning with overfitted linear models. These results on linear models suggests that the test error indeed decreases again in the overparameterized region, as the model complexity increases beyond the number of samples \citep{belkin2018understand, belkin2019two, bartlett2019benign, hastie2019surprises,muthukumar2019harmless,ju2020overfitting,mei2019generalization}. However, these studies use linear models with simple features such as Gaussian or Fourier features, and hence they fail to capture the non-linearity in neural networks. In contrast, NTK models adopt features generated by non-linear activation functions (i.e., neurons of DNNs), and thus they can be viewed as an  intermediate step between simple linear models and DNNs. Along this line, the work in \cite{ju2021generalization} studies 2-layer NTK models, and shows that the 2-layer NTK model indeed exhibits better and different descent behavior in the overparameterized region, which might be closer to that of an actual neural network.

% However, in real-world machine learning tasks, neural networks with more than two layers are more commonly used, which makes us wonder the generalization performance of NTK models with more than two layers in the overparameterized region. 
Motivated by \cite{ju2021generalization}, it is of great interest to understand whether similar insights extend to deeper NTK models. In particular, in this paper we study NTK models with 3 layers.
Although both 2-layer and 3-layer NTK models share similar assumptions (e.g., trained weights do not change much from initialization, and features are linearized around the the initial state), their difference in structure leads to completely different feature formation. Compared with 2-layer NTK models that only contain one hidden-layer of neurons, 3-layer NTK models have two hidden-layers, which interact in more complex ways not observed in 2-layer NTK models.
% The interacting between two hidden-layers in 3-layer NTK models (which does not exist in 2-layer NTK) brings much complexity in analyzing its generalization performance. 
Specifically, let $p_1$ and $p_2$ denote the number of neurons in the two hidden layers. 
Then, the ultimate features of the 3-layer NTK models depend on both $p_1$ and $p_2$. This dependency leads to the following questions. 
First, the width of which layer is more important in governing the descent behavior, $p_1$ or $p_2$? Further, to get better descent behaviors, should $p_1$ and $p_2$ grow at the same speed, or should one of them grow faster than the other?
% Second, does the use of two hidden-layers make certain types of ground-truth function easier to learn than using 2-layer NTK models?
Second, do 3-layer NTK models have any performance advantage over 2-layer NTK models?

To answer these questions, in this paper we study the generalization performance of overfitted min-$\ell_2$-norm solutions for 3-layer NTK models where the middle layer is trained. For a set of learnable functions (which we refer to as the ``learnable set''), we provide an upper bound on the test error for finite values of $p_1$ and $p_2$. To the best of our knowledge, this upper bound is the first result that can reveal the dependency of the descent behavior on $p_1$ and $p_2$ separately. We then compare 3-layer NTK with 2-layer NTK with respect to the corresponding learnable set and the actual generalization performance. 
Our comparison reveals several important differences between 3-layer NTK and 2-layer NTK, in terms of the descent behavior, the size of learnable set, and the sensitivity of the generalization performance to the choice of bias of the neurons.
% The comparison shows that {\color{cyan}in general the 3-layer NTK model is at least as good as various 2-layer NTK models with respect to the types of learnable ground-truth functions. Besides, we show that 3-layer NTK can have better generalization performance than 2-layer NTK in certain situations.}

% {\color{red}[Explain our results and contribution in detail]} \TODO

\textbf{Analyzing the Generalization Error:} First, we show that the generalization error (denoted by the absolute value of the difference between the model output and the ground-truth for a test input) is upper bounded by the sum of several terms {on the order of} $O(1/\sqrt{n})$ ($n$ denotes the number of training data), $O(1/p_2)$ ($p_2$ denotes the number of neurons in the second hidden-layer),  $O(\sqrt[4]{\log p_1/p_1})$ ($p_1$ denotes the the number of neurons in the first hidden-layer), plus another term related to the magnitude of noise. 
Similar to 2-layer NTK \citep{arora2019fine,ju2021generalization,satpathi2021dynamics}, our upper bound suggests that when there are infinitely many neurons, the generalization error decreases with the number of samples $n$ at the speed of $\sqrt{n}$ and will approach zero when $n\to \infty$ in the noiseless situation. 
% Such $\sqrt{n}$ decreasing speed can also be found in 2-layer NTK (e.g., \cite{ju2021generalization}). 
Further, the noise term will not explode when the number of neurons goes to infinity, which is also similar to that for 2-layer NTK.
However, our upper bound also reveals new insights that are different from the results for 2-layer NTK. 
Specifically, our upper bound decreases with the number of neuron in the first hidden-layer $p_1$ at the speed of $\sqrt[4]{(\log p_1)/p_1}$, and decreases with the number of neurons in the second\footnote{In this paper, the first hidden-layer denotes the one closer to the input layer, while the second hidden-layer denotes the one closer to the output layer.} hidden-layer $p_2$ at the speed of $1/\sqrt{p_2}$. Such difference in the decreasing speed for our upper bound implies that the width of the second hidden-layer is more important for reducing the generalization error than the first hidden-layer. Further, our upper bounds hold regardless of how fast $p_1$ and $p_2$ increase relative to each other (e.g., they could increase at the same speed, or one could increase faster than the other).
% do not have strong correlation in terms of affecting the generalization performance. Thus, increasing number of neurons usually leads to better generalization performance (i.e., the descent in the overparameterized region), regardless of the way of increasing the number of neurons, such as increasing $p_1$ but fixing $p_2$, increasing $p_2$ but fixing $p_1$, or increasing $p_1$ and $p_2$ simultaneously.

\textbf{Characterizing the Learnable Set:} We then show that, even if we only train the middle-layer weights, the learnable set (i.e., the set of ground-truth functions for which the above upper bound holds) of the 3-layer NTK without bias contains all finite degree polynomials, which is strictly larger than that of the 2-layer NTK without bias and is at least as large as the 2-layer NTK with bias. Recently, \citet{geifman2020similarity,chen2020deep} show that when all layers are trained, 3-layer NTK leads to exactly the same reproducing kernel Hilbert space (RKHS) as 2-layer NTK with biased ReLU (although they assumed an infinite number of neurons, and did not characterize the descent behavior of the generalization error). 
Combining with their results, we can draw the conclusion that training only the middle-layer weights is at least as effective as training all layers in 3-layer NTK, in terms of  the size of the learnable set.

\textbf{Sensitivity to the Choices of Bias:} Even though a similar learnable set can be attained by 3-layer NTK (with or without bias) and 2-layer NTK (with bias), our results suggest that the actual generalization performance can still differ significantly in terms of the sensitivity to the choice of bias, especially when the input dimension $d$ is large. 
One type of bias setting commonly used in literature \citep{ghorbani2019linearized, satpathi2021dynamics} is that the bias has a similar magnitude as each element of the input vector, which we refer to as ``normal bias''. However, we show that such normal bias setting has a negative impact on the generalization error for overfitted 2-layer NTK when $d$ is large. To avoid this negative impact, it is important to use another type of bias setting where the bias has a similar magnitude as the norm of the whole input vector, which we refer to as ``balanced bias''. 
In contrast, for 3-layer NTK, different bias settings do not have obvious effect on the generalization performance.
% {\color{cyan}We emphasize that this performance advantage of 3-layer NTK primarily comes from training the middle layer, since training the bottom layer alone leads to a degenerated model which is the same as 2-layer NTK.}
In summary, compared with 2-layer NTK, the use of an extra non-linear layer in 3-layer NTK appears to significantly reduce the impact due to the choice of bias, and therefore makes the learning more robust.

Our work is related to the growing literature on the generalization performance of such shallow and fully-connected neural network.
However, most of these studies focus on 2-layer neural networks. Among them, they differ in which layer to train. For example, \citet{mei2019generalization,d2020double,mei2022generalization} consider the ``random feature'' (RF) model that only trains the top-layer weights and fixes the bottom-layer weights, while 2-layer NTK trains the bottom-layer weights.
In contrast, our work on 3-layer NTK neither trains the bottom-layer or top-layer weights. Instead, we train the middle-layer weights, since the middle-layer of a 3-layer model involves the interaction between two hidden-layers, which does not exist in 2-layer models.
The above studies of 2-layer network also differ in how the number of neurons/features $p$, the number of training samples $n$, and the input dimension $d$ grow.
\citet{mei2019generalization,mei2022generalization} study the generalization performance of the RF model where the number of neurons $p$, the number of training data $n$, and the input dimension grow proportionally to infinity. While \citet{ghorbani2021linearized} focuses on the approximation error (i.e., expressiveness) of both RF and NTK models, their analysis on generalization error is only on the limit $n$ or $p\to\infty$. All of these studies are quite different from ours with fixed $n$ and finite $p$.
% , RF model can only learn linear functions. 
Other works such as \cite{arora2019fine, satpathi2021dynamics, fiat2019decoupling} study the situation where the number of training samples $n$ is given and the number of neurons $p$ is larger than a threshold, which is closer to our setup.
However, these studies usually do not quantify how the generalization performance depends on the the number of neurons $p$. Specifically, they usually provide an upper bound on the generalization error when the number of neurons $p$ is greater than a threshold, while the upper bound itself does not depend on $p$. Thus, such an upper bound cannot explain the descent behavior of NTK models. 
The work in \cite{ju2021generalization} does study the decent behavior with respect to $p$, and is therefore the closet to our work. However, as we have explained earlier, there are crucial differences between 2 and 3 layers in both the descent behavior and the learnable set of ground-truth functions.
In addition to the above references, our work is also related to \citet{allen2019learning} (which studies NTK without overfitting) and \citet{ji2019polylogarithmic} (which studies classification by NTK). Their settings are however different from ours in that we consider overfitted solutions for regression.
% Besides all above, there are some other related settings, such as NTK without overfitting \citep{allen2019learning} and classification by NTK \citep{ji2019polylogarithmic}, which is different from our setting where we consider overfitted solutions for regression.
\section{System Model}\label{sec.model}
\begin{figure}[t!]
    \centering
    \includegraphics[width=3.3in]{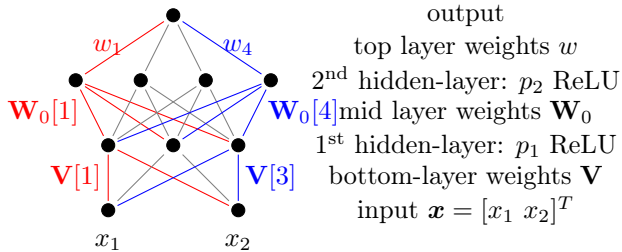}
    \caption{A fully-connected three-layer neural network where input dimension $d=2$, the number of neurons of the first hidden-layer $p_1=3$, and the number of neurons of the second hidden-layer $p_2=4$.}
    \label{fig.struct}
\end{figure}

Let $\f: \mathds{R}^d\mapsto\mathds{R}$ denote the ground-truth function. Let $(\XX_i,\ \f(\XX_i)+\esmall_i)$, $i=1,2,\cdots,n$ denote $n$ pieces of training data, where $\XX\in\mathds{R}^{n\times d}$ is the matrix, each column of which is the input of one training sample, $\esmall\in\mathds{R}^{n\times 1}$ denotes the noise in the output of training data. We define the training output vector generated by the ground-truth function as $\F(\XX)\defeq [\f(\XX_1)\ \f(\XX_2)\ \cdots\ \f(\XX_n)]^T\in \mathds{R}^n$. 

We use $\bm{a}[j]$ to denote the $j$-th part (sub-vector) of the vector $\bm{a}$. The part size depends on $\bm{a}$. Specifically, If $\bm{a}$ has $p_1$ elements, then each part has $1$ elements. If $\bm{a}$ has $dp_1$ elements, then each part has $d$ elements. If $\bm{a}$ has $p_1p_2$ elements, then each part has $p_1$ elements.

We consider a fully-connected 3-layer neural network as illustrated in Fig.~\ref{fig.struct}, which consists of normalized $d$-dimensional input $\xsmall\in \sd$ (a unit hyper-sphere), $p_1$ ReLUs (rectifier linear units $\max(\cdot, 0)$) at the first hidden-layer, $p_2$ ReLUs at the second hidden-layer, bottom-layer weights (between input and 1$^\text{st}$ hidden-layer) $\Vzero\in\mathds{R}^{(p_1 d)\times 1}$, middle-layer weights (between 1$^\text{st}$ hidden-layer and 2$^\text{nd}$ hidden-layer) $\Wzero\in\mathds{R}^{(p_1p_2)\times 1}$, and top-layer weights $\w\in \mathds{R}^{p_2\times 1}$ (between 2$^\text{nd}$ hidden-layer and output).

Let $\hxRF\in\mathds{R}^{p_1\times 1}$ denote the output of the first hidden-layer. We then have
\begin{align}
    \hxRF[j]\defeq (\xsmall^T\Vzero[j]) \indicator{\xsmall^T\Vzero[j]>0},\ j=1,2,\cdots,p_1\label{eq.def_hxRF}.
\end{align}
(We use the superscript ``RF'' because $\hxRF$ is indeed the feature vector of a random feature model \citep{mei2019generalization}.)
After training the middle-layer weights, $\Wzero$ changes to $\Wone\defeq \Wzero+\overline{\DW}$.
Then, the change of the output is
\begin{align*}
    &\sum_{k=1}^{p_2}\w_k \indicator{\Wone[k]^T\hxRF > 0}\Wone[k]^T\hxRF -\sum_{k=1}^{p_2} \w_k \indicator{\Wzero[k]^T \hxRF>0}\Wzero[k]^T \hxRF.
\end{align*}
The NTK model \citep{jacot2018neural} assumes that $\overline{\DW}$ is very small and thus the activation pattern does not change much. In other words, we can approximate $\indicator{\Wone[k]^T\hxRF > 0}$ by $\indicator{\Wzero[k]^T \hxRF>0}$.
Define $\DW\in \mathds{R}^{(p_1 p_2)\times 1}$ as $\DW[k]\defeq \w_k\cdot  \overline{\DW}[k]$, $k=1,2,\cdots,p_2$.
Define $\hhx\in\mathds{R}^{1\times(p_1 p_2)}$ such that
\begin{align}
\hhx[k] \defeq (\hxRF)^T\cdot \indicator{(\hxRF)^T\Wzero[k]>0},\label{eq.def_h3}
\end{align}
where $k=1,2,\cdots,p_2$.
Therefore, the change of the output can be approximated by
\begin{align*}
    \sum_{k=1}^{p_2} \w_k \indicator{\Wzero[k]^T \hxRF>0} \overline{\DW}[k]^T\hxRF =\hhx \DW.
\end{align*}

We thus obtain a linear model in $\DW$. We provide an illustration of the formation and structure of these vectors in Fig.~\ref{fig.struct_vector}, Appendix~\ref{app.formation_of_features} in Supplementary Material.
Define the design matrix $\HHt\in \mathds{R}^{n\times (p_1p_2)}$ such that its $i$-th row is $\HHt_i=\hhXi$. Notice that overfitted gradient descent on a linear model converges to the min $\ell_2$-norm solution, which is denoted by
\begin{align*}
    \DWl\defeq\argmin_{\bm{w}\in\mathds{R}^{(p_1p_2)\times 1}}\|\bm{w}\|_2\text{ subject to }\HHt \bm{w}=\F(\XX)+\esmall.
\end{align*}
When $\HHt$ is full row-rank (which holds with high probability under certain conditions), the trained model is then
\begin{align}\label{eq.def_minl2_solution}
    \fl(\xsmall)=\hhx\DWl=\hhx\HHt^T(\HHt\HHt^T)^{-1}(\F(\XX)+\esmall).
\end{align}
Notice that the trained model is determined by multiple random variables. In order to analyze the generalization performance of the trained model, we have to make assumptions on the distribution of those random variables.
Let $\xdensity(\cdot)$, $\vdensity(\cdot)$, and $\wdensity(\cdot)$ denote the probability density function of $\xsmall$, $\Vzero[j]$, and $\Wzero[k]$, respectively. For simplicity, we make the following assumption that all random variables follow uniform distribution.
\begin{assumption}\label{as.normalize}\label{as.uniform}
The input $\xsmall$ and the bottom-layer initial weights $\Vzeroj$'s ($j=1,2,\cdots,p_1$) are \emph{i.i.d.} and uniformly distributed in $\sd$. In other words, $\xdensity(\cdot)$ and $\vdensity(\cdot)$ are both $\mathsf{unif}(\sd)$. The middle-layer initial weights $\Wzero[k]$'s ($k=1,2,\cdots,p_2$) are \emph{i.i.d.} and uniformly distributed in $\spp$. In other words, $\wdensity(\cdot)$ is $\mathsf{unif}(\spp)$. The top-layer weights $\w$ are all non-zero\footnote{We do not need to specify the distribution of $\w$, since $\w$ is absorbed into the regressor $\DW$ by definition $\DW[k]\defeq \w_k\cdot  \overline{\DW}[k]$, $k=1,2,\cdots,p_2$.}.
\end{assumption}

\begin{remark}\label{remark.why_only_middle}
Readers may be curious why we only train the middle-layer weights. Part of the reason is technicality:
% We did not train all layers because,
if the bottom layer is also trained, the aggregate output of the first hidden-layer may have changed so much that the second hidden-layer's inputs and ReLU activation patterns change significantly from initialization, which may violate the NTK assumption.
The work in \citet{geifman2020similarity,chen2020deep} is not concerned about this difficulty, since they are mostly interested in the expressive power of the RKHS, assuming an infinite number of neurons. In contrast, we wish to capture the effect of finite width, and thus train only the middle layer to avoid this difficulty. More importantly, the middle-layer weights interact with both the first hidden layer and the second hidden layer, and are the major structural distinction compared with 2-layer NTK. 
% Training only the middle-layer weights will be helpful to isolate the effect of this ``distinct'' layer of 3-layer NTK. Indeed, as we will show later, training only the middle layer already achieves all the benefits compared with 2-layer NTK in terms of the size of the learnable set and the sensitivity to the choice of bias.
This setting thus helps us to answer the following interesting question: will training the middle-layer alone already achieve the same (potential) benefit as training all layers (especially given that the latter encounters more technical difficulty)?
\end{remark}

\section{Generalization Performance}

In this section, we will show our main results about the generalization performance of the aforementioned 3-layer NTK model for a specific set of functions. We first introduce a set of ground-truth functions that may be learnable and then provide a high-probability upper bound on the test error. We then discuss some useful implications of our upper bound.

\subsection{A set of ground-truth functions that may be learnable}\label{subsec.expressiveness}
We define kernel functions $\KRF$, $\KTwo$, and $\KThree:\ [-1,\ 1]\mapsto\mathds{R}$ as follows (whose meanings will be explained soon):
% that correspond to the RF model, the 2-layer NTK model, and the 3-layer NTK model, respectively:
\begin{align}
    &\KRF(a)\defeq \frac{\sqrt{1-a^2}+a\cdot \left(\pi-\arccos(a)\right)}{2d\pi},\label{eq.def_KRF}\\
    &\KTwo(a)\defeq a\cdot \frac{\pi-\arccos (a)}{2\pi},\label{eq.def_KTwo}\\
    &\KThree(a)\defeq\frac{\KTwo\left(2d\cdot \KRF(a)\right)}{2d}.\label{eq.def_KThree}
\end{align}
(Notice that $2d\cdot \KRF(a)\in[0,\ 1]$ for all $a\in[-1, 1]$ by Lemma~\ref{le.RF_kernel_bound} in Supplementary Material, Appendix~\ref{app.main}, and hence $\KThree(\cdot)$ is well defined.)
We define a set $\learnableSet$ of ground-truth functions based on those kernels:
\begin{definition}[learnable set of 3-layer NTK]
\begin{align}\label{eq.def_learnableSet}
    \learnableSet &\defeq \left\{\fg:\sd\mapsto\mathds{R}\ \Big|\ \fg(\xsmall)=\int_{\sd}\KThree(\xsmall^T\bm{z})g(\bm{z})d\xdensity(\bm{z}),\ \|g\|_\infty<\infty\right\},
\end{align}
where $\|g\|_{\infty}\defeq \sup_{\bm{z}\in\sd}|g(\bm{z})|$.
\end{definition}

To see why functions in $\learnableSet$ may be learnable, we can check what the learned result $\fl$ in Eq.~\eqref{eq.def_minl2_solution} should look like. 
When there are infinite number of neurons and there is no noise (i.e., $\esmall=\bm{0}$), what remains on the right-hand-side of Eq.~\eqref{eq.def_minl2_solution} can be viewed as the product of two terms, $\hhx\Ht^T$ and $(\Ht\Ht^T)^{-1}\F(\XX)$. For the first term $\hhx\Ht^T$, note that each row of $\Ht$ is given by $\hhXi$ for $i=1,2,\cdots,n$. Thus, when $p_1,p_2\to \infty$, the $i$-th element of $\hhx\Ht^T$, which is the inner product between $\hhx$ and $\hhXi$, converges in probability to $\KThree(\hhx(\hhXi)^T)$, which is exactly the kernel function of 3-layer NTK. By representing the second term $(\Ht\Ht^T)^{-1}\F(\XX)$ with a certain $g(\cdot)$, $\fl$ must then approach the form in Eq.~\eqref{eq.def_learnableSet}. (See Supplementary Material, Appendix~\ref{app.learnableSet_derivation} for details.)
Intuitively, $\KThree$ can be thought of as the composition of the kernels of each of the two layers, which are $\KTwo$ and $\KRF$ given in Eq.~\eqref{eq.def_KTwo} and Eq.~\eqref{eq.def_KRF}. Specifically, suppose that
% Here we briefly explain why $\KThree(\cdot)$ is the composition of $\KTwo(\cdot)$ and $\KRF(\cdot)$ (as shown in Eqs.~\eqref{eq.def_KRF}\eqref{eq.def_KTwo}\eqref{eq.def_KThree}). First,
we fix the output of the first hidden layer (i.e., $\hxRF$) and regard it as the input of a 2-layer NTK formed by the top two layers of the 3-layer neural network. By letting $p_2\to\infty$, we can show that the inner product between $\hhx$ and $\hhXi$ approaches $\KTwo((\hxRF)^T\hXiRF)$ (with necessary normalization of $\hxRF$ and $\hXiRF$), where $\KTwo$ is exactly the kernel of 2-layer NTK in \cite{ju2021generalization}. Second, when $p_1\to \infty$, we can show that $(\hxRF)^T\hXiRF$ approaches $\KRF(\xsmall^T \Xii)$, where $\KRF$ is exactly the kernel of the random-feature model \citep{mei2019generalization}. 
% Combining the results of the first step and the second step, we then show that $\KThree$ appears in $\fl$ when $p_1,p_2\to \infty$, which further explain the form of $\learnableSet$ in Eq.~\eqref{eq.def_learnableSet}.
% More details about the derivation can be found in Appendix~\ref{app.learnableSet_derivation}.
In summary, we expect that functions in $\learnableSet$ can be approximated by $\fl(\cdot)$. However, we note that the above deviation is only about the expressiveness of 3-layer NTK and it does not precisely reveal its generalization performance.

\subsection{An upper bound on the generalization error}
We now present the first main result of this paper, which is an upper bound that quantifies the relationship between the generalization performance and system parameters.
\begin{theorem}\label{th.main_simplified}
    For any ground-truth function $\f(x)=\fg(x)\in \learnableSet$, when $d$ is fixed and $p_1,p_2$ are much larger than $n$, (with high probability) we have
    \begin{align}\label{eq.simple_bound}
        &|\fl(\xsmall)-\f(\xsmall)|=\underbrace{O\left(\frac{\|g\|_\infty}{\sqrt{n}}\right)}_{\text{Term A}}+\bigg(\underbrace{O\left(\frac{\|g\|_1}{\sqrt{p_2}}\right)}_{\text{Term B}}+\nonumber\\
        &\underbrace{O\left(\|g\|_1\sqrt[4]{\frac{\log p_1}{p_1}}\right)}_{\text{Term C}}+\underbrace{\frac{\|\esmall\|_2}{\sqrt{n}}}_{\text{Term D}}\bigg)\cdot \underbrace{O\left(n^{\frac{2}{d-1}+\frac{1}{2}}\cdot \sqrt{\log n}\right)}_{\text{Term E}}.
    \end{align}
\end{theorem}
A more precise version of the upper bound and the condition of Theorem~\ref{th.main_simplified} as well as its derivation can be found in Supplementary Material, Appendix~\ref{app.full_theorem}.

As we can see, Eq.~\eqref{eq.simple_bound} captures how the test error depends on finite values of parameters $n$, $p_1$, $p_2$, $\|\esmall\|_2$, and $g$. Later in this section we will examine more closely how $n$, $p_1$, and $p_2$ affect the value of the upper bound. Regarding the dependency on $g$, Eq.~\eqref{eq.simple_bound} works as long as $\|g\|_1$ and $\|g\|_\infty$ are finite\footnote{Indeed, as long as $\|g\|_\infty<\infty$, then $\|g\|_1<\infty$. That is why we only include the condition $\|g\|_\infty<\infty$ in Eq.~\eqref{eq.def_learnableSet}. Notice that the assumption $\|g\|_\infty<\infty$ can be relaxed to $\|g\|_1<\infty$ by similar methods showing in \cite{ju2021generalization}. However, as shown in \cite{ju2021generalization}, such relaxation leads to a different upper bound with slower descent speed with respect to $n$.\label{footnote.finite_g_1}}. Intuitively, the norm of $g$ represents the complexity of the ground-truth function in $\learnableSet$. When the norm of $g$ is larger, then the right-hand-side of Eq.~\eqref{eq.simple_bound} becomes larger, which indicates that such ground-truth function is harder to learn. A simple example is that if we enlarge a ground truth function $\fg\in \learnableSet$ by 2 times (which means $g$ is 2 times larger), then since the model is linear, the test error $|\fl(\xsmall)-\f(\xsmall)|$ will become 2 times larger.
We will discuss more about which types of functions satisfy the condition of finite norm of $g$ in Section~\ref{sec.functions}.

Next, we will discuss some implications of this upper bound of 3-layer NTK. While some of them are similar to 2-layer NTK, others are significantly different, revealing the complexity due to having more layers.

\subsection{Interpretations similar to 2-layer NTK}
Based on the upper bound in Theorem~\ref{th.main}, we have the following insights for 3-layer NTK, which are similar to those for 2-layer NTK shown in \cite{ju2021generalization}. These similarities may reveal some intrinsic properties of the NTK models regardless of the number of layers.

\textbf{Zero test error with $n\to \infty$ in the ideal situation:} In the ideal situation where there are infinitely many neurons and no noise, the only remaining term in Eq.~\eqref{eq.simple_bound} is Term~A. Notice that Term~A decreases to zero as $n\to\infty$, which indicates that the generalization error decreases to zero when more training data are provided in the ideal situation. Term~A suggests that such decreasing speed is at least $1/\sqrt{n}$. Such result is consistent with that of 2-layer NTK, e.g., in \cite{arora2019fine}.

\subsection{Insights that are distinct compared with 2-layer NTK}

Compared with 2-layer NTK, an important difference for 3-layer NTK is that there are more than one hidden-layers. Therefore, the speed of the descent of 3-layer NTK involves the interaction between two hidden-layers.

\textbf{Descent with respect to number of neurons:} In Eq.~\eqref{eq.simple_bound}, Term B and Term C contain $p_1$ and $p_2$, respectively. For any given $n$ and noise level $\|\esmall\|_2$, Terms~A and D do not change, and Term~E decreases with $p_1$ and $p_2$. (More discussion about Term~E can be found in Supplementary Material, Appendix~\ref{app.noise_effect}, where we discuss the noise effect.) Therefore, by increasing $p_1$ and $p_2$, Term B and Term C keep decreasing. In summary, right-hand-side of Eq.~\eqref{eq.simple_bound} decreases as the number of neurons $p_1$ and $p_2$ increases, which validates the descent in the overparameterized region of 3-layer NTK.

\begin{figure}[t!]
\centering
\includegraphics[width=0.9\textwidth]{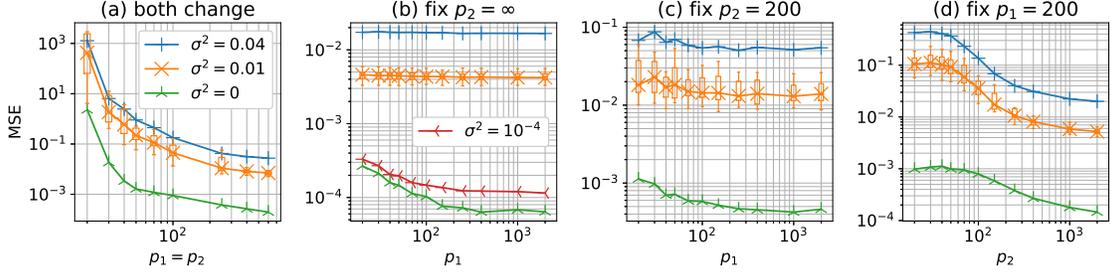}
\caption{Curves of MSE for 3-layer NTK (no-bias) with respect to $p_1$ or $p_2$ when there exists Gaussian noise whose mean is zero and the variance is $\sigma^2$. The ground-truth function is $\f(\xsmall)=\left(\bm{x}^T\bm{e}_1\right)^2+\left(\bm{x}^T\bm{e}_1\right)^3$ where $d=3$. Sample size is $n=200$. Every curve is the median of 20 random simulations.}\label{fig.noise_effect_p}
\end{figure}

\textbf{The second hidden-layer is more important:} As shown in Eq.~\eqref{eq.simple_bound}, $p_1$ and $p_2$ play different role in the descent of the generalization error. Comparing Term~B and Term~C of Eq.~\eqref{eq.simple_bound}, we can see that the upper bound of the test error $|\fl(\xsmall)-\f(\xsmall)|$ decreases at the speed of $\sqrt{p_2}$ and $\sqrt[4]{p_1/\log p_1}$, respectively.
We emphasize that this difference is not due to the number of weights/parameters contributed by the number of neurons in each hidden-layer of $p_1$ and $p_2$ \footnote{Specifically, the number of weights that get trained equals to $p_1p_2$ and the total number of weights for bottom, middle, and top layers equals to $d p_1+ p_1 p_2+ p_2$. In other words, the number of weights (either for trained ones or total) does not increase faster by increasing $p_2$ instead of $p_1$.}.
Instead, we conjecture that such difference in the speed of descent may be due to the different positions in this 3-layer neural network structure, where the second hidden-layer takes the trained middle-layer weights as its input (and thus utilizes the trained weights better than the first hidden-layer).
% Specifically, the first hidden-layer is closer to the input and is a transformation/preprocessing of the input. In contrast, the second hidden-layer is closer to the final output and takes the trained middle-layer weights as its input. 

% This difference in position suggests that the second hidden-layer utilizes the trained parameters better and has a more direct impact on the output, which is consistent with the different descent speed with respect to $p_1$ and $p_2$.

We use numerical results to illustrate the different role of $p_1$ and $p_2$ in reducing the generalization error. We fix $p_2=200$ and plot the MSE with respect to $p_1$ in Fig.~\ref{fig.noise_effect_p}(b). Although the test error decreases when $p_1$ increases, the decreasing speed is slow, especially for the noisy situation. Such a slow decreasing speed with $p_1$ remains even when $p_2$ is fixed to a much higher value. For example, in Fig.~\ref{fig.noise_effect_p}(b), we fix $p_2=\infty$, we still observe the similarly slow decreasing speed with $p_1$ as shown by Fig.~\ref{fig.noise_effect_p}(c). In contrast, the descent with respect to $p_2$ should be easier to observe and can reach a lower test MSE. In Fig.~\ref{fig.noise_effect_p}(d), we fix $p_1=200$ and increase $p_2$ (i.e., we exchange the values of $p_1$ and $p_2$ in Fig.~\ref{fig.noise_effect_p}(c)(d)). As we can see, all three curves in Fig.~\ref{fig.noise_effect_p}(d) have a more obvious descent and decrease to lower MSE compared with those in Fig.~\ref{fig.noise_effect_p}(c), which validates our conjecture that the second hidden-layer is more important.

Notice that our upper bound Eq.~\eqref{eq.simple_bound} also suggests a descent when both $p_1$ and $p_2$ increase simultaneously. We use simulation result by Fig.~\ref{fig.noise_effect_p}(a) to support this point. We fixed $n=200$ and let $p_1=p_2$ increase simultaneously. The ground-truth model in this figure is $\f(\xsmall)=(\xsmall^T\bm{e})^2+(\xsmall^T\bm{e})^3$ where $d=3$. The green, orange, and blue curves denote the situations of $\sigma^2=0$ (no noise), $\sigma^2=0.01$, and $\sigma^2=0.04$, respectively. Every point in this figure is the median of 20 simulation runs. We also provide the box plot\footnote{From bottom to top, the five horizontal lines of each marker of a box plot represent the minimum (excluding outliers), first quartile (25\%), median (50\%), third quartile (75\%), and maximum (excluding outliers), respectively. See \citep{mcgill1978variations} for more details.}  of the situation of $\sigma^2=0.01$ (correspond to the orange curve). It is obvious that all three curves descend, which verifies that the generalization error of the overfitted 3-layer NTK model decreases when $p_1$ and $p_2$ increases simultaneously at the same speed. By observing the box plot for the situation $\sigma^2=0.01$ (the orange curve), we also notice that when $p_1=p_2$ becomes large, the variance becomes small. This is because all initial weights are \emph{i.i.d.} random and a large number of weights may reduce the variance of the model due to the law of large numbers. Our upper bound in Theorem~\ref{th.main} also suggests such reduced variance as the probability in Theorem~\ref{th.main} increases as $p_1$ increases.
\section{Types of Ground-Truth Functions}\label{sec.functions}

% Although deeper neural network usually has stronger capability of fitting the data, it is unclear whether deeper neural network must have a better generalization performance. 
Are 3-layer (i.e., deeper) networks better than 2-layer networks in any way for generalization performance? In the last section, we have seen that both 3-layer NTK and 2-layer NTK can achieve zero test error when $n\to\infty$ in the ideal noiseless situation,  when the ground-truth functions are in their respective learnable set\footnote{We illustrate the generalization performance of ground-truth functions outside the learnable set in Supplementary Material, Appendix~\ref{sec.outside_learnable_set}.}. 
A natural question is then to compare the learnable sets between these two models, and to compare the generalization performance when the ground-truth function belongs to both learnable sets.
% We wonder whether 3-layer NTK shows some advantages in the size of learnable set, or whether 3-layer NTK shows a better generalization performance for some specific settings (e.g., a type of ground-truth functions, large or small input dimension). 
In this section, we provide some answers by studying various types of ground-truth functions and their effects on the generalization performance.

% Although we describe the set of learnable ground-truth function $\learnableSet$ by Eq.~\eqref{eq.def_learnableSet}, it is not very clear what specific types of functions in or outside of $\learnableSet$. Besides, it is hard to compared Eq.~\eqref{eq.def_learnableSet} with similar results in 2-layer NTK. To solve those problems, in this section we investigate different types of ground-truth functions and compare their generalization performance.

\subsection{Size of the learnable set}
For a 2-layer NTK, as shown in \cite{ju2021generalization}, when no bias is used in ReLU, the corresponding learnable set $\learnableSetTwo$ contains all even polynomials and linear functions, but does not contain other odd polynomials. In order to learn both even and odd polynomials, it is critical that bias is added to ReLU \citep{satpathi2021dynamics,ju2021generalization}. 
In contrast, we prove the following result:

\begin{proposition}\label{prop.concise_learnable_set}
 $\learnableSet$ (with unbiased ReLU, middle layer being trained) already contains all polynomials with finite degree (i.e., including both even and odd polynomials).
Further, the learnable set $\learnableSet$ of 3-layer NTK is strictly larger than that of the 2-layer NTK with unbiased ReLU, and is at least as large as that of the 2-layer NTK with biased ReLU.
\end{proposition}

This independence to bias shown by Proposition~\ref{prop.concise_learnable_set} can be seen as one performance advantage of 3-layer NTK compared to 2-layer NTK. Details (including more precise statement) about this result is in Supplementary Material, Appendix~\ref{app.details_learnable_set}. Notice that \citet{geifman2020similarity,chen2020deep} show that when training all layers, 3-layer NTK leads to the same RKHS as 2-layer NTK with biased ReLU. However, it is unclear whether training one layer is already sufficient for achieving the same RKHS as training all layers. Our result in Proposition~\ref{prop.concise_learnable_set} answers  this question positively, i.e., only training the middle layer has already achieved all benefits of training all layers in terms of the size of the learnable set. (In other words, training all three layers will not expand the learnable set over training only the middle layer.)

% In the following two subsections, we will emphasize more on the relationship between the generalization performance and the learnable set.}

\subsection{Different bias setting and high input dimension}\label{sec.high_d}

% Whether a function belongs to the learnable sets is of course important, but it is not the only thing we care about. 
Even when a ground-truth function belongs to both $\learnableSetTwoBias$ and $\learnableSet$, their generalization performance may still exhibit some differences.
% Although the learnable sets for different bias settings (i.e., different value of $b$ in Eq.~\eqref{eq.def_bias_input}) in 2-layer NTK are exactly the same, they may have different generalization performance in specific setup. 
In this subsection, we will show that when the input dimension $d$ is high, some specific choice of bias of the 2-layer NTK has better generalization performance than others. In contrast, the 3-layer NTK is less sensitive to different bias settings.

Notice that adding bias to each ReLU in 2-layer NTK is equivalent to appending a constant to $\xsmall$ while still using ReLU without bias. Specifically, the input vector for biased 2-layer NTK is
\begin{align}\label{eq.def_bias_input}
    \xGeneralBias\defeq \left[\begin{smallmatrix}
    \sqrt{1-b^2}\cdot \xsmall\\
    b
    \end{smallmatrix}\right]\in \mathds{R}^{d+1},
\end{align}
where $b\in (0,1)$ denotes the initial bias.
We also normalize the first $d$ elements of $\xGeneralBias$ by $\sqrt{1-b^2}$ in Eq.~\eqref{eq.def_bias_input} to make sure that $\|\xGeneralBias\|_2=1$. Under this biased setting, the 2-layer NTK model has the learnable set
$\learnableSetTwoBias\defeq \{\int_{\sd}\KTwo\left((1-b^2)\xsmall^T\bm{z}+b^2\right)g(\bm{z})d\xdensity(\bm{z}),\ \|g\|_\infty<\infty\}$.

\begin{table*}
\begin{center}
\footnotesize
\begin{tabular}{ |c|c|c| } 
 \hline
 \textbf{Model} & \textbf{Learnable functions set} & \textbf{Category} \\ \hline
 3-layer NTK, no-bias & $\learnableSet=\left\{\int_{\sd}\KThree(\xsmall^T\bm{z})g(\bm{z})d\xdensity(\bm{z})\right\}$ & (i) \\ \hline
 2-layer NTK, no-bias & $\learnableSetTwo=\left\{\int_{\sd}\KTwo(\xsmall^T\bm{z})g(\bm{z})d\xdensity(\bm{z})\right\}$ & (ii) \\ 
%  3-layer NTK, normal-bias & $\int_{\sd}\KThree(\frac{d}{d+1}\xsmall^T\bm{z}+\frac{1}{d+1})g(\bm{z})d\xdensity(\bm{z})$ & yes \\ 
 2-layer NTK, normal-bias & $\learnableSetTwoNB=\left\{\int_{\sd}\KTwo(\frac{d}{d+1}\xsmall^T\bm{z}+\frac{1}{d+1})g(\bm{z})d\xdensity(\bm{z})\right\}$ & (i) \\ 
%  3-layer NTK, balanced-bias & $\int_{\sd}\KThree(\frac{1}{2}\xsmall^T\bm{z}+\frac{1}{2})g(\bm{z})d\xdensity(\bm{z})$ & yes \\ 
 2-layer NTK, balanced-bias & $\learnableSetTwoBB=\left\{\int_{\sd}\KTwo(\frac{1}{2}\xsmall^T\bm{z}+\frac{1}{2})g(\bm{z})d\xdensity(\bm{z})\right\}$ & (i) \\ 
 \hline
\end{tabular}
\caption{Learnable functions for different NTK models. Category: (i) can learn both even- and odd-power polynomials; (ii) cannot learn other odd-power polynomials except linear functions. (We omit the condition $\|g\|_\infty<\infty$ in the expression of learnable sets to save space.)}\label{tb.learnable}
\end{center}
\end{table*}

A common setup for the initial magnitude of the bias of each ReLU is to use a value that is close or equal to the average magnitude of each element of input $\xsmall$, e.g., \cite{satpathi2021dynamics, ghorbani2019linearized}. Specifically, we let $b=\frac{1}{\sqrt{d+1}}$ in Eq.~\eqref{eq.def_bias_input}, and denote the corresponding learnable set by $\learnableSetTwoNB$. We refer to this setting as the ``normal-bias'' setting. Alternatively, the initial magnitude of the bias can be chosen to be close or equal to $\|\xsmall\|_2$. Specifically, we let $b=\frac{1}{\sqrt{2}}$ in Eq.~\eqref{eq.def_bias_input} and denote the corresponding learnable set by $\learnableSetTwoBB$. We refer to this second setting as ``balanced-bias''. The specific expression of $\learnableSetTwo$, $\learnableSetTwoNB$, and $\learnableSetTwoBB$ can be derived by using similar methods shown in Section~\ref{subsec.expressiveness} (results are listed in Table~\ref{tb.learnable}).

We now discuss how the two different bias settings could affect the generalization performance when $d$ is large. For 2-layer NTK under the normal-bias setting, the kernel is $\KTwo(\frac{d}{d+1}\xsmall^T\bm{z}+\frac{1}{d+1})$. Although it contains both even and odd power polynomials, we notice that when $d$ increases, $\KTwo$ approaches its no-bias counterpart $\KTwo(\xsmall^T\bm{z})$, which only contains even power polynomials and linear term. 
Thus, we conjecture that, by increasing $d$, the generalization performance of 2-layer NTK with normal-bias will deteriorate for those ground-truth functions inside $\learnableSetTwoNB$ but far away from $\learnableSetTwo$ (e.g., odd-degree non-linear polynomials). In contrast, for 2-layer NTK under the balanced-bias setting, the kernel is $\KTwo(\frac{1}{2}\xsmall^T\bm{z}+\frac{1}{2})$, which does not change with $d$. Therefore, we expect that such deterioration should not happen.
Note that in 3-layer NTK, although normal-bias setting still approaches no-bias setting when $d$
increases, there does not exist such performance deterioration, because $\learnableSet$ (the learnable set of 3-layer NTK without bias) already contains both even and odd power polynomials. These insights will be verified by numerical results below.

\begin{figure}[t!]
    \centering
    \includegraphics[width=0.95\textwidth]{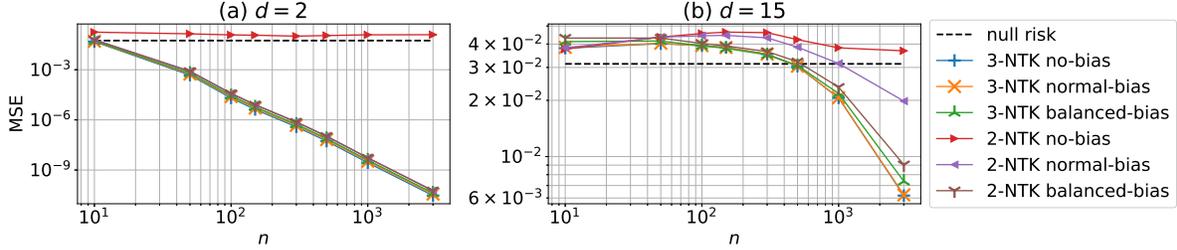}
    \caption{Comparison of test MSE with respect to $n$ between different NTK models when the number of neurons is infinite and without noise. The ground-truth function is $\f(\xsmall)=\frac{d+2}{3}\left(\bm{x}^T\bm{e}_1\right)^3-\bm{x}^T\bm{e}_1$. Every curve is the average of 10 random simulations.}
    \label{fig.type_B}
\end{figure}

We now use simulation results in Fig.~\ref{fig.type_B} to validate the conjecture that 3-layer NTK models are less sensitive to different bias settings than 2-layer NTK models. We let the ground-truth function be $\f(\xsmall)=\frac{d+2}{3}\left(\bm{x}^T\bm{e}_1\right)^3-\bm{x}^T\bm{e}_1$, which is orthogonal to $\learnableSetTwo$. In Fig.~\ref{fig.type_B}(a) when $d=2$, all settings have similar performance except 2-layer NTK without bias, whose test error is always above the null risk. In Fig.~\ref{fig.type_B}(b) when $d=15$, the purple curve of 2-layer NTK with normal bias gets closer to the red curve of 2-layer NTK without bias (and thus the generalization performance becomes worse), while other curves are still close to each other.
This validates our conjecture that 3-layer NTK models are less sensitive to different bias settings than 2-layer NTK models. Further simulations can be found in Appendix~\ref{app.more_simu_for_high_d}.

\section{Conclusion}
In this paper, we studied the generalization performance of overfitted 3-layer NTK models. Compared with 2-layer NTK models, 3-layer NTK is less sensitive to different bias settings. Further, training only the middle layer can get most of the performance advantage of 3-layer NTK, in terms of the learnable set.
Possible future directions include: (i) studying whether training other layers will get the same the benefit as training the middle layer; (ii) approximating the actual neural network where the learned result is far away from the initial state; (iii) investigating deeper network as well as other structures such as convolutional neural network (CNN) and recursive neural network (RNN).
% First, NTK models require that all weights and activation pattern of neurons do not change much after training, which may only happen when the neural network is shallow and wide. It will be interesting to build a new model that can approximate the actual neural network where the learned result is far away from the initial state. 
% Second, besides fully-connected structure, it will be worthwhile to investigate other structures such as convolutional neural network (CNN) and recursive neural network (RNN).

\bibliographystyle{plainnat}
\bibliography{ref}

\clearpage
\newpage

\appendix
\section{Additional Figures}

\subsection{Formation of features}\label{app.formation_of_features}
\begin{figure}[t!]
    \centering
    \includegraphics[width=\textwidth]{figs/h_vector_structure.tikz}
    \caption{Formation and structure of $\hhx$ and $\hxRF$.}
    \label{fig.struct_vector}
\end{figure}
In Fig.~\ref{fig.struct_vector}, we illustrate the formation and the structure of the features shown in Section~\ref{sec.model}.

\subsection{About the conjecture in Section~\ref{sec.high_d}}\label{app.more_simu_for_high_d}
We provide some additional simulation results (in addition to Fig.~\ref{fig.type_B}) to validate our conjecture in Section~\ref{sec.high_d} that 2-layer NTK is more sensitive to different bias settings, especially when $d$ is large. Note that in Fig.~\ref{fig.type_B}, we only consider one type of ground-truth functions that contains only odd-power polynomials. Here, we also examine other types of ground-truth functions.

\begin{figure}[ht!]
\centering
\includegraphics[width=0.8\textwidth]{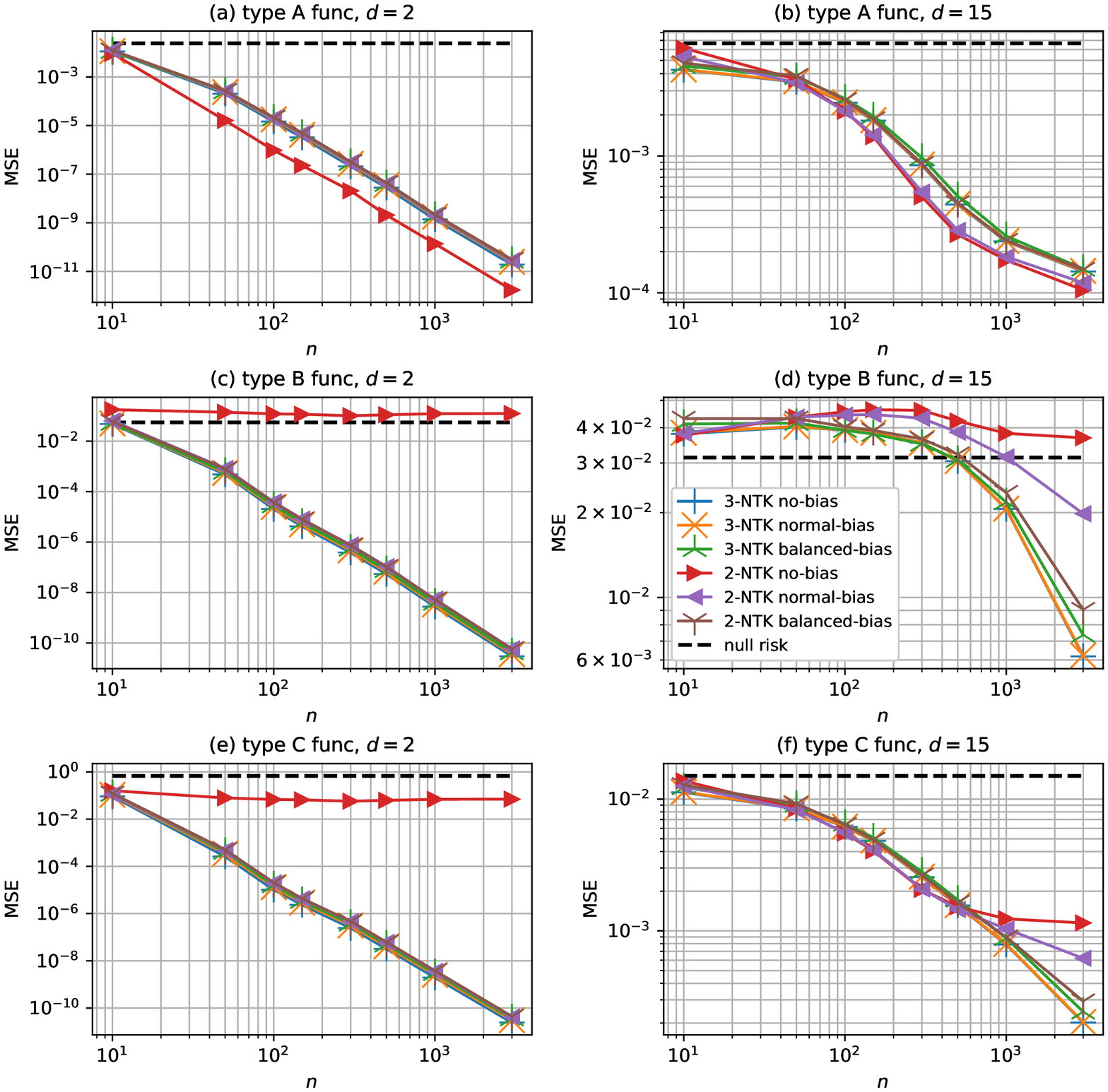}
\caption{Curves of MSE with respect to $n$ for 2-layer and 3-layer NTK models when $p,p_1,p_2\to \infty$ and $\esmall=\bm{0}$. Let $\bm{e}_1=[1\ 0\ 0\ \cdots\ 0]^T\in\mathds{R}^d$. Type A function is $\f(\xsmall)=\left(\bm{x}^T\bm{e}_1\right)^4-\left(\bm{x}^T\bm{e}_1\right)^2$. Type B function is $\f(\xsmall)=\frac{d+2}{3}\left(\bm{x}^T\bm{e}_1\right)^3-\bm{x}^T\bm{e}_1$. Type C function is $\f(\xsmall)=\left(\bm{x}^T\bm{e}_1\right)^2+\left(\bm{x}^T\bm{e}_1\right)^3$. Every curve is the average of 10  random simulations.}\label{fig.changing_n_all}
\end{figure}

Similar to Fig.~\ref{fig.type_B}, in Fig.~\ref{fig.changing_n_all}, we consider the ideal case where there are infinite number of neurons. We plot curves of MSE with respect to $n$ when $p,p_1,p_2\to\infty$. The simulation setup is similar to Fig.~\ref{fig.type_B}, but here we consider more types of ground-truth functions (whose exact forms are given in the caption of Fig.~\ref{fig.changing_n_all}). In sub-figures~(a)(b), Type~A function corresponds to even-power polynomials. We can see that all curves are close to each other in both low-dimensional case ($d=2$) and high-dimensional case ($d=15$). This is because the 2-layer NTK without bias can learn even-power polynomials. In other words, in high-dimensional cases, although the performance of the normal-bias setting approaches that of the no-bias setting, it does not hurt the generalization performance because the no-bias setting can already learn the Type~A function. Sub-figures~(c)(d) are exactly the same as Fig.~\ref{fig.type_B}, which uses the Type~B ground-truth function corresponds to odd-power polynomials. Sub-figures~(e)(f) adopt the Type~C ground-truth function that contains both odd-power and even-power polynomials. The generalization performance shown by sub-figures (e)(f) is between that in sub-figures (a)(b) and that in sub-figures (c)(d). This is expected because Type~C functions can be viewed as a mix of Type~A and Type~B functions.

We also consider the situation of finite number of neurons. In Fig.~\ref{fig.change_p1_all}, we fix the number of training data and let the x-axis be $p$ (for 2-layer NTK) or $p_1$ (for 3-layer NTK with fixed $p_2=100$). The setup of Fig.~\ref{fig.change_p2_all} is similar to the setup of Fig.~\ref{fig.change_p1_all} except that for 3-layer NTK we fix $p_1=100$ and change $p_2$. Both in Fig.~\ref{fig.change_p1_all} and Fig.~\ref{fig.change_p2_all}, when $d$ is large and the ground-truth function is Type~B (i.e., sub-figure (d)), we can see that the curve of 2-layer NTK with normal-bias (the purple curve marked by $\blacktriangleleft$) is closer to the curve of 2-layer NTK without bias (the red curve marked by $\blacktriangleright$). This validates our conjecture in Section~\ref{sec.high_d} that 2-layer NTK is more sensitive to different bias settings, especially when $d$ is large.

\begin{figure}[ht!]
\centering
\includegraphics[width=0.8\textwidth]{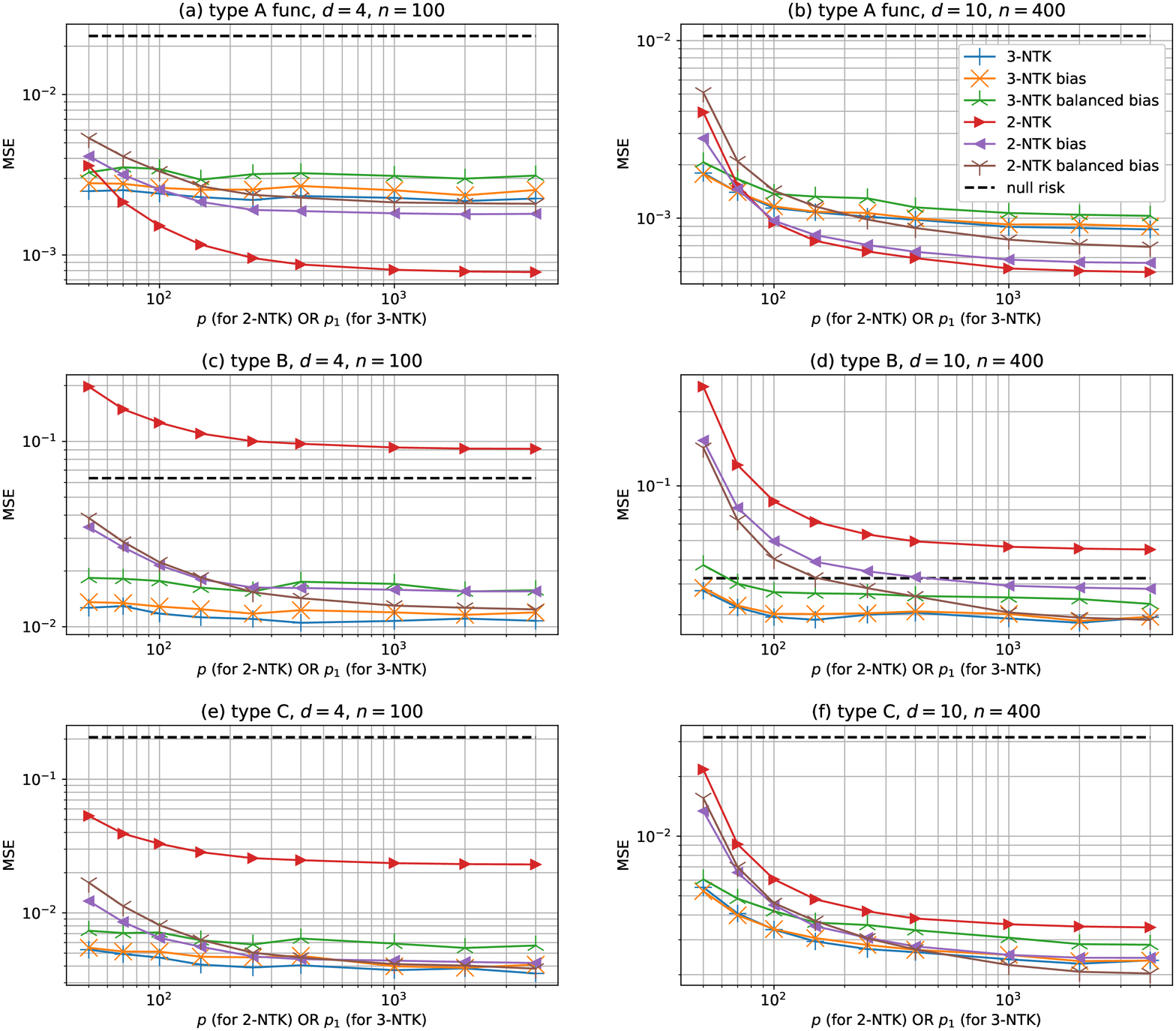}
\caption{Curves of MSE with respect to $p$ (for 2-layer NTK) or $p_1$ (for 3-layer NTK where $p_2=100$). Other settings such as types of ground-truth functions and $\esmall=0$ are the same as those in Fig.~\ref{fig.changing_n_all}. Every curve is the average of 10 random simulations.}\label{fig.change_p1_all}
\end{figure}

\begin{figure}[ht!]
\centering
\includegraphics[width=0.8\textwidth]{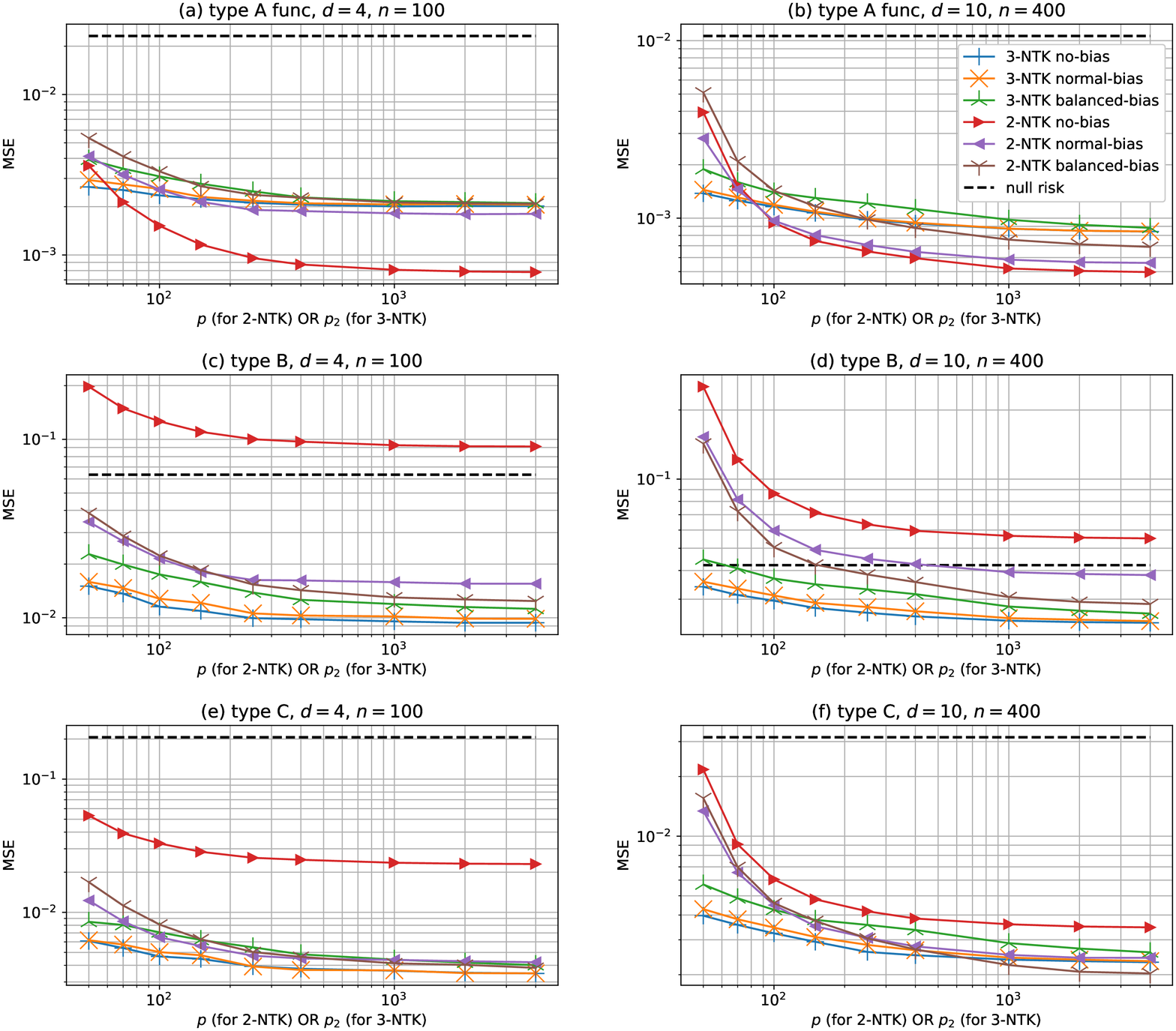}
\caption{Curves of MSE with respect to $p$ (for 2-layer NTK) or $p_2$ (for 3-layer NTK where $p_1=100$). Other settings such as types of ground-truth functions and $\esmall=0$ are the same as those in Fig.~\ref{fig.changing_n_all}. Every curve is the average of 10 random simulations.}\label{fig.change_p2_all}
\end{figure}

\section{Derivation of the learnable set \texorpdfstring{$\learnableSet$}{learnable set}}\label{app.learnableSet_derivation}
For the derivation of the learnable set, we assume that the noise $\esmall$ is zero in Eq.~\eqref{eq.def_minl2_solution}.
We first rewrite Eq.~\eqref{eq.def_minl2_solution} as the sum of terms contributed by each sample. Recall that $\HHt^T=[\HHt_1^T\ \cdots\ \HHt_n^T]\in \mathds{R}^{(p_1p_2)\times n}$ where $\HHt_i\in \mathds{R}^{1\times(p_1p_2)}$, $i=1,2,\cdots,n$. Thus, we have $\hhx\HHt^T=\sum_{i=1}^n \left(\hhx\HHt_i^T\right)\bm{e}_i^T$ where $\bm{e}_i\in \mathds{R}^n$ denotes the $i$-th standard basis (i.e., the $i$-th element is $1$ while all other elements are $0$). Thus, we have
\begin{align}\label{eq.temp_101505}
    \fl(\xsmall)=\hhx\HHt^T(\HHt\HHt^T)^{-1}\F(\XX) =\sum_{i=1}^n \left(\frac{1}{p_2}\hhx \HHt_i^T\right)p_2 \bm{e}_i^T (\HHt\HHt^T)^{-1}\F(\XX).
\end{align}
For any $\bm{a},\bm{b}\in\mathds{R}^{p_1}$, we define a set
\begin{align}\label{eq.def_Cw}
    \mathcal{C}^{\Wzero}_{\bm{a},\bm{b}}\defeq\left\{k\in\{1,2,\cdots,p_2\}\ \big|\ \bm{a}^T\Wzero[k]>0,\ \bm{b}^T\Wzero[k]>0\right\},
\end{align}
whose cardinality is given by
\begin{align*}
    \left|\mathcal{C}^{\Wzero}_{\bm{a},\bm{b}}\right|=\sum_{k=1}^{p_2}\indicator{\bm{a}^T\Wzero[k]>0,\ \bm{b}^T\Wzero[k]>0}.
\end{align*}
Intuitively, $\mathcal{C}^{\Wzero}_{\bm{a},\bm{b}}$ denotes the indices of the ReLU in the second hidden-layer that are activated both when the output of the first layer is $\bm{a}$ and when the output of the first layer is $\bm{b}$.
Then, by Eq.~\eqref{eq.def_h3}, we have
\begin{align}\label{eq.temp_101501}
    \frac{1}{p_2} \hhx \HHt_i^T= (\hxRF)^T \hXiRF \frac{\left|\mathcal{C}^{\Wzero}_{\bm{\hxRF},\bm{\hXiRF}}\right|}{p_2}.
\end{align}
By Assumption~\ref{as.uniform}, which gives the distribution of $\Wzero$, we can calculate the limiting value of Eq.~\eqref{eq.temp_101501} when there are an infinite number of neurons in the second hidden-layer. Specifically, since
\begin{align}\label{eq.temp_101502}
    \frac{\left|\mathcal{C}^{\Wzero}_{\bm{\hxRF},\bm{\hXiRF}}\right|}{p_2}\stackrel{\text{P}}{\rightarrow}\frac{\pi-\arccos\left(\frac{(\hxRF)^T\hXiRF}{\left\|\hxRF\right\|_2\cdot\left\|\hXiRF\right\|_2}\right)}{2\pi},\text{ as }p_2\to \infty,
\end{align}
where $\stackrel{\text{P}}{\rightarrow}$ denotes convergence in probability, we have
\begin{align*}
    \text{Eq.~\eqref{eq.temp_101501}}\stackrel{\text{P}}{\rightarrow} \left\|\hxRF\right\|_2\cdot\left\|\hXiRF\right\|_2\cdot \KTwo\left(\frac{(\hxRF)^T\hXiRF}{\left\|\hxRF\right\|_2\cdot\left\|\hXiRF\right\|_2}\right),\text{ as }p_2\to \infty.
\end{align*}
Note that $\KTwo$ is known to be the kernel of 2-layer NTK \citep{ju2021generalization}. It is natural that $\KTwo$ appears here, since we can regard the output of the first hidden-layer as the input of a 2-layer network consisting of the top- and middle-layer of the 3-layer network.

To further simplify the above expression, it remains to calculate $(\hxRF)^T\hXiRF$. Similar to the derivation above, when the first hidden-layer has an infinite number of neurons, we have
\begin{align}\label{eq.temp_101503}
    (\hxRF)^T\hXiRF \stackrel{\text{P}}{\rightarrow}\KRF(\xsmall^T\XX_i), \text{ as }p_1\to \infty.
\end{align}
(Eq.~\eqref{eq.temp_101502} and Eq.~\eqref{eq.temp_101503} can be derived from integration over a hyper-sphere, which is shown in Lemma~\ref{le.NTK2_kernel} and Lemma~\ref{le.RF_kernel} in Appendix~\ref{app.integral}, respectively.)
Note that $\KRF$ is also the kernel of the random-feature model \citep{mei2019generalization}. It is natural that $\KRF$ appears here since Eq.~\eqref{eq.temp_101503} represents the situation that the bottom-layer has infinite width, which also appears in a random feature model.
Notice that $\KRF(\XX_i^T\XX_i)=\KRF(\xsmall^T\xsmall)=\KRF(1)=\frac{1}{2d}$. Thus, we have
\begin{align}\label{eq.temp_101504}
    \left\|\hxRF\right\|_2\cdot\left\|\hXiRF\right\|_2\stackrel{\text{P}}{\rightarrow} \frac{1}{2d}, \text{ as }p_1\to \infty.
\end{align}
Plugging Eq.~\eqref{eq.temp_101502}\eqref{eq.temp_101503}\eqref{eq.temp_101504} into Eq.~\eqref{eq.temp_101501} and recalling Eq.~\eqref{eq.def_KThree}, we thus have
\begin{align*}
    \frac{1}{p_2}\hhx \HHt_i^T \stackrel{\text{P}}{\rightarrow}\KThree(\xsmall^T\XX_i), \text{ as }p_1,p_2\to \infty.
\end{align*}
If we let
\begin{align*}
    g(\bm{z})=\sum_{i=1}^n p_2 \bm{e}_i^T (\HHt\HHt^T)^{-1}\F(\XX) \delta_{\bm{\XX_i}}(\bm{z}),
\end{align*}
(where $\delta_{\bm{z}_0}(\bm{z})$ denotes a $\delta$-function, i.e., it has zero value for all $\bm{z}\in\sd\setminus \{\bm{z}_0\}$, but its $L_1$-norm is $\|\delta_{\bm{z}_0}\|_1\defeq \int_{\sd}\delta_{\bm{z}_0}(\bm{z})d \xdensity(\bm{z})=1$), then as $p_1$ and $p_2\to \infty$, Eq.~\eqref{eq.temp_101505} approaches $\int_{\sd} \KThree(\xsmall^T\bm{z})g(\bm{z})d\xdensity(\bm{z})$, which is in the same form\footnote{We acknowledge that the form here is still not exactly the same  as $\learnableSet$ because the $\delta$-function does not satisfy the constrain of finite $\|g\|_\infty$. Nonetheless, $\learnableSet$ can be relaxed to allow finite $\|g\|_1$, which then includes the $\delta$-function. See footnote~\ref{footnote.finite_g_1} on Page~\pageref{footnote.finite_g_1}.} as functions in $\learnableSet$.

\section{A Precise Form of the Upper Bound in Theorem~\ref{th.main_simplified}}\label{app.full_theorem}

% \subsection{Extra notations and a condition that \texorpdfstring{$p_1$}{p1} is sufficiently large}
We first introduce some extra notations and a condition about large $p_1$ that will be used later in our upper bound of the generalization error. 
Define
\begin{align}
    &\Cdq\defeq \frac{\pi-1}{4\pi}\min\left\{\frac{1}{2},\ \left(\frac{(d-1)^2}{8d}\right)^{\frac{1}{d-1}}\left(qn\right)^{-\frac{4}{d-1}}\right\},\label{eq.def_Cndq}\\
    &\Jnppdq\defeq \frac{1}{16\pi d}\sqrt{\frac{\Cdq}{\log(4n)}}-\left(\qnThree\right),\label{eq.def_J}\\
    &\Qpd\defeq 8d\sqrt{\frac{2(d+1)\log(p_1+1)}{p_1}}.\label{eq.def_Q}
\end{align}

\begin{condition}\label{cond.large_p1}
(Given $n$, $d$, and $q>0$) $p_1$ and $p_2$ are sufficiently large such that $9d\cdot \Qpd \leq 1$, $p_1\geq \left(\frac{10dnq\sqrt{2d}}{\Cdq}\right)^2$, and $\Jnppdq>0$.
\end{condition}

\begin{theorem}\label{th.main}
    Given a ground-truth function $\f(x)=\fg(x)\in \learnableSet$, for any $q>0$, under Condition~\ref{cond.large_p1}, we must have
    \begin{align*}
        \prob_{\Vzero,\Wzero,\XX}\bigg\{|\fl(\xsmall)-\f(\xsmall)|\leq \frac{q\|g\|_{\infty}}{\sqrt{n}}+\frac{q\|g\|_1}{\sqrt{p_2}}+ \sqrt{\frac{\Qpd}{d}} \|g\|_1\\
        +\frac{\sqrt{n}\|g\|_1\left(\frac{q}{\sqrt{p_2}}+\sqrt{\frac{\Qpd}{d}}\right)+\|\esmall\|_2}{\sqrt{\Jnppdq}}\bigg\}\geq 1-\frac{10}{q^2}-\frac{2d^2}{(p_1+1)e^{d+1}}.
    \end{align*}
    \end{theorem}
    A proof sketch can be found in Appendix~\ref{sec.sketch_proof}.
    To better illustrate the meaning of this upper bound, we provide a simplification in Theorem~\ref{th.main_simplified} when $p_1$ and $p_2$ are much larger than $n$.
    If we view $d$ as a constant, we have $\Cdq=O(n^{-\frac{4}{d-1}})$. When $p_1$ and $p_2$ are much larger than $n$, we have $\frac{\sqrt{n}}{\sqrt{\Jnppdq}}=O\left(n^{\frac{2}{d-1}+\frac{1}{2}}\cdot\sqrt{\log(n)}\right)$ and $\sqrt{\frac{\Qpd}{d}}=O\left(\sqrt[4]{\frac{\log p_1}{p_1}}\right)$.
    Therefore, when $d$ is fixed and when both $p_1$ and $p_2$ are much larger than $n$, Theorem~\ref{th.main} can be simplified to Eq.~\eqref{eq.simple_bound} (with high probability).
    
    (In the above reduction to Eq.~\eqref{eq.simple_bound}, we ignore the stand-alone term $\frac{q\|g\|_1}{\sqrt{p_2}}$ appeared in Theorem~\ref{th.main}, since it is much smaller than the product of Term~B and Term~E. Similarly, we ignore the stand-alone term $\sqrt{\frac{\Qpd}{d}} \|g\|_1$, since it is much smaller than the product of Term~C and Term~E.)
\section{Noise Effect}\label{app.noise_effect}

\begin{figure}[t!]
\centering
\includegraphics[width=5in]{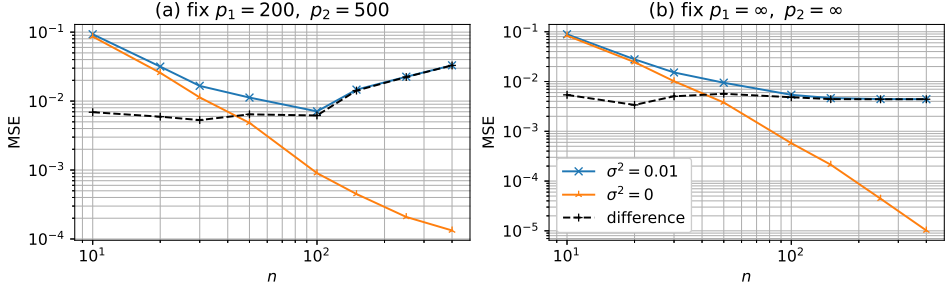}
\caption{Noise effect on the test MSE of 3-layer NTK (no-bias) with respect to $n$. The noise follows \emph{i.i.d.} Gaussian distribution with zero mean and variance $\sigma^2$. The ground-truth function is $\f(\xsmall)=\left(\bm{x}^T\bm{e}_1\right)^2+\left(\bm{x}^T\bm{e}_1\right)^3$ where $d=3$. Every curve is the average of 20 random simulations.}\label{fig.noise_effect_n}
\end{figure}

% In Eq.~\eqref{eq.simple_bound}, Term~D denotes the noise power in the training data, and Term~E (whose original form in Theorem~\ref{th.main} is $\frac{\sqrt{n}}{\sqrt{\Jnppdq}}$) denotes the factor of how the learned model enlarge the noise effect on the generalization error.

Before we present the proof of Theorem~\ref{th.main} in Appendix~\ref{sec.sketch_proof}, we elaborate on how Theorem~\ref{th.main} reveals the impact of noise on the generalization error. Note that
in Eq.~\eqref{eq.simple_bound}, Term~D denotes the average noise power in each training sample, and Term~E denotes the extra multiplication factor with which the noise impacts the generalization error. As we see in Theorem~\ref{th.main} in Appendix~\ref{app.full_theorem}, the precise form of Term~E is $\frac{\sqrt{n}}{\sqrt{\Jnppdq}}$. Therefore, we will refer to the multiplication of $\|\esmall\|_2/\sqrt{n}$ with this factor as the ``noise effect''. Note that although the precise form $\frac{\sqrt{n}}{\sqrt{\Jnppdq}}$ of this factor in Theorem~\ref{th.main} decreases with respect to both $p_1$ and $p_2$ by Eq.~\eqref{eq.def_J}, when $p_1$ and $p_2$ are much larger than $n$, it can be simplified to Term~E, which does not depend on $p_1$ and $p_2$.

% \begin{remark}\label{remark.term_E}
% As we mentioned in the process of deriving Eq.~\eqref{eq.simple_bound}, Term~E is a simplification of $\frac{\sqrt{n}}{\sqrt{\Jnppdq}}$ when $p_1$ and $p_2$ are much larger than $n$. Nonetheless, even without that simplification, the original form $\frac{\sqrt{n}}{\sqrt{\Jnppdq}}$ in Theorem~\ref{th.main} decreases with respect to both $p_1$ and $p_2$ by Eq.~\eqref{eq.def_J}.
% \end{remark}

In the following, we will analyze the relationship between the noise effect and various system parameters.
First, we are interested in know how the numbers of neurons in two hidden-layers $p_1$ and $p_2$ impact the noise effect. 
Since Term~E is an approximation when $p_1$ and $p_2$ are large and it does not contain $p_1$ or $p_2$, we conjecture that even when $p_1$ and $p_2$ are extremely large (e.g., $p_1,p_2\to\infty$), the noise effect will neither grow dramatically nor go to zero. Further, when $p_1$ and $p_2$ are not so large,
by Eq.~\eqref{eq.def_J}, we know that the precise form of Term~E in Theorem~\ref{th.main} decreases when $p_1$ and $p_2$ increase, which suggests that the noise will likely contribute more to the test error when the number of neurons is small. An intuitive explanation of such effect is that when $p_1$ and $p_2$ are small, the randomness of the initial weights brings some extra ``pseudo-noise'' to the model, and thus the generalization performance deteriorates.

Second, we are interested in how the noise effect changes with the number of training data $n$. We notice that Term~E increases with $n$ at a speed faster than $\sqrt{n}$. However, since it is only an upper bound, the actual noise effect may grow much slower than $\sqrt{n}$. Therefore, precisely estimate the relationship between $n$ and the noise effect of NTK model can be a interesting future research direction.

We then use simulation to study the noise effect and compare them with the implications derived from our upper bound. In Fig.~\ref{fig.noise_effect_n}, we plot the curves of the test MSE with respect to $n$. The noise follows \emph{i.i.d.} Gaussian $\mathcal{N}(0, \sigma^2)$. The blue curve denotes the situation where the noise level is $\sigma^2=0.01$. The orange curve denotes the noiseless situation. The noise effect (the value of the gap between the blue and the orange curves) is denoted by the dashed black curve. 
As we can see, when $n$ is large, the value of the black curve in Fig.~\ref{fig.noise_effect_n}(a) (fix $p_1=200,p_2=500$) is higher than that in Fig.~\ref{fig.noise_effect_n}(b) ($p_1,p_2\to\infty$), which validates our conjecture that the noise contributes more to the test error when the number of neurons is small. Further, Fig.~\ref{fig.noise_effect_n}(b) shows that an infinite number of neurons does not make the noise effect diminish or explode for every $n$, which also confirms our previous analysis on the relationship between the number of neurons and the noise effect. We also notice that the black curve in Fig.~\ref{fig.noise_effect_n}(b) (where $p_1,p_2\to\infty$) does not increase significantly with $n$, which suggests that our estimate on how fast Term~E increases with $n$ could be further improved.
\section{Proof of Theorem~\ref{th.main}}\label{sec.sketch_proof}
Recall that Theorem~\ref{th.main} is the precise form of Theorem~\ref{th.main_simplified}, and is stated in Appendix~\ref{app.full_theorem}. To prove Theorem~\ref{th.main}, we follow the line of analysis in \cite{ju2021generalization}.
We first study the class of the ground-truth functions that can be learned when weights $\Vzero$ and $\Wzero$ are fixed and there is no noise. We refer to them as \emph{pseudo ground-truth} in the following definition, to differentiate them with the set $\learnableSet$ of learnable functions for random $\Vzero$ and $\Wzero$. 
\begin{definition}\label{def.fv}
Given $\Vzero$ and $\Wzero$, for any learnable ground-truth function $\fg\in\learnableSet$ with the corresponding function $g(\cdot)$, define the corresponding \textbf{pseudo ground-truth} as
\begin{align}\label{eq.def_pseudoGT}
    \pseudoGT(\xsmall)\defeq&\int_{\sd}\frac{(\hhz)^T\hhx}{p_1p_2} g(\bm{z})d\xdensity(\bm{z})\nonumber\\
    =&\int_{\sd} (\hzRF)^T\hxRF \frac{\left|\Chzhx\right|}{p_1p_2}g(\bm{z})d\xdensity(\bm{z}).
\end{align}
The last equality of Eq.~\eqref{eq.def_pseudoGT} follows from Eq.~\eqref{eq.def_h3} and Eq.~\eqref{eq.def_Cw}. (The form of Eq.~\eqref{eq.def_pseudoGT} can be derived using the similar process shown in Appendix~\ref{app.learnableSet_derivation}.)
\end{definition}

We prove Theorem~\ref{th.main} in several steps as follows.

\textbf{Step 1: use pseudo ground-truth as a ``intermediary''
.}

Recall the definition of pseudo ground-truth $\pseudoGT(\cdot)$ in Eq.~\eqref{eq.def_pseudoGT}. We define
\begin{align}\label{eq.def_F_pseudoGT}
    \FWVg(\XX)\defeq [\pseudoGT(\XX_1)\ \pseudoGT(\XX_2)\ \cdots\ \pseudoGT(\XX_n)]^T\in\mathds{R}^n.
\end{align}
We then have
\begin{align}
    \fl(\xsmall)=&\hhx\HHt^T(\HHt\HHt^T)^{-1}\left(\F(\XX)+\esmall\right)\text{ (by Eq.~\eqref{eq.def_minl2_solution})}\nonumber\\
    =&\hhx\HHt^T(\HHt\HHt^T)^{-1}\FWVg(\XX) + \hhx\HHt^T(\HHt\HHt^T)^{-1}\left(\F(\XX)-\FWVg(\XX)\right)\nonumber\\
    &+\hhx\HHt^T(\HHt\HHt^T)^{-1}\esmall.\label{eq.temp_041501}
\end{align}
Thus, we have
\begin{align}
    &|\fl(\xsmall)-\f(\xsmall)|\nonumber\\
    =& |\fl(\xsmall)-\pseudoGT(\xsmall)+\pseudoGT(\xsmall)-\f(\xsmall)|\nonumber\\
    =&|\hhx\HHt^T(\HHt\HHt^T)^{-1}\FWVg(\XX)-\pseudoGT(\xsmall) \nonumber\\
    &+ \hhx\HHt^T(\HHt\HHt^T)^{-1}\left(\F(\XX)-\FWVg(\XX)\right)\nonumber\\
    &+\pseudoGT(\xsmall)-\f(\xsmall)+\hhx\HHt^T(\HHt\HHt^T)^{-1}\esmall|\text{ (by Eq.~\eqref{eq.temp_041501})}\nonumber\\
    \leq & \underbrace{|\hhx\HHt^T(\HHt\HHt^T)^{-1}\FWVg(\XX)-\pseudoGT(\xsmall)|}_{\text{term A}}\nonumber \\
    & + \underbrace{|\hhx\HHt^T(\HHt\HHt^T)^{-1}\left(\F(\XX)-\FWVg(\XX)\right)|}_{\text{term B}}\nonumber\\
    &+\underbrace{|\pseudoGT(\xsmall)-\f(\xsmall)|}_{\text{term C}}+\underbrace{\left|\hhx\HHt^T(\HHt\HHt^T)^{-1}\esmall\right|}_{\text{term D}}.\label{eq.term_ABC}
\end{align}
In Eq.~\eqref{eq.term_ABC}, term A denotes the test error when using the pseudo ground-truth function, term B denotes the effect of replacing the original ground-truth function by the pseudo ground-truth function in the training samples, term C denotes the difference between  the original ground-truth function and the pseudo ground-truth function on the test input, term D denotes the noise effect. Next, we bound these terms one by one.

% The final step is to allow $\Vzero$ and $\Wzero$ to be random and consider noise effect. 

\textbf{Step 2: estimate term A.}

The following proposition gives an upper bound of the test error when the data model is based on the pseudo ground-truth and the NTK model uses exactly the same $\Vzero$ and $\Wzero$.

\begin{proposition}\label{prop.pseudoGT}
Assume fixed $\Vzero$ and $\Wzero$, (thus $p_1$, $p_2$ and $d$ are also fixed), and there is no noise. If the ground-truth function is $\f=\pseudoGT$ in Definition~\ref{def.fv} and $\|g\|_\infty<\infty$, then for any $\xsmall\in\sd$ and $q>0$, we must have
\begin{align*}
    \prob_{\XX}\left\{|\pseudoGT(\xsmall)-\fl(\xsmall)|\geq \frac{q\|g\|_{\infty}}{\sqrt{n}}\right\}\leq \frac{1}{q^2}.
\end{align*}
\end{proposition}
The proof of Proposition~\ref{prop.pseudoGT} is in Appendix~\ref{app.pseudoGT}. 
Proposition~\ref{prop.pseudoGT} captures how the test error decreases with the number of training samples $n$, if the data model is based on a pseudo ground-truth function with the same $\Vzero$ and $\Wzero$ as the NTK.
The result shown in  Proposition~\ref{prop.pseudoGT} contributes to Term~A in Eq.~\eqref{eq.simple_bound}. Here we sketch the proof of Proposition~\ref{prop.pseudoGT}. By Eq.~\eqref{eq.def_pseudoGT}, we can find a vector $\DWs\in \mathds{R}^{(p_1 p_2)\times 1}$ and rewrite $\pseudoGT$ as $\pseudoGT=\hhx \DWs$. The specific form of $\DWs$ can be found in Eq.~\eqref{eq.def_DWs} in Appendix~\ref{app.pseudoGT}. Then, by Eq.~\eqref{eq.def_minl2_solution}, we can see that the learned model is $\fl(\xsmall)=\hhx \mathbf{P}\DWs$ where $\mathbf{P}\defeq \HHt ^T(\HHt \HHt^T)^{-1}\HHt$ (an orthogonal projection to the row-space of $\HHt$). Thus, we have $|\pseudoGT(\xsmall)-\fl(\xsmall)|=|\hhx(\mathbf{P}-\mathbf{I})\DWs|\leq \|\hhx\|_2\cdot \left\|(\mathbf{P}-\mathbf{I})\DWs\right\|_2$. Further, it is easy to show that $\|\hhx\|_2\leq \sqrt{p_1 p_2}$. It then remains to estimate $\left\|(\mathbf{P}-\mathbf{I})\DWs\right\|_2$, which is upper bounded by $\min_{\bm{a}\in \mathds{R}^{n}}\left\|\DWs-\HHt^T \bm{a}\right\|_2$ (because $\mathbf{P}$ is an orthogonal projection). The rest of proof focuses on how to choose a vector $\bm{a}$ to make $\left\|\DWs-\HHt^T \bm{a}\right\|_2$ as small as possible.
Notice that although the similar method of choosing a suitable $\bm{a}$ is also used for 2-layer NTK \citep{ju2021generalization}, the process of estimating $\left\|\DWs-\HHt^T \bm{a}\right\|_2$ is much more complicated than that in \cite{ju2021generalization}, since the feature vector $\hhx$ of 3-layer NTK involves non-linear activation for two hidden-layers (instead of one in 2-layer NTK).

With Proposition~\ref{prop.pseudoGT}, now we are ready to estimate term~A of Eq.~\eqref{eq.term_ABC}. We have
\begin{align*}
    &\prob_{\XX,\Vzero,\Wzero}\left\{\text{term A}\geq \frac{q\|g\|_{\infty}}{\sqrt{n}}\right\}\\
    =&\int_{\mathds{R}^{dp_1}}\int_{\mathds{R}^{p_1p_2}}\prob_{\XX}\left\{\text{term A}\geq \frac{q\|g\|_{\infty}}{\sqrt{n}}\right\}d\Lambda_{w}(\Wzero)d\Lambda_{v}(\Vzero)\\
    &\text{ (where $\Lambda_w(\cdot)$ and $\Lambda_v(\cdot)$ are probability distribution of $\Wzero$ and $\Vzero$, respectively)}\\
    \leq &\frac{1}{q^2} \text{ (by Proposition~\ref{prop.pseudoGT})}.
\end{align*}

\textbf{Step 3: estimate term C.}

Intuitively, when $p_1$ and $p_2$ become larger, the randomness brought by $\Vzero$ and $\Wzero$ in the pseudo ground-truth $\pseudoGT$ will be ``averaged out'', and thus $\pseudoGT(\xsmall)$ will approach $\f(\xsmall)$ (i.e., term~C will approaches zero). The following proposition makes this statement rigorous.
\begin{proposition}\label{prop.term_C}
    For any $\xsmall\in \sd$ and $q>0$, we must have 
    \begin{align*}
    \prob_{\Vzero,\Wzero}\left\{\left|\pseudoGT(\xsmall)-\f(\xsmall)\right|\geq \frac{q\|g\|_1}{\sqrt{p_2}}+ \sqrt{\frac{\Qpd}{d}} \|g\|_1\right\}\leq \frac{d^2}{(p_1+1)e^{d+1}} +\frac{1}{q^2}.
    \end{align*}
\end{proposition}
The proof of Proposition~\ref{prop.term_C} is in Appendix~\ref{app.proof_term_C}. Note that as $p_1$ and $p_2$ increase, both $\frac{q\|g\|_1}{\sqrt{p_2}}$ and $\sqrt{\frac{\Qpd}{d}} \|g\|_1$ decrease, which implies that the pseudo ground-truth $\pseudoGT(\xsmall)$ approaches $\f(\xsmall)$ with high probability. The above result thus directly bounds term C.

\textbf{Step 4: estimate terms B and D.}

We note that both terms~B and D are of a similar form. Specifically, we can view the difference between $\F(\XX)$ and $\FWVg(\XX)$ as a special type of ``noise'' due to random $\Vzero$ and $\Wzero$ (which will approaches zero when $p_1,p_2\to \infty$). Then, both terms~B and D are the multiplication of $\hhx\HHt^T(\HHt \HHt^T)^{-1}$ with the noise (either real noise or the special ``noise'' above). Further, we can show that the magnitude of $\hhx\HHt^T(\HHt \HHt^T)^{-1}$ can be upper bounded by a quantity inversely proportional to the minimum eigenvalue of $\HHt \HHt^T$. Thus, a key step of the proof is to estimate the minimum eigenvalue of $\HHt \HHt^T$. We prove the following proposition about $\min\eig(\HHt\HHt^T)$ in Appendix~\ref{app.minEig}.

\begin{proposition}\label{prop.min_eig}
Recall the definition of $J(\cdot)$ in Eq.~\eqref{eq.def_J}. For any $q>0$, when Condition~\ref{cond.large_p1} is satisfied,
we must have
\begin{align*}
    \prob_{\XX,\Vzero,\Wzero}\left\{\frac{1}{p_1p_2}\min\eig(\HHt\HHt^T)\leq \Jnppdq\right\}\leq \frac{7}{q^2}.
\end{align*}
\end{proposition}

Using Proposition~\ref{prop.min_eig}, we can then bound terms~B and D by the following Proposition~\ref{prop.term_DB}.
\begin{proposition}\label{prop.term_DB}
    For any $q>0$, when Condition~\ref{cond.large_p1} is satisfied, we must have
    \begin{align*}
        &\prob_{\XX,\Vzero,\Wzero}\left\{\text{term D}+\text{term B of Eq.~\eqref{eq.term_ABC}}\geq \frac{\sqrt{n}\|g\|_1\left(\frac{q}{\sqrt{p_2}}+\sqrt{\frac{\Qpd}{d}}\right)+\|\esmall\|_2}{\sqrt{\Jnppdq}} \right\}\\
        &\leq \frac{d^2}{(p_1+1)e^{d+1}} +\frac{8}{q^2}.
    \end{align*}
\end{proposition}
Note that $\sqrt{n}\|g\|_1\left(\frac{q}{\sqrt{p_2}}+\sqrt{\frac{\Qpd}{d}}\right)$ and $\|\esmall\|_2$ correspond to the magnitude of the special ``noise'' $(\F(\XX)-\FWVg(\XX))$ (which can be bounded just like Proposition~\ref{prop.term_C}) and the real noise $\esmall$, respectively.
The proof of Proposition~\ref{prop.term_DB} is in Appendix~\ref{app.proof_term_DB}.

% Specifically, random $\Vzero$ and $\Wzero$ lead to the difference between the $\F(\XX)$ and $\FWVg(\XX)$, which can be viewed as a special type of ``noise'' and corresponds to Term~B in Eq.~\eqref{eq.simple_bound} {\color{cyan}(similar to the intuition explained in Step 3, such ``noise'' will approaches zero when $p_1,p_2\to \infty$)}. Together with the noise in the training data (corresponding to Term~D in Eq.~\eqref{eq.simple_bound}), all those ``noise'' is multiplied by $\hhx\HHt^T(\HHt \HHt^T)$ in the trained model Eq.~\eqref{eq.def_minl2_solution} (corresponding to Term~E in Eq.~\eqref{eq.simple_bound}), whose effect can be bounded by $\min\eig(\HHt\HHt^T)$. 

Plugging the results in Steps 2, 3, and 4 into Eq.~\eqref{eq.term_ABC}, the result of Theorem~\ref{th.main} thus follows. Appendices~\ref{app.pseudoGT} to \ref{app.main} will prove the above propositions, after we present some supporting lemmas in Appendix~\ref{app.support_lemmas}.

% Combining the noise effect and Proposition~\ref{prop.pseudoGT}, we can then prove Theorem~\ref{th.main}. See Appendix~\ref{app.main} for details.
\section{Useful Notations and Lemmas}\label{app.support_lemmas}
We first collect some useful notations and lemmas, which will be used in the proofs of propositions appeared in Appendix~\ref{sec.sketch_proof}, as well as the analysis of learnable functions.
% \section{Extra Notations}\label{app.extra_notation}
% {\color{cyan}We introduce some extra notations which will be used later}.
Let $I_{\cdot}(\cdot,\cdot)$ denote the regularized incomplete beta function \cite{dutka1981incomplete}. 
Let $B(\cdot, \cdot)$ denote the beta function \cite{chaudhry1997extension}. Specifically,
\begin{align}
    &B(x,y)\defeq\int_0^1 t^{x-1}(1-t)^{y-1}dt,\label{eq.def_betaFunction}\\
    &I_x(a,b)\defeq\frac{\int_0^x t^{a-1}(1-t)^{b-1}dt}{B(a,b)}.\label{eq.def_reg_incomplete_beta}
\end{align}
Define a cap on a unit hyper-sphere $\sd$ as the intersection of $\sd$ with an open ball in $\mathds{R}^d$ centered at $\bm{v}_*$ with radius $r$, i.e.,
\begin{align}\label{eq.temp_100101}
    \capsr\defeq \left\{\bm{v}\in\sd\ |\ \|\bm{v}-\bm{v}_*\|_2< r\right\}.
\end{align}
\begin{remark}\label{remark.no_star}
For ease of exposition, we will sometimes neglect the subscript $\bm{v}_*$ of $\capsr$ and use $\capr$ instead, when the quantity that we are estimating only depends on $r$ but not $\bm{v}_*$. For example, where we are interested in the area of $\capsr$, it only depends on $r$ but not $\bm{v}_*$. Thus, we write $\lambda_{d-1}(\capr)$ instead.
\end{remark}

\subsection{Quantities related to the area of a cap on a hyper-sphere}
The lemmas of this subsection support for the proof of Proposition~\ref{prop.min_eig}.
The following lemma is introduced by \cite{li2011concise}, which gives the area of a cap on a hyper-sphere with respect to the colatitude angle.
\begin{lemma}\label{le.original_cap}
Let $\phi\in[0,\ \frac{\pi}{2}]$ denote the colatitude angle of the smaller cap on the unit hyper-sphere $\mathcal{S}^{a-1}$, then the area (in the measure of $\lambda_{a-1}$) of this hyper-spherical cap is
\begin{align*}
    \frac{1}{2}\lambda_{a-1}(\mathcal{S}^{a-1})I_{\sin^2\phi}\left(\frac{a-1}{2},\ \frac{1}{2}\right),
\end{align*}
or equivalently\footnote{Proof of this equivalence can be found in Lemma~9 of \cite{ju2021generalization}.},
\begin{align*}
    \lambda_{a-1}(\capr)=\frac{1}{2}\lambda_{a-1}(\mathcal{S}^{a-1})I_{r^2(1-\frac{r^2}{4})}\left(\frac{a-1}{2},\frac{1}{2}\right).
\end{align*}
where $r\leq \sqrt{2}$.
\end{lemma}

The following lemma is shown by Lemma~35 of \cite{ju2021generalization}.
\begin{lemma}\label{le.estimate_Ix}
For any $x\in[0,\ 1]$, we must have
\begin{align*}
    I_x\left(\frac{a-1}{2},\frac{1}{2}\right)\in\left[ \frac{2x^{\frac{a-1}{2}}}{B(\frac{a-1}{2}, \frac{1}{2})\cdot (a-1)},\  \frac{2x^{\frac{a-1}{2}}}{B(\frac{a-1}{2}, \frac{1}{2})\cdot (a-1)\sqrt{1-x}}\right].
\end{align*}
\end{lemma}

The following lemma is shown by Lemma~32 of \cite{ju2021generalization}.
\begin{lemma}\label{le.bound_B}
For any integer $a\geq 2$,
\begin{align*}
    B\left(\frac{a-1}{2},\ \frac{1}{2}\right)\in \left[\frac{1}{\sqrt{a}},\ \pi\right].
\end{align*}
Further, if $a\geq 5$, we have
\begin{align*}
    B\left(\frac{a-1}{2},\ \frac{1}{2}\right)\in \left[\frac{1}{\sqrt{a}},\ \frac{4}{\sqrt{a-3}}\right].
\end{align*}
\end{lemma}

\subsection{Estimation of certain norms}
In our proofs, we will often need to estimate the norms of the NTK feature vectors. We list some useful lemmas below.
\begin{lemma}\label{le.bound_hx}
For any $\xsmall\in\sd$, we have
\begin{align*}
    \|\hxRF\|_2\leq \sqrt{p_1},\quad \|\hhx\|_2\leq \sqrt{p_1 p_2}.
\end{align*}
\end{lemma}
\begin{proof}
Notice that $\|\xsmall\|_2=1$ and $\|\Vzero[j]\|_2=1$ for all $j\in\{1,2\cdots,p_1\}$. By Eq.~\eqref{eq.def_hxRF}, we have
\begin{align*}
    \|\hxRF\|_2=\sqrt{\sum_{j=1}^{p_1}\left((\xsmall^T\Vzero[j]) \indicator{\xsmall^T\Vzero[j]>0}\right)^2}\leq \sqrt{p_1}.
\end{align*}
Thus, by Eq.~\eqref{eq.def_h3}, we have
\begin{align*}
    \|\hhx\|_2=\sqrt{\sum_{k=1}^{p_2}\|\hxRF\indicator{(\hxRF)^T\Wzero[k]>0}\|_2^2}\leq \sqrt{\sum_{k=1}^{p_2}\|\hxRF\|_2^2}\leq \sqrt{p_1p_2}.
\end{align*}
\end{proof}

The following lemma is from Lemma~12 of \cite{ju2021generalization}, but we repeat here for the convenience of the readers.
\begin{lemma}\label{le.matrix_norm}
If $\mathbf{C}=\mathbf{A}\mathbf{B}$, then $\|\mathbf{C}\|_2\leq \|\mathbf{A}\|_2\cdot\|\mathbf{B}\|_2$. Here $\mathbf{A}$, $\mathbf{B}$, and $\mathbf{C}$ could be scalars, vectors, or matrices.
\end{lemma}
\begin{proof}
This lemma directly follows the definition of matrix norm.
\end{proof}
\begin{remark}
Note that the ($\ell_2$) matrix-norm (i.e., spectral norm) of a vector is exactly its $\ell_2$ vector-norm (i.e., Euclidean norm)\footnote{To see this, consider a (row or column) vector $\bm{a}$. The matrix norm of $\bm{a}$ is \begin{align*}
    &\max_{|x|=1}\|\bm{a}x\|_2\text{ (when $\bm{a}$ is a column vector)},\\
    \text{or }&\max_{\|\xsmall\|_2=1}\|\bm{a}\xsmall\|_2\text{ (when $\bm{a}$ is a row vector)}.
\end{align*}
In both cases, the value of the matrix-norm equals to $\sqrt{\sum a_i^2}$, which is exactly the $\ell_2$-norm (Euclidean norm) of $\bm{a}$.
}. Therefore, when applying Lemma~\ref{le.matrix_norm}, we do not need to worry about whether $\mathbf{A}$, $\mathbf{B}$, and $\mathbf{C}$ are matrices or vectors.
\end{remark}

\begin{lemma}\label{le.small_diff_eig}
For any $\mathbf{A},\mathbf{B}\in\mathds{R}^{k\times k}$, we must have
\begin{align*}
    \|\mathbf{A}-\mathbf{B}\|_2\leq  k\cdot\max_{i,j}|\mathbf{A}_{i,j}-\mathbf{B}_{i,j}|.
\end{align*}
Consequently, if both $\mathbf{A}$ and $\mathbf{B}$ are positive semi-definite, then
\begin{align*}
    \left|\min\eig (\mathbf{A})-\min\eig(\mathbf{B})\right|\leq k\cdot\max_{i,j}|\mathbf{A}_{i,j}-\mathbf{B}_{i,j}|.
\end{align*}
\end{lemma}
\begin{proof}
Let $\mathbf{C}\defeq\mathbf{A}-\mathbf{B}$.
For any $\bm{a}\in \mathcal{S}^{k-1}$, we have
\begin{align*}
    \|\mathbf{C}\bm{a}\|_2^2=&\sum_{i=1}^k \left(\sum_{j=1}^k\mathbf{C}_{i,j}a_{j}\right)^2\\
    \leq & k(\max_{i,j}\mathbf{C}_{i,j})^2\left(\sum_{j=1}^k a_j\right)^2\\
    \leq & k^2(\max_{i,j}\mathbf{C}_{i,j})^2\sum_{j=1}^k a_j^2\text{ (by Cauchy–Schwarz inequality)}\\
    =&k^2(\max_{i,j}\mathbf{C}_{i,j})^2\text{ (because $\|\bm{a}\|_2=1$)}.
\end{align*}
Because $\|\mathbf{C}\|_2=\max_{\bm{a}\in\mathcal{S}^{k-1}}\|\mathbf{C}\bm{a}\|_2$, we have $\|\mathbf{A}-\mathbf{B}\|_2\leq  k\cdot\max_{i,j}|\mathbf{A}_{i,j}-\mathbf{B}_{i,j}|$.

Let $\bm{a}^*\in\argmin_{\bm{a}\in\mathcal{S}^{k-1}}\|\mathbf{B}\bm{a}\|_2$. We have
\begin{align*}
    \min\eig(\mathbf{A})=&\min_{\bm{a}\in\mathcal{S}^{k-1}}\|\mathbf{A}\bm{a}\|_2\\
    \leq &\|\mathbf{A}\bm{a}^*\|_2\\
    = & \|(\mathbf{A}-\mathbf{B})\bm{a}^*+\mathbf{B}\bm{a}^*\|_2\\
    \leq& \|(\mathbf{A}-\mathbf{B})\bm{a}^*\|_2+ \|\mathbf{B}\bm{a}^*\|_2\\
    \leq & \|\mathbf{A}-\mathbf{B}\|_2+\min\eig(\mathbf{B})\text{ (by the definition of $\bm{a}^*$)}.
\end{align*}
Thus, we have $\min\eig(\mathbf{A})-\min\eig(\mathbf{B})\leq \|\mathbf{A}-\mathbf{B}\|_2\leq k\cdot \max_{i,j}|\mathbf{C}_{i,j}|$. Similarly, we have $\min\eig(\mathbf{B})-\min\eig(\mathbf{A})\leq  k\cdot \max_{i,j}|\mathbf{C}_{i,j}|$. The result of this lemma thus follows.
\end{proof}

\subsection{Estimates of certain tail probabilities}

% {\color{red}Ref}
% \begin{lemma}\label{le.large_deviation}
% Let $X_1,X_2,\cdots,X_k$ be \emph{i.i.d.} random vectors (of the same dimension) in a real Euclidean space such that $\|X_i\|_2\leq U$ for all $i=1,2,\cdots,k$. Then, for any $q\in[1,\ \infty)$,
% \begin{align*}
%     \prob\left\{\left\|\left(\frac{1}{k}\sum_{i=1}^k X_i\right)-\expectation X_1\right\|_2\geq U \cdot k^{\frac{1}{2q}-\frac{1}{2}}\right\}<2e^2\exp\left(-\frac{\sqrt[q]{k}}{8}\right).
% \end{align*}
% \end{lemma}
% \begin{proof}
% \TODO
% \end{proof}

\begin{lemma}[Chebyshev's inequality on the sum of \emph{i.i.d.} random variables/vectors]\label{le.chebyshev}
Let $X_1,X_2,\cdots,X_k$ be $\emph{i.i.d.}$ random variables and $|X_i|\leq U$ for all $i=1,2,\cdots,k$. Then, for any $m>0$,
\begin{align*}
    \prob\left\{\left|\left(\frac{1}{k}\sum_{i=1}^k X_i\right)-\expectation X_1\right|\geq \frac{mU}{\sqrt{k}}\right\}\leq \frac{1}{m^2}.
\end{align*}
This inequality also holds when $X_1,X_2,\cdots,X_k$ are \emph{i.i.d.} random vectors and $\|X_i\|_2\leq U$ for all $i=1,2,\cdots,k$.
\end{lemma}
\begin{proof}
Because $|X_1|\leq U$, we have
\begin{align*}
    \vari[X_1]=\expectation[(X_1-\expectation[X_1])^2]= \expectation[X_1^2]-(\expectation[X_1])^2\leq\expectation[X_1^2] \leq U^2.
\end{align*}
Because all $X_i$'s are \emph{i.i.d.}, we have
\begin{align*}
    \vari\left[\frac{1}{k}\sum_{i=1}^k X_i\right]\leq \frac{U^2}{k},\quad \expectation\left[\frac{1}{k}\sum_{i=1}^k X_i\right]=\expectation[X_1].
\end{align*}
The result of this lemma thus follows by applying Chebyshev's inequality on $\frac{1}{k}\sum_{i=1}^k X_i$. For the situation that $X_1,X_2,\cdots,X_k$ are vectors, the proof is the same by using the generalized Chebyshev's inequality for random vectors which we state in Lemma~\ref{le.Chebyshev_vector} as follows.
\end{proof}

The following is the Chebyshev's inequality for random vectors that can be found in many textbooks of probability theory (see, e.g.,  pp.~446-451 of \cite{laha1979probability}).
\begin{lemma}[Chebyshev's inequality for random vectors]\label{le.Chebyshev_vector}
For a random vector $\bm{w}\in\mathds{R}^a$ with probability distribution $\Lambda(\cdot)$, for any $\delta>0$, we must have
\begin{align*}
    \prob\left\{\|\bm{w}-\expectation(\bm{w})\|_2\geq \delta\right\}\leq \frac{\vari (\bm{w})}{\delta^2},
\end{align*}
where
\begin{align*}
    \vari(\bm{w})\defeq\int_{\bm{v}\in\mathds{R}^a} \|\bm{v}-\expectation(\bm{w})\|_2^2 \ d\Lambda(\bm{v}).
\end{align*}
\end{lemma}

\subsection{Estimation about double factorial}
Let $m$ be a positive integer. A double factorial can be defined by
\begin{align}\label{def.double_factorial}
    (2m)!!\defeq \prod_{i=1}^m (2i),\quad (2m-1)!!\defeq \prod_{i=1}^m (2i-1).
\end{align}
They are useful in our study of learnable functions.
The following lemma is proven by \cite{chen2005best}.
\begin{lemma}[Improved Wallis' Inequality]\label{le.wallis_inequality}
For all natural numbers $k$, let $k!!$ denote a double factorial. Then
\begin{align*}
    \frac{1}{\sqrt{\pi\left(k+\frac{4}{\pi}-1\right)}}\leq  \frac{(2k-1)!!}{(2k)!!}<\frac{
    1}{\sqrt{\pi\left(k+\frac{1}{4}\right)}}.
\end{align*}
Further, the constants $\frac{4}{\pi}-1$ and $\frac{1}{4}$ are the best possible.
\end{lemma}

\subsection{Taylor expansion of kernels}
The following Taylor expansions are related to the NTK kernel functions, which will also be used in our characterization of the learnable functions.
\begin{lemma}\label{le.taylor_of_kernel}
For any $\theta\in[0,\pi]$,
\begin{align*}
    &\cos\theta\frac{(\pi-\theta)}{2\pi}=\frac{\cos\theta}{4}+\frac{1}{2\pi}\sum_{k=0}^\infty\frac{(2k)!}{(k!)^2}\frac{4}{2k+1}\left(\frac{\cos\theta}{2}\right)^{2k+2},\\
    &\frac{\sin\theta+(\pi-\theta)\cos\theta}{\pi}=\frac{1}{\pi}\left(1+\frac{\pi}{2}\cos\theta+\sum_{k=0}^\infty\frac{2(2k)!}{(k+1)(2k+1)(k!)^2}\left(\frac{\cos\theta}{2}\right)^{2k+2}\right).
\end{align*}
Consequently, recalling Eq.~\eqref{eq.def_KTwo} and Eq.~\eqref{eq.def_KRF}, by  letting $a=\cos \theta$, we have
\begin{align*}
    \KTwo(a)=&a\frac{\pi-\arccos a}{2\pi}=\frac{a}{4}+\frac{1}{2\pi}\sum_{k=0}^\infty \frac{(2k)!}{(k!)^2}\frac{4}{2k+1}\left(\frac{a}{2}\right)^{2k+2},\\
    2d\cdot \KRF(a)=&\frac{\sqrt{1-a^2}+a(\pi-\arccos a)}{\pi}\\
    =&\frac{1}{\pi}\left(1+\frac{\pi}{2}a+\sum_{k=0}^\infty\frac{2(2k)!}{(k+1)(2k+1)(k!)^2}\left(\frac{a}{2}\right)^{2k+2}\right).
\end{align*}
\end{lemma}
\begin{proof}
Using Taylor expansion on $\arccos{x}$, we have
\begin{align*}
    \arccos(x)=\frac{\pi}{2}-\sum_{k=0}^{\infty}\frac{(2k)!}{2^{2k}(k!)^2}\frac{x^{2k+1}}{2k+1}.
\end{align*}
We then have
\begin{align*}
    \theta = \arccos (\cos\theta)=\frac{\pi}{2}-\sum_{k=0}^\infty\frac{(2k)!}{(k!)^2}\frac{2}{2k+1}\left(\frac{\cos\theta}{2}\right)^{2k+1}.
\end{align*}
Thus, we have
\begin{align}
    \cos\theta\frac{(\pi-\theta)}{2\pi}=&\cos\theta\cdot\left(\frac{1}{2}-\frac{1}{2\pi}\left(\frac{\pi}{2}-\sum_{k=0}^\infty\frac{(2k)!}{(k!)^2}\frac{2}{2k+1}\left(\frac{\cos\theta}{2}\right)^{2k+1}\right)\right)\nonumber\\
    =&\frac{\cos\theta}{4}+\frac{1}{2\pi}\sum_{k=0}^\infty\frac{(2k)!}{(k!)^2}\frac{4}{2k+1}\left(\frac{\cos\theta}{2}\right)^{2k+2}.\label{eq.temp_012101}
\end{align}
Using Taylor expansion on $\sqrt{1+x}$, we have
\begin{align*}
    \sqrt{1+x}=1-\sum_{k=0}^{\infty}\frac{2}{k+1}\binom{2k}{k}\left(-\frac{x}{4}\right)^{k+1},
\end{align*}% \url{https://math.stackexchange.com/questions/732540/taylor-series-of-sqrt1x-using-sigma-notation}.
Replacing $x$ by $-\cos^2\theta$, we thus have
\begin{align*}
    \sin\theta = \sqrt{1-\cos^2\theta}=1- \sum_{k=0}^\infty \frac{2}{k+1}\binom{2k}{k}\left(\frac{\cos\theta}{2}\right)^{2k+2}.
\end{align*}
Therefore, using Eq.~\eqref{eq.temp_012101} again, we have
\begin{align*}
    \frac{\sin\theta + (\pi-\theta)\cos\theta}{\pi}&=\frac{1}{\pi}\left(1+\frac{\pi}{2}\cos\theta+\sum_{k=0}^\infty\left(\frac{2}{2k+1}-\frac{1}{k+1}\right)\frac{2(2k)!}{(k!)^2}\left(\frac{\cos\theta}{2}\right)^{2k+2}\right)\\
    &=\frac{1}{\pi}\left(1+\frac{\pi}{2}\cos\theta+\sum_{k=0}^\infty\frac{2(2k)!}{(k+1)(2k+1)(k!)^2}\left(\frac{\cos\theta}{2}\right)^{2k+2}\right).
\end{align*}
The result of this lemma thus follows.
\end{proof}

\subsection{Calculation of certain integrals}\label{app.integral}

\begin{lemma}\label{le.sin_k_power}
For any integer $k\geq 2$, we have
\begin{align*}
    \int_0^{\pi}\sin^k\varphi\ d\varphi = \frac{k-1}{k}\int_0^{\pi}\sin^{k-2}\varphi\  d\varphi.
\end{align*}
\end{lemma}
\begin{proof}
We have
\begin{align*}
    \int_0^{\pi}\sin^k\varphi\ d\varphi =& \int_0^{\pi} \sin\varphi\cdot  \sin^{k-1}\varphi\ d\varphi\\
    =&-\cos\varphi\cdot \sin^{k-1}\varphi \big|_0^{\pi} + (k-1)\int_0^{\pi}\cos^2\varphi\cdot \sin^{k-2}\varphi\ d\varphi\\
    &\quad \text{ (integration by parts)}\\
    =&(k-1)\int_0^{\pi} (1-\sin^2\varphi)\sin^{k-2}\varphi\ d\varphi\\
    =&(k-1)\int_0^{\pi}\sin^{k-2}\varphi\ d\varphi-(k-1)\int_0^{\pi}\sin^k\varphi\ d\varphi.
\end{align*}
Moving the second term of the right hand side to the left hand side, we have
\begin{align*}
    k\int_0^{\pi}\sin^k\varphi\ d\varphi=(k-1)\int_0^{\pi}\sin^{k-2}\varphi\ d\varphi.
\end{align*}
The result of this lemma thus follows.
\end{proof}

\begin{lemma}\label{le.cos_cos_integral}
For any $\theta\in[0,\pi]$,
\begin{align*}
    \int_{-\frac{\pi}{2}+\theta}^{\frac{\pi}{2}}\cos(\alpha)\cos(\alpha-\theta)\ d\alpha=\frac{\sin\theta}{2}+\frac{(\pi-\theta)\cos\theta}{2}.
\end{align*}
\end{lemma}
\begin{proof}
Notice that
\begin{align*}
    \frac{\partial (\sin(2\alpha-\theta)+2\alpha\cos\theta)}{\partial \alpha}=&2\cos(2\alpha-\theta)+2\cos\theta\\
    =&2\cos(\alpha+(\alpha-\theta))+2\cos(\alpha-(\alpha-\theta))\\
    =& 4\cos(\alpha)\cos(\alpha-\theta).
\end{align*}
Thus, we have
\begin{align*}
    \int \cos(\alpha)\cos(\alpha - \theta)d\alpha = \frac{\sin(2\alpha-\theta)+2\alpha\cos(\theta)}{4}+\text{constant}.
\end{align*}
Notice that
\begin{align*}
    &\sin(2\alpha-\theta)\bigg|_{\alpha=-\frac{\pi}{2}+\theta}^{\frac{\pi}{2}}=\sin(\pi-\theta)-\sin (\theta-\pi)=2\sin\theta,\\
    &2\alpha \cos(\theta)\bigg|_{\alpha=-\frac{\pi}{2}+\theta}^{\frac{\pi}{2}}=2(\pi-\theta)\cos\theta.
\end{align*}
The result of this lemma thus follows.
\end{proof}

\begin{lemma}\label{le.NTK2_kernel}
Recall that $\wdensity(\cdot)$ denotes the probability density function of $\Wzero[k]$ and is $\mathsf{unif}(\spp)$ by Assumption~\ref{as.uniform}.
For any $\bm{a},\bm{b}\in \mathds{R}^{p_1}$, we have
\begin{align*}
    \int_{\spp}\bm{a}^T\bm{b}\cdot \indicator{\bm{a}^T\bm{w}>0,\ \bm{b}^T\bm{w}>0} d \wdensity(\bm{w})=\bm{a}^T\bm{b}\frac{\pi-\arccos\left(\frac{\bm{a}^T\bm{b}}{\|\bm{a}\|_2\|\bm{b}\|_2}\right)}{2\pi}.
\end{align*}
(Although the right hand side is not defined when $\bm{a}=\bm{0}$ or $\bm{b}=\bm{0}$, we can artificially re-define the value of the right hand side as $0$ when $\bm{a}=\bm{0}$ or $\bm{b}=\bm{0}$, so the equation still holds.)
\end{lemma}
\begin{proof}
The result holds trivially when $\bm{a}=0$ or $\bm{b}=0$. When $\bm{a}$ and $\bm{b}$ are both non-zero,
it suffices to prove that
\begin{align*}
    \int_{\spp} \indicator{\bm{a}^T\bm{w}>0,\ \bm{b}^T\bm{w}>0} d \wdensity(\bm{w})=\frac{\pi-\arccos\left(\frac{\bm{a}^T\bm{b}}{\|\bm{a}\|_2\|\bm{b}\|_2}\right)}{2\pi},
\end{align*}
which has been proven by Lemma~17 of \cite{ju2021generalization} (where its geometric explanation is given as well).
\end{proof}

\begin{lemma}\label{le.RF_kernel}
For any $\xsmall,\bm{z}\in\sd$, we have
\begin{align}\label{eq.RF_temp_111001}
    \int_{\sd} (\xsmall^T \bm{v})(\bm{z}^T\bm{v})\indicator{\bm{z}^T\bm{v}>0,\ \xsmall^T\bm{v}>0}d\vdensity(\bm{v})=\frac{\sin\theta + (\pi-\theta)\cos\theta}{2d\pi},
\end{align}
where $\theta$ denotes the angle between $\xsmall$ and $\bm{z}$, i.e.,
\begin{align}\label{eq.RF_temp_020901}
    \theta = \arccos(\xsmall^T\bm{z})\in [0,\ \pi].
\end{align}
\end{lemma}
To help readers understand the correctness of Lemma~\ref{le.RF_kernel}, we first give a simple proof for the special case that $d=2$, i.e., when vectors $\bm{x}$, $\bm{z}$, and $\bm{v}$ are all in the 2-D plane. Then we prove Lemma~\ref{le.RF_kernel} for the general cases that $d=2,3,4,\cdots$.

\emph{Proof (of the case when $d=2$)}: Without loss of generality, we let
\begin{align*}
    \bm{v}=\begin{bmatrix}
    \cos \alpha\\
    \sin \alpha
    \end{bmatrix},\quad \bm{z}=\begin{bmatrix}
    1\\
    0
    \end{bmatrix},\quad\text{and }
    \bm{x}=\begin{bmatrix}
    \cos \theta\\
    \sin \theta
    \end{bmatrix}.
\end{align*}
Thus, we have
\begin{align*}
    \text{The left-hand-side of Eq.~\eqref{eq.RF_temp_111001}}=& \frac{1}{2\pi}\int_{\left(\theta-\frac{\pi}{2},\ \theta+\frac{\pi}{2}\right)\cap \left(-\frac{\pi}{2},\ \frac{\pi}{2}\right)}  \left(\cos\alpha \cos \theta+\sin\alpha\sin \theta \right)\cos \alpha\ d\alpha\\
    =&\frac{1}{2\pi}\int_{\theta-\frac{\pi}{2}}^{\frac{\pi}{2}} \cos (\alpha -\theta)\cos \alpha\ d\alpha\quad \text{ (since $\theta\in [0, \pi]$)}\\
    =&\frac{\sin\theta +(\pi-\theta)\cos \theta}{4\pi}\quad \text{ (by Lemma~\ref{le.cos_cos_integral})}.
\end{align*}

\begin{proof}[Proof (of the general case)]
Due to symmetry, we know that the integral in the left-hand-side of Eq.~\eqref{eq.RF_temp_111001} only depends on the angle between $\xsmall$ and $\bm{z}$. Thus, without loss of generality, we let
\begin{align*}
    \xsmall =[\xsmall_1\ \xsmall_2\ \cdots\ \xsmall_d]= [0\ 0\ \cdots\ 0\ 1\ 0]^T,\ \bm{z}=[0\ 0\ \cdots\ 0\ \cos\theta\ \sin\theta]^T.
\end{align*}
Thus, for any $\bm{v}=[\bm{v}_1\ \bm{v}_2\ \cdots\ \bm{v}_{d}]^T$, in order for $\bm{z}^T\bm{v}>0$ and $\xsmall^T\bm{v}>0$ to hold, it only needs to satisfy
\begin{align}\label{eq.RF_temp_122901}
    [\cos\theta\ \sin\theta]\begin{bmatrix}\bm{v}_{d-1}\\\bm{v}_d\end{bmatrix}>0,\quad [1\ 0]\begin{bmatrix}\bm{v}_{d-1}\\\bm{v}_d\end{bmatrix}>0.
\end{align}
We use the spherical coordinate $\bm{\varphi}_{\xsmall}=[\varphi_1^{\xsmall}\ \varphi_2^{\xsmall}\ \cdots\ \varphi_{d-1}^{\xsmall}]^T$ where $\varphi_1^{\xsmall},\cdots,\varphi_{d-2}^{\xsmall}\in [0,\pi]$ and $\varphi_{d-1}^{\xsmall}\in[0,2\pi)$ with the convention that
\begin{align*}
    &\bm{x}_1=\cos(\varphi_1^{\xsmall}),\\
    &\bm{x}_2=\sin(\varphi_1^{\xsmall})\cos(\varphi_2^{\xsmall}),\\
    &\bm{x}_3=\sin(\varphi_1^{\xsmall})\sin(\varphi_2^{\xsmall})\cos(\varphi_3^{\xsmall}),\\
    &\vdots\\
    &\bm{x}_{d-1}=\sin(\varphi_1^{\xsmall})\sin(\varphi_2^{\xsmall})\cdots\sin(\varphi_{d-2}^{\xsmall})\cos(\varphi_{d-1}^{\xsmall}),\\
    &\bm{x}_d=\sin(\varphi_1^{\xsmall})\sin(\varphi_2^{\xsmall})\cdots\sin(\varphi_{d-2}^{\xsmall})\sin(\varphi_{d-1}^{\xsmall}).
\end{align*}
Thus, we have $\bm{\varphi}_{\xsmall}=[\pi/2\ \pi/2\ \cdots\ \pi/2\ 0]^T$. Similarly, the spherical coordinate for $\bm{z}$ is $\bm{\varphi}_{\bm{z}}=[\pi/2\ \pi/2\ \cdots \pi/2\ \theta]^T$. Let the spherical coordinates for $\bm{v}$ be $\bm{\varphi}_{\bm{v}}=[\varphi_1^{\bm{v}}\ \varphi_2^{\bm{v}}\ \cdots\ \varphi_{d-1}^{\bm{v}}]^T$. Thus, Eq.~\eqref{eq.RF_temp_122901} is equivalent to
\begin{align}
    &\bm{z}^T\bm{v}=\sin(\varphi_1^{\bm{v}})\sin(\varphi_2^{\bm{v}})\cdots\sin(\varphi_{d-2}^{\bm{v}})\left(\cos\theta \cos(\varphi_{d-1}^{\bm{v}})+\sin\theta\sin(\varphi_{d-1}^{\bm{v}})\right)>0,\label{eq.RF_temp_020902}\\
    &\xsmall^T\bm{v}=\sin(\varphi_1^{\bm{v}})\sin(\varphi_2^{\bm{v}})\cdots\sin(\varphi_{d-2}^{\bm{v}})\cos(\varphi_{d-1}^{\bm{v}})>0.\label{eq.RF_temp_020903}
\end{align}
Because $\varphi_1^{\bm{v}},\cdots,\varphi_{d-2}^{\bm{v}}\in[0,\pi]$ (by the convention of spherical coordinates), we have
\begin{align*}
    \sin(\varphi_1^{\bm{v}})\sin(\varphi_2^{\bm{v}})\cdots\sin(\varphi_{d-2}^{\bm{v}})\geq 0.
\end{align*}
Thus, for Eq.~\eqref{eq.RF_temp_020902} and Eq.~\eqref{eq.RF_temp_020903} to hold, we must have
\begin{align*}
    \cos(\theta-\varphi_{d-1}^{\bm{v}})>0,\quad \cos(\varphi_{d-1}^{\bm{v}})>0,
\end{align*}
i.e., $\varphi_{d-1}^{\bm{v}}\in (-\pi/2,\ \pi/2)\cap (\theta-\pi/2,\ \theta+\pi/2)\pmod{2\pi}$. By Eq.~\eqref{eq.RF_temp_020901}, we thus have
\begin{align*}
    \varphi_{d-1}\in \left(-\frac{\pi}{2}+\theta,\ \frac{\pi}{2}\right) \pmod{2\pi}.
\end{align*}
Let
\begin{align*}
    A(\theta,\varphi_{d-1}^{\bm{v}})\defeq \left(\cos\theta \cos(\varphi_{d-1}^{\bm{v}})+\sin\theta\sin(\varphi_{d-1}^{\bm{v}})\right)\cos(\varphi_{d-1}^{\bm{v}})=\cos(\varphi_{d-1}^{\bm{v}}-\theta)\cos\varphi_{d-1}^{\bm{v}}.
\end{align*}
By Eq.~\eqref{eq.RF_temp_020902} and Eq.~\eqref{eq.RF_temp_020903}, we have
\begin{align*}
    (\xsmall^T \bm{v})(\bm{z}^T\bm{v})\bm{1}_{\{\bm{z}^T\bm{v}>0,\ \xsmall^T\bm{v}>0\}}=\sin^2(\varphi_1^{\bm{v}})\sin^2(\varphi_2^{\bm{v}})\cdots\sin^2(\varphi_{d-2}^{\bm{v}})A(\theta,\varphi_{d-1}^{\bm{v}}).
\end{align*}
Integrating using such spherical coordinates, we have
\begin{align*}
    &\int_{\sd} (\xsmall^T \bm{v})(\bm{z}^T\bm{v})\bm{1}_{\{\bm{z}^T\bm{v}>0,\ \xsmall^T\bm{v}>0\}}d\vdensity(\bm{v})\\
    =&\frac{\int_{-\frac{\pi}{2}+\theta}^{\frac{\pi}{2}}A(\theta,\varphi_{d-1}^{\bm{v}})\int_0^\pi\cdots\int_0^\pi \sin^{d}(\varphi_1)\sin^{d-1}(\varphi_2)\cdots \sin^3(\varphi_{d-2})\ d\varphi_1\ d\varphi_2\cdots d\varphi_{d-1}}{\int_0^{2\pi}\int_0^\pi\cdots\int_0^\pi \sin^{d-2}(\varphi_1)\sin^{d-3}(\varphi_2)\cdots \sin(\varphi_{d-2})\ d\varphi_1\ d\varphi_2\cdots d\varphi_{d-1}}\\
    =&\frac{\int_{-\frac{\pi}{2}+\theta}^{\frac{\pi}{2}}A(\theta,\varphi_{d-1}^{\bm{v}})\cdot d\varphi_{d-1}}{\int_0^{2\pi}d\varphi_{d-1}}\cdot\frac{d-1}{d}\frac{d-2}{d-1}\cdots\frac{2}{3}\text{ (by Lemma~\ref{le.sin_k_power})}\\
    =&\frac{\sin\theta + (\pi-\theta)\cos\theta}{2d\cdot\pi}\text{ (by Lemma~\ref{le.cos_cos_integral})}.
\end{align*}
The result of this lemma thus follows.
\end{proof}

\subsection{Convergence of \texorpdfstring{$\frac{1}{p_1}(\hxRF)^T\hzRF$ with respect to $p_1$}{RF kernel}}

\begin{lemma}[Theorem 4.2 of \cite{uniformConvergence}]\label{le.Rademacher}
Let $\mathscr{F}$ be a class of real-valued functions $f$ such that $\|f\|_{\infty}\leq b$ for all $f\in\mathscr{F}$. Then for all $k\geq 1$ and $\delta \geq 0$, we have
\begin{align*}
    \prob\left\{\sup_{f\in \mathscr{F}}\left|\frac{1}{k}\sum_{i=1}^kf(X_i)-\expectation_{x\sim \mathcal{X}(\cdot)}f(x)\right|\leq 2\mathcal{R}_k(\mathscr{F})+\delta\right\}\geq 1 - \exp\left(-\frac{k\delta^2}{8b^2}\right),
\end{align*}
where $\mathcal{R}_k(\mathscr{F})$ denotes the Rademacher complexity, $X_1,X_2,\cdots,X_k$ are $\emph{i.i.d.}$ random variables/vectors that follow the distribution $\mathcal{X}(\cdot)$.
\end{lemma}

\textbf{Polynomial discrimination}. A class $\mathscr{F}$ of functions with domain $\mathcal{X}$ has polynomial discrimination of order $\nu\geq 1$ if for each positive integer $k$ and collection $X_1^k=\{X_1,\cdots,X_k\}$ of $k$ points in $\mathcal{X}$, the set $\mathscr{F}(X_1^k)$ has cardinality upper bounded by 
\begin{align*}
    \text{card}(\mathscr{F}(X_1^k))\leq (k+1)^{\nu}.
\end{align*}

\begin{lemma}[Lemma~4.1 and Eq.~(4.23) of \cite{uniformConvergence}]\label{le.poly_discrimination}
Suppose that $\mathscr{F}$ has polynomial discrimination of order $\nu$ and $\|f\|_{\infty}\leq b$ for all $f\in\mathscr{F}$. Then
\begin{align*}
    \mathcal{R}_k(\mathscr{F})\leq 3\sqrt{\frac{b^2\nu \log(k+1)}{k}}\quad \text{ for all }k\geq 10.
\end{align*}
\end{lemma}

Given a function $h:\sd\mapsto\mathds{R}$ such that $\|h\|_{\infty}<\infty$ and given any $\delta>0$, consider the function class $\mathscr{F}_*$ that consists of functions $h(\bm{v})\indicator{\xsmall^T\bm{v}>0,\bm{z}^T\bm{v}>0}$, which maps $\bm{v}\in\sd$ to either $0$ or $h(\bm{v})$. By Lemma~20 of \cite{ju2021generalization}, we have 
\begin{align*}
    \text{card}(\mathscr{F}_*(X_1^k))\leq (k+1)^{2(d+1)}.
\end{align*}
(Here $X_1^k$ corresponds to $\{\Vzero[1],\cdots,\Vzero[k]\}$.) Thus, combined with Lemma~\ref{le.Rademacher} and Lemma~\ref{le.poly_discrimination}, we have
\begin{align*}
    \prob_{\Vzero}\bigg\{ & \max_{\xsmall,\bm{z}}\left|\frac{1}{p_1}\sum_{j=1}^{p_1}h(\Vzero[j])\indicator{\xsmall^T\Vzero[j]>0,\bm{z}^T\Vzero[j]>0}-\expectation_{\bm{v}}[h(\bm{v})\indicator{\xsmall^T\bm{v}>0,\bm{z}^T\bm{v}>0}]\right|\\
    &\leq 6\sqrt{\frac{\|h\|_{\infty}^2 2(d+1)\log(p_1+1)}{p_1}}+\delta\bigg\}\geq 1 - \exp\left(-\frac{p_1\delta^2}{8\|h\|_{\infty}^2}\right).
\end{align*}
Further, if we let $\delta = 2\sqrt{\frac{\|h\|_{\infty}^2 2(d+1)\log(p_1+1)}{p_1}}$, we have proven the following lemma.

\begin{lemma}\label{le.max_RF_converge_single_element}
For any given function $h:\sd\mapsto\mathds{R}$ that $\|h\|_{\infty}<\infty$, when $p_1\geq 10$, we have
\begin{align*}
    \prob_{\Vzero}\bigg\{ & \max_{\xsmall,\bm{z}}\left|\frac{1}{p_1}\sum_{j=1}^{p_1}h(\Vzero[j])\indicator{\xsmall^T\Vzero[j]>0,\bm{z}^T\Vzero[j]>0}-\expectation_{\bm{v}}[h(\bm{v})\indicator{\xsmall^T\bm{v}>0,\bm{z}^T\bm{v}>0}]\right|\\
    &\leq 8\sqrt{\frac{\|h\|_{\infty}^2 2(d+1)\log(p_1+1)}{p_1}}\bigg\}\geq 1 - \frac{1}{(p_1+1)e^{d+1}}.
\end{align*}
\end{lemma}

By Eq.~\eqref{eq.def_hxRF}, we have
\begin{align*}
    \frac{1}{p_1}(\hxRF)^T\hzRF=&\frac{1}{p_1}\sum_{j=1}^{p_1}(\xsmall^T\Vzero[j])(\Vzero[j]^T\bm{z})\indicator{\xsmall^T\Vzero[j]>0,\bm{z}^T\Vzero[j]>0}\\
    =&\xsmall^T\left(\frac{1}{p_1}\sum_{j=1}^{p_1}(\Vzero[j]\Vzero[j]^T)\indicator{\xsmall^T\Vzero[j]>0,\bm{z}^T\Vzero[j]>0}\right)\bm{z}.
\end{align*}
Notice that $\Vzero[j]\Vzero[j]^T$ is a $d\times d$ matrix.
Define
\begin{align*}
    \mathbf{K}_j\defeq (\Vzero[j]\Vzero[j]^T)\indicator{\xsmall^T\Vzero[j]>0,\bm{z}^T\Vzero[j]>0}\in\mathds{R}^{d\times d}.
\end{align*}
Thus, we have
\begin{align}
    &\max_{\xsmall,\bm{z}}\left|\frac{1}{p_1}(\hxRF)^T\hzRF-\KRF(\xsmall,\bm{z})\right|\nonumber\\
    = & \max_{\xsmall,\bm{z}}\left|\xsmall^T\left(\frac{1}{p_1}\sum_{j=1}^{p_1}\mathbf{K}_j-\expectation_{\bm{v}\sim\vdensity(\cdot) }(\bm{v}\bm{v}^T)\indicator{\xsmall^T\bm{v}>0,\bm{z}^T\bm{v}>0}\right)\bm{z}\right|\nonumber\\
    \leq & \max_{\xsmall,\bm{z}}\left\|\xsmall^T\right\|_2\cdot \left\|\left(\frac{1}{p_1}\sum_{j=1}^{p_1}\mathbf{K}_j-\expectation_{\bm{v}\sim\vdensity(\cdot) }(\bm{v}\bm{v}^T)\indicator{\xsmall^T\bm{v}>0,\bm{z}^T\bm{v}>0}\right)\right\|_2\cdot\left\|\bm{z}\right\|_2\text{ (by Lemma~\ref{le.matrix_norm})}\nonumber\\
    =&\max_{\xsmall,\bm{z}}\left\|\frac{1}{p_1}\sum_{j=1}^{p_1}\mathbf{K}_j-\expectation_{\bm{v}\sim\vdensity(\cdot) }(\bm{v}\bm{v}^T)\indicator{\xsmall^T\bm{v}>0,\bm{z}^T\bm{v}>0}\right\|_2\text{ (because $\|\xsmall\|_2=\|\bm{z}\|_2=1$)}.\label{eq.temp_042301}
\end{align}

For any $k,l\in \{1,2,\cdots,d\}$, define the $(k,l)$-th element of $\mathbf{K}_j$ as $K_{j,k,l}$. Thus, by Lemma~\ref{le.max_RF_converge_single_element} (notice that $|K_{j,k,l}|\leq 1$), we have
\begin{align*}
    \prob_{\Vzero}\bigg\{ & \max_{\xsmall,\bm{z}}\left|\frac{1}{p_1}\sum_{j=1}^{p_1}K_{j,k.l}-\left(\expectation_{\bm{v}\sim\vdensity(\cdot) }(\bm{v}\bm{v}^T)\indicator{\xsmall^T\bm{v}>0,\bm{z}^T\bm{v}>0}\right)_{k,l}\right|\leq 8\sqrt{\frac{2(d+1)\log(p_1+1)}{p_1}}\bigg\}\\
    &\geq 1-\frac{1}{(p_1+1)e^{d+1}}.
\end{align*}

Applying the union bound on all $d\times d$ elements of $\mathbf{K}_j$ and by Lemma~\ref{le.small_diff_eig}, we have
\begin{align*}
    \prob_{\Vzero}\bigg\{ & \max_{\xsmall,\bm{z}}\left\|\frac{1}{p_1}\sum_{j=1}^{p_1}\mathbf{K}_j-\expectation_{\bm{v}\sim\vdensity(\cdot) }(\bm{v}\bm{v}^T)\indicator{\xsmall^T\bm{v}>0,\bm{z}^T\bm{v}>0}\right\|_2\leq 8d\sqrt{\frac{2(d+1)\log(p_1+1)}{p_1}}\bigg\}\\
    &\geq 1-\frac{d^2}{(p_1+1)e^{d+1}}.
\end{align*}

Plugging it into Eq.~\eqref{eq.temp_042301}, we thus have proven the following lemma.
\begin{lemma}\label{le.max_RF_converge}
Recall the definition of $Q(\cdot,\cdot)$ in Eq.~\eqref{eq.def_Q}. When $p_1\geq 10$, we have
\begin{align*}
    \prob_{\Vzero}\left\{\max_{\xsmall,\bm{z}}\left|\frac{1}{p_1}(\hxRF)^T\hzRF-\KRF(\xsmall,\bm{z})\right|\leq \Qpd\right\}\geq 1-\frac{d^2}{(p_1+1)e^{d+1}}.
\end{align*}
\end{lemma}

\subsection{Some useful lemmas about multinomial expansion}
\begin{lemma}[Multinomial theorem (multinomial expansion)]\label{le.multinomial}
For any positive integer $i$ and non-negative integer $j$,
\begin{align*}
    (x_1+x_2+\cdots+x_i)^j=\sum_{k_1+k_2+\cdots k_i=j}(k_1,k_2,\cdots,k_i)!\cdot x_1^{k_1}x_2^{k_2}\cdots x_i^{k_i},
\end{align*}
where
\begin{align*}
    (k_1,k_2,\cdots,k_i)!=\frac{(k_1+k_2+\cdots+k_i)!}{k_1!k_2!\cdots k_i!}
\end{align*}
denotes the multinomial coefficient.
\end{lemma}

\begin{lemma}\label{le.expansion}
We have
\begin{align*}
    \left(\sum_{i=0}^{\infty}a_i x^i\right)^j=\sum_{s=0}^{\infty}\left(\sum_{\substack{k_0+k_1+\cdots+k_s=j\\
    k_1+2k_2+\cdots+sk_s=s\\
    k_0,k_1,\cdots,k_s\in \mathds{Z}_{\geq 0}}}(k_1,k_2,\cdots,k_s)!\cdot a_0^{k_0}a_1^{b_1}\cdots a_s^{k_s}\right)x^s.
\end{align*}
\end{lemma}
\begin{proof}
The result directly follows from Lemma~\ref{le.multinomial}. Notice that $a_ix^i$ will not contribute to $x^s$ when $i>s$.
\end{proof}
\section{Proof of Proposition~\ref{prop.pseudoGT}}\label{app.pseudoGT}

Define
\begin{align}\label{eq.def_DWs}
    \DWs[k]\defeq \int_{\sd}\indicator{(\hzRF)^T\Wzero[k]>0}\hzRF\frac{g(\bm{z})}{p_1p_2}d\mu(\bm{z}),\quad k=1,2,\cdots,p_2.
\end{align}
Notice that $\DWs[k]$ is a vector of size $p_1\times 1$ (same as the size of $\hzRF$ and $\hxRF$).
The connection between $\DWs$ and the pseudo ground-truth $\pseudoGT$ is shown by the following lemma.
\begin{lemma}\label{le.pseudo_is_linear}
For all $\bm{x}\in\sd$, we have
\begin{align*}
    \hhx\cdot \DWs=\pseudoGT(\xsmall).
\end{align*}
\end{lemma}
\begin{proof}
We have
\begin{align*}
    &\hhx\cdot \DWs\\
    =&\sum_{k=1}^{p_2}(\hhx[k])^T\DWs[k]\\
    =&\sum_{k=1}^{p_2}\int_{\sd}\indicator{(\hxRF)^T\Wzero[k]>0,(\hzRF)^T\Wzero[k]>0}(\hxRF)^T\hzRF\frac{g(\bm{z})}{p_1p_2}d\xdensity(\bm{z})\\
    &\text{(by Eq.~\eqref{eq.def_h3} and Eq.~\eqref{eq.def_DWs})}\\
    =&\int_{\sd}\sum_{k=1}^{p_2}\frac{\indicator{(\hxRF)^T\Wzero[k]>0,(\hzRF)^T\Wzero[k]>0}}{p_1p_2}(\hxRF)^T\hzRF g(\bm{z})d\xdensity(\bm{z})\\
    =&\int_{\sd} (\hzRF)^T\hxRF \frac{\left|\Chzhx\right|}{p_1p_2}g(\bm{z})d\xdensity(\bm{z})\text{ (by Eq.~\eqref{eq.def_Cw})}\\
    =&\pseudoGT(\xsmall)\text{ (by Eq.~\eqref{eq.def_pseudoGT})}.
\end{align*}
\end{proof}

The following lemma bounds the test error for the pseudo ground-truth function with respect to the distance between $\DWs$ and the row-space of $\HHt$.
\begin{lemma}\label{le.hPIDW}
For all $\bm{a}\in\mathds{R}^{n}$, we have
\begin{align*}
    |\pseudoGT(\xsmall)-\fl(\xsmall)|\leq \sqrt{p_1p_2}\|\DWs-\HHt^T\bm{a}\|_2.
\end{align*}
\end{lemma}
\begin{proof}
Define $\mathbf{P}\defeq \HHt^T(\HHt\HHt^T)^{-1}\HHt$. It is easy to verify that $\mathbf{P}^2=\mathbf{P}=\mathbf{P}^T$, so $\mathbf{P}$ is an orthogonal projection onto the space spanned by the rows of $\HHt$. By Lemma~\ref{le.pseudo_is_linear} and Eq.~\eqref{eq.def_minl2_solution}, when $\esmall=\bm{0}$ and the ground-truth function is $\pseudoGT$, we have $\F(\XX)=\HHt \DWs$ and 
\begin{align*}
    \fl(\xsmall)=\hhx\HHt^T(\HHt\HHt^T)^{-1}\HHt\DWs=\hhx\mathbf{P}\DWs.
\end{align*}
Thus, by Lemma~\ref{le.pseudo_is_linear}, we have
\begin{align}\label{eq.temp_040205}
    |\pseudoGT(\xsmall)-\fl(\xsmall)|=|\hhx (\mathbf{P}-\mathbf{I})\DWs|.
\end{align}
Because $\PH=\HHt^T(\HHt\HHt^T)^{-1}\HHt$, we have
\begin{align}\label{eq.temp_033001}
    \PH\HHt^T=\HHt^T(\HHt\HHt^T)^{-1}\HHt\HHt^T=\HHt^T.
\end{align}
We then have
\begin{align}
    \|(\PH-\identity)\DWs\|_2&=\|\PH\DWs-\DWs\|_2\nonumber\\
    &=\|\PH(\HHt^T\bm{a}+\DWs-\HHt^T\bm{a})-\DWs\|_2\nonumber\\
    &=\|\PH\HHt^T\bm{a}+\PH(\DWs-\HHt^T\bm{a})-\DWs\|_2\nonumber\\
    &=\|\HHt^T\bm{a}+\PH(\DWs-\HHt^T\bm{a})-\DWs\|_2\text{ (by Eq.~\eqref{eq.temp_033001})}\nonumber\\
    &=\|(\PH-\identity)(\DWs-\HHt^T\bm{a})\|_2\nonumber\\
    &\leq \|\DWs-\HHt^T\bm{a}\|_2\text{ (because $\PH$ is an orthogonal projection)}.\label{eq.temp_082901}
\end{align}
Therefore, we have
\begin{align*}
    \left|\hhx(\PH-\identity)\DWs\right|=&\left\|\hhx(\PH-\identity)\DWs\right\|_2\\
    \leq &\|\hhx\|_2\cdot\|(\PH-\identity)\DWs\|_2\text{ (by Lemma~\ref{le.matrix_norm})}\\
    \leq& \sqrt{p_1p_2}\|\DWs-\HHt^T\bm{a}\|_2\text{ (by Lemma~\ref{le.bound_hx} and Eq.~\eqref{eq.temp_082901})}.
\end{align*}
By Eq.~\eqref{eq.temp_040205}, the result of this lemma thus follows.
\end{proof}

Now we are ready to prove Proposition~\ref{prop.pseudoGT}.

Define $\mathbf{K}_i\in\mathds{R}^{(p_1p_2)\times 1}$ (the same shape as $\Wzero$) as
\begin{align}\label{eq.temp_040101}
    \mathbf{K}_i[k]\defeq \hXiRF\indicator{(\hXiRF)^T\Wzero[k]>0}\frac{g(\XX_i)}{p_1 p_2},\ i\in\{1,2,\cdots,n\},\ k\in\{1,2,\cdots,p_2\}.
\end{align}
It is obvious that $\mathbf{K}_1,\mathbf{K}_2,\cdots,\mathbf{K}_n$ are \emph{i.i.d.} with respect to the randomness of $\XX$.
By Eq.~\eqref{eq.def_DWs}, for all $k=1,2,\cdots, p_2$, we have
\begin{align}\label{eq.temp_040203}
    \expectation_{\XX_i}\left[\mathbf{K}_i[k]\right]=\DWs[k].
\end{align}
Further, note that
\begin{align*}
    \|\mathbf{K}_i[k]\|_2\leq & \frac{\|g\|_{\infty}}{p_1p_2}\|\hXiRF\|_2\text{ (by Lemma~\ref{le.matrix_norm} and Eq.~\eqref{eq.temp_040101})}\\
    \leq & \frac{\|g\|_{\infty}}{\sqrt{p_1}p_2}\text{ (by Lemma~\ref{le.bound_hx})}.
\end{align*}
Thus, we have
\begin{align*}
    \|\mathbf{K}_i\|_2=\sqrt{\sum_{k=1}^{p_2}\|\mathbf{K}_i[k]\|_2^2}\leq \frac{\|g\|_{\infty}}{\sqrt{p_1p_2}},
\end{align*}
i.e.,
\begin{align}\label{eq.temp_040201}
    \sqrt{p_1p_2}\|\mathbf{K}_i\|_2\leq \|g\|_{\infty}.
\end{align}

We now construct the vector $\bm{a}\in\mathds{R}^n$ that we will use in Lemma~\ref{le.hPIDW}. Its $i$-th element is $\bm{a}_i=\frac{g(\XX_i)}{np_1p_2}$, $i=1,2,\cdots,n$. Then, for all $k\in\{1,2,\cdots, p_2\}$, we have
\begin{align*}
    (\HHt^T\bm{a})[k]=&\sum_{i=1}^n\HHt_i^T[k]\bm{a}_i\\
    =&\sum_{i=1}^n\hXiRF\indicator{(\hXiRF)^T\Wzero[k]>0}\frac{g(\XX_i)}{n p_1 p_2}\text{ (by Eq.~\eqref{eq.def_h3})}\\
    =&\frac{1}{n}\sum_{i=1}^n \mathbf{K}_i[k]\text{ (by Eq.~\eqref{eq.temp_040101})},
\end{align*}
i.e.,
\begin{align}\label{eq.temp_040202}
    \HHt^T\bm{a}=\frac{1}{n}\sum_{i=1}^n \mathbf{K}_i.
\end{align}

Thus, by Lemma~\ref{le.chebyshev} (with $X_i=\sqrt{p_1p_2}\mathbf{K}_i$, $U=\|g\|_{\infty}, m=q$), we have
\begin{align*}
    \prob_{\XX}\left\{\sqrt{p_1p_2}\left\|\left(\frac{1}{n}\sum_{i=1}^n \mathbf{K}_i\right)-\expectation_{\XX}\mathbf{K}_1\right\|_2\geq \frac{q\|g\|_{\infty}}{\sqrt{n}}\right\}\leq \frac{1}{q^2}.
\end{align*}
Further, by Eq.~\eqref{eq.temp_040202} and Eq.~\eqref{eq.temp_040203}, we have
\begin{align*}
    \prob_{\XX}\left\{\sqrt{p_1p_2}\left\|\HHt^T\bm{a}-\DWs\right\|_2\geq \frac{q\|g\|_{\infty}}{\sqrt{n}}\right\}\leq \frac{1}{q^2}.
\end{align*}
By Lemma~\ref{le.hPIDW}, we thus have
\begin{align*}
    \prob_{\XX}\left\{|\pseudoGT(\xsmall)-\fl(\xsmall)|\geq \frac{q\|g\|_{\infty}}{\sqrt{n}}\right\}\leq \frac{1}{q^2}.
\end{align*}
The result of Proposition~\ref{prop.pseudoGT} thus follows.

% {\color{cyan}
% Thus, by Eq.~\eqref{eq.temp_040201} and Lemma~\ref{le.large_deviation} (with $X_i=\sqrt{p_1p_2}\mathbf{K}_i$, $U=\|g\|_{\infty}$, and $k=n$), we have
% \begin{align*}
%     \prob_{\XX}\left\{\sqrt{p_1p_2}\left\|\left(\frac{1}{n}\sum_{i=1}^n \mathbf{K}_i\right)-\expectation_{\XX}\mathbf{K}_1\right\|_2\geq n^{-\frac{1}{2}\left(1-\frac{1}{q}\right)}\|g\|_{\infty}\right\}\leq 2e^2\exp\left(-\frac{\sqrt[q]{n}}{8}\right).
% \end{align*}
% Further, by Eq.~\eqref{eq.temp_040202} and Eq.~\eqref{eq.temp_040203}, we have
% \begin{align*}
%     \prob_{\XX}\left\{\sqrt{p_1p_2}\left\|\HHt^T\bm{a}-\DWs\right\|_2\geq n^{-\frac{1}{2}\left(1-\frac{1}{q}\right)}\|g\|_{\infty}\right\}\leq 2e^2\exp\left(-\frac{\sqrt[q]{n}}{8}\right).
% \end{align*}
% By Lemma~\ref{le.hPIDW}, we thus have
% \begin{align*}
%     \prob_{\XX}\left\{|\pseudoGT(\xsmall)-\fl(\xsmall)|\geq n^{-\frac{1}{2}\left(1-\frac{1}{q}\right)}\|g\|_{\infty}\right\}\leq 2e^2\exp\left(-\frac{\sqrt[q]{n}}{8}\right).
% \end{align*}
% }

\section{Proof of Proposition~\ref{prop.min_eig} (Minimum Eigenvalue of \texorpdfstring{$\HHt\HHt^T$}{HHT})}\label{app.minEig}
Define
\begin{align*}
    &\RFnormij\defeq \left\|\hXiRF\right\|_2\cdot\left\|\hXjRF\right\|_2,\\
    &\RFangleij\defeq \arccos\left(\frac{(\hXiRF)^T\hXjRF}{\RFnorm_{i,j}}\right)\in \left[0,\ \frac{\pi}{2}\right],\\
    &\RFanglemin\defeq \min_{i\neq j}\RFangleij.
    % &\cos\RFanglemin\defeq \max_{i\neq j}\left|\frac{(\hXiRF)^T\hXjRF}{\RFnorm_{i,j}}\right|.
\end{align*}

(By Eq.~\eqref{eq.def_hxRF}, we know that every element of $\hXiRF$ and $\hXjRF$ are non-negative, and hence $\RFangleij\in[0,\ \frac{\pi}{2}]$.)

Define $\Hinfnormalized\in\mathds{R}^{n\times n}$ as
\begin{align}\label{eq.def_Hinf}
    \Hinfnormalized_{i,j}\defeq \frac{p_1}{2d}\cos(\RFangleij)\cdot\frac{\pi-\RFangleij}{2\pi}.
\end{align}

The following lemma (restated) is from the proof of Lemma~1 of \cite{satpathi2021dynamics}, which relates $\min\eig (\Hinfnormalized)$ to $\RFanglemin$. For reader's convenience, we also provide its proof in Appendix~\ref{app.proof_lemma_srikant}.
\begin{lemma}\label{le.Srikant}
\begin{align*}
   \min\eig (\Hinfnormalized) \geq \frac{1}{8\pi}\cdot \frac{p_1}{2d}\cdot \sqrt{\frac{\log(1/\cos\RFanglemin)}{\log(2n/\cos\RFanglemin)}}.
\end{align*}
\end{lemma}

We then focus on estimating $\RFanglemin$.
\begin{lemma}\label{le.min_angle}
Recall the definition of $\Cdq$ in Eq.~\eqref{eq.def_Cndq}.
For any $q>0$, when $p_1$ is sufficient large such that
\begin{align}\label{eq.temp_041202}
    % p_1\geq \max\left\{8qnd\sqrt{2d},\ 32d,\ \left(\frac{32qnd\pi}{\pi-1}\right)^2,\ \left(\frac{16qnd\pi}{\pi-1}\right)^2\left(\frac{(d-1)^2}{8d}\right)^{-\frac{2}{d-1}}\left(qn\right)^{\frac{8}{d-1}}\right\},
    \frac{10dnq\sqrt{2d}}{\sqrt{p_1}}\leq \Cdq .
\end{align}
we have
\begin{align*}
    \prob_{\Vzero,\XX}\left\{\cos\RFanglemin\geq 1-\Cdq  \right\}\leq \frac{4}{q^2}.
\end{align*}
\end{lemma}
The proof of Lemma~\ref{le.min_angle} is in Appendix~\ref{app.le_min_angle}. Intuitively, when $n$ becomes larger, some $\Xii$'s (together with $\hXiRF$'s) will get closer to each other, and thus $\RFanglemin$ will get closer to zero. Such intuition is captured by Lemma~\ref{le.min_angle} since $\Cdq$ is monotone decreasing with respect to $n$.

The above lemmas study the minimum eigenvalue of $\Hinfnormalized$. We need to relate it to the minimum eigenvalue of $\HHt\HHt^T$, which is achieved by the following lemma.

\begin{lemma}\label{le.diff_H_Hinf}
For any $q>0$,
\begin{align*}
    \prob_{\XX,\Vzero,\Wzero}\left\{\left|\frac{1}{p_2}\min\eig(\HHt\HHt^T)-\min\eig(\Hinfnormalized)\right|\geq \qnTwo\right\} \leq \frac{3}{q^2}.
\end{align*}
\end{lemma}
The proof of Lemma~\ref{le.diff_H_Hinf} is in Appendix~\ref{app.proof_le_diff_H_Hinf}. From the derivation in Appendix~\ref{app.learnableSet_derivation}, we know that each element of $\frac{\HHt\HHt^T}{p_1p_2}$ will approach the corresponding element of $\frac{1}{p_1}\Hinfnormalized$ as $p_1$ and $p_2$ get larger. Therefore, it is natural to expect that the minimum eigenvalue of those two matrices will also be closer to each other when $p_1$ and $p_2$ becomes larger, which is captured by Lemma~\ref{le.diff_H_Hinf}.

\begin{lemma}\label{le.log_estimate}
For any $a\in(0,\ 1]$, we have $\log \frac{1}{a}\geq 1-a$.
\end{lemma}
\begin{proof}
Consider the function $h(a)\defeq \log(1/a)-1+a$. We have $\frac{\partial h(a)}{\partial a}=-\frac{1}{a}+1\leq 0$. Thus, we know $h(a)$ is monotone decreasing in $a\in(0,\ 1]$. Thus, we have $h(a)\geq h(1)=0$. The result of this lemma thus follows.
\end{proof}

Now we are ready to prove Proposition~\ref{prop.min_eig}.
\begin{proof}[Proof of Proposition~\ref{prop.min_eig}]
We define three events
\begin{align*}
    &\mathcal{J}_1\defeq\left\{\cos\RFanglemin\geq 1-\Cdq \right\},\\
    &\mathcal{J}_2\defeq \left\{\left|\frac{1}{p_2}\min\eig(\HHt\HHt^T)-\min\eig(\Hinfnormalized)\right|\geq \qnTwo\right\},\\
    &\mathcal{J}_3\defeq \left\{\frac{1}{p_2}\min\eig(\HHt\HHt^T)\leq p_1J(n,p_1,p_2,d,q)\triangleq\frac{p_1}{16\pi d}\sqrt{\frac{\Cdq }{\log(4n)}}\right.\\
     &\qquad\qquad \left.-\left(\qnTwo\right)\right\}.
\end{align*}

\textbf{Step 1: prove $\mathcal{J}_1\cup \mathcal{J}_2\supseteq \mathcal{J}_3$.}

In order to prove $\mathcal{J}_1\cup \mathcal{J}_2\supseteq \mathcal{J}_3$, it is equivalent to prove $\mathcal{J}_1^c\cap \mathcal{J}_2^c\subseteq \mathcal{J}_3^c$. To that end, suppose $\mathcal{J}_1^c$ and $\mathcal{J}_2^c$ happen. Thus, we have
\begin{align}\label{eq.temp_041401}
    \log(1/\cos\RFanglemin)\geq & 1-\cos\RFanglemin\text{ (by Lemma~\ref{le.log_estimate})}\nonumber\\
    \geq &\Cdq \quad \text{ (by the event $\mathcal{J}_1^c$)}.
\end{align}
Thus, we have
\begin{align*}
    &\min\eig(\Hinfnormalized)\\
    \geq &\frac{1}{8\pi}\cdot \frac{p_1}{2d}\sqrt{\frac{\log(1/\cos\RFanglemin)}{\log(2n/\cos\RFanglemin)}}\text{ (by Lemma~\ref{le.Srikant})}\\
    =&\frac{p_1}{16\pi d}\sqrt{\frac{\log(1/\cos\RFanglemin)}{\log(2n)+\log(1/\cos \RFanglemin)}}\\
    \geq& \frac{p_1}{16\pi d}\sqrt{\frac{\Cdq }{\log(2n)+\Cdq }}\\
    &\text{ (by Eq.~\eqref{eq.temp_041401} and $\frac{a}{\log(2n)+a}$ is monotone increasing with respect to $a$)}\\
    \geq &\frac{p_1}{16\pi d}\sqrt{\frac{\Cdq }{\log(4n)}}\text{ (since $\log(2)\approx 0.7$, $\Cdq \leq\frac{\pi-1}{4\pi}\cdot \frac{1}{2}\leq \frac{1}{8}\leq \log 2$)}.
\end{align*}
Thus, we have
\begin{align*}
    \frac{1}{p_2}\min\eig(\HHt\HHt^T)\geq & \min\eig(\Hinfnormalized)-\left|\frac{1}{p_2}\min\eig(\HHt\HHt^T)-\min\eig(\Hinfnormalized)\right|\\
    &\text{ (by the triangle inequality)}\\
    > &\frac{p_1}{16\pi d}\sqrt{\frac{\Cdq }{\log(4n)}}-\left(\qnTwo\right)\text{ (by the event $\mathcal{J}_2^c$)}\\
    =&p_1J(n,p_1,p_2,d,q)\quad \text{ (by Eq.~\eqref{eq.def_J})},
\end{align*}
i.e., $\mathcal{J}_3^c$ must then occur. Thus, we have shown that $\mathcal{J}_1^c\cap \mathcal{J}_2^c\subseteq \mathcal{J}_3^c$, which implies that $\mathcal{J}_1\cup \mathcal{J}_2\supseteq \mathcal{J}_3$.

\textbf{Step 2: estimate $\mathcal{J}_3$}

We have
\begin{align*}
    \prob_{\XX,\Vzero,\Wzero}[\mathcal{J}_3]\leq & \prob_{\XX,\Vzero,\Wzero}[\mathcal{J}_1\cup \mathcal{J}_2]\text{ (because $\mathcal{J}_1\cup \mathcal{J}_2\supseteq \mathcal{J}_3$)}\\
    \leq & \prob_{\XX,\Vzero,\Wzero}[\mathcal{J}_1]+\prob_{\XX,\Vzero,\Wzero}[\mathcal{J}_2]\text{ (by the union bound)}\\
    =&\prob_{\XX,\Vzero}[\mathcal{J}_1]+\prob_{\XX,\Vzero,\Wzero}[\mathcal{J}_2]\text{ (as $\mathcal{J}_1$ is independent of $\Wzero$)}\\
    \leq & \frac{7}{q^2}\text{ (by Lemma~\ref{le.min_angle} and Lemma~\ref{le.diff_H_Hinf})}. 
\end{align*}
The result of Proposition~\ref{prop.min_eig} thus follows.
\end{proof}

In the rest of this section, we prove  Lemma~\ref{le.Srikant}, Lemma~\ref{le.min_angle}, and Lemma~\ref{le.diff_H_Hinf}.
\subsection{Proof of Lemma~\ref{le.Srikant}}\label{app.proof_lemma_srikant}
\begin{proof}
For simplicity of notation, we define $\bm{a}_i\in \mathds{R}^{p_1}$ as
\begin{align*}
    \bm{a}_i\defeq \frac{\hXiRF}{\|\hXiRF\|_2}\text{ for all }i=1,2,\cdots, n.
\end{align*}
Let $\bm{a}_i^{\otimes k}\in \mathds{R}^{p_1 k}$ (a column vector with $p_1 k$ elements) denote the $k$-time Kronecker product of the vector $\bm{a}_i$ with itself. We define
\begin{align*}
    &\mathbf{A}\defeq [\bm{a}_1\ \bm{a}_2\ \cdots\ \bm{a}_n]\in \mathds{R}^{p_1\times n},\\
    &\mathbf{A}^{(k)}\defeq \left[\bm{a}_1^{\otimes k}\ \bm{a}_2^{\otimes k}\ \cdots\ \bm{a}_n^{\otimes k}\right]\in \mathds{R}^{(p_1 k)\times n},\\
    &\mathbf{B}^{(k)}\defeq \left(\mathbf{A}^{(k)}\right)^T\mathbf{A}^{(k)}.
\end{align*}
Thus, we have
\begin{align}\label{eq.temp_081801}
    \cos\RFangleij=\bm{a}_i^T\bm{a}_j.
\end{align}
By the definition of Kronecker product, we thus have\footnote{To help  readers understand the correctness of Eq.~\eqref{eq.temp_081701}, we give a toy example as follows. We have
\begin{align*}
    \left(\begin{bmatrix}a &b\end{bmatrix}\begin{bmatrix}c\\ d\end{bmatrix}\right)^2=(ac+bd)^2=a^2c^2+2abcd + b^2d^2.
\end{align*}
We also have
\begin{align*}
    \begin{bmatrix}a\\ b\end{bmatrix}^{\otimes 2}=\begin{bmatrix}aa\\
    ab\\
    ba\\
    bb\end{bmatrix},\  \begin{bmatrix}c\\ d\end{bmatrix}^{\otimes 2}=\begin{bmatrix}cc\\
    cd\\
    dc\\
    dd\end{bmatrix}\implies \left(\begin{bmatrix}a\\ b\end{bmatrix}^{\otimes 2}\right)^T\left(\begin{bmatrix}c\\ d\end{bmatrix}^{\otimes 2}\right)=a^2c^2+2abcd+b^2d^2.
\end{align*}
Thus, we have shown that $ \left(\begin{bmatrix}a &b\end{bmatrix}\begin{bmatrix}c\\ d\end{bmatrix}\right)^2= \left(\begin{bmatrix}a\\ b\end{bmatrix}^{\otimes 2}\right)^T\left(\begin{bmatrix}c\\ d\end{bmatrix}^{\otimes 2}\right)$.
}
\begin{align}\label{eq.temp_081701}
    \left(\bm{a}_i^T\bm{a}_j\right)^k=\left(\bm{a}_i^{\otimes k}\right)^T\left(\bm{a}_j^{\otimes k}\right).
\end{align}
Thus, by Lemma~\ref{le.taylor_of_kernel}, we have
\begin{align*}
    \cos(\RFangleij)\cdot\frac{\pi-\RFangleij}{2\pi}=&\frac{\cos\RFangleij}{4}+\frac{1}{2\pi}\sum_{k=0}^\infty\frac{(2k)!}{(k!)^2}\frac{4}{2k+1}\left(\frac{\cos\RFangleij}{2}\right)^{2k+2}\\
    =&\frac{\bm{a}_i^T\bm{a}_j}{4}+\frac{1}{2\pi}\sum_{k=0}^\infty\frac{(2k)!}{(k!)^2}\frac{4}{2k+1}\left(\frac{\bm{a}_i^T\bm{a}_j}{2}\right)^{2k+2}\\
    =&\frac{\bm{a}_i^T\bm{a}_j}{4}+\frac{1}{2\pi}\sum_{k=0}^\infty\frac{(2k)!}{(k!)^2}\frac{4}{2k+1}\left(\frac{1}{2}\right)^{2k+2}\left(\bm{a}_i^{\otimes 2k+2}\right)^T\bm{a}_j^{\otimes 2k+2}.
\end{align*}
Using Eq.~\eqref{eq.def_Hinf}, we then have
\begin{align*}
    \Hinfnormalized=\frac{p_1}{2d}\left(\frac{\mathbf{A}^T\mathbf{A}}{4}+\frac{1}{2\pi}\sum_{k=0}^\infty\frac{(2k)!}{(k!)^2}\frac{4}{2k+1}\left(\frac{1}{2}\right)^{2k+2}\mathbf{B}^{(2k+2)}\right).
\end{align*}
Thus, we have
\begin{align}
    \min\eig (\Hinfnormalized)=&\min_{\bm{u}:\ \|\bm{u}\|_2=1}\bm{u}^T\Hinfnormalized\bm{u}\nonumber\\
    \geq &\frac{p_1}{2d}\cdot \frac{1}{2\pi}\sum_{k=0}^\infty\frac{(2k)!}{(k!)^2}\frac{4}{2k+1}\left(\frac{1}{2}\right)^{2k+2}\min_{\bm{u}:\ \|\bm{u}\|_2=1}\bm{u}^T\mathbf{B}^{(2k+2)}\bm{u}\nonumber\\
    =& \frac{p_1}{2d}\cdot \frac{1}{2\pi}\sum_{k=0}^\infty\frac{(2k)!}{(k!)^2}\frac{4}{2k+1}\left(\frac{1}{2}\right)^{2k+2}\min\eig(\mathbf{B}^{(2k+2)}).\label{eq.temp_081802}
\end{align}
Notice that all diagonal elements of $\mathbf{B}^{(2k+2)}$ equal to $1$. Thus, by Gershgorin circle theorem \citep{bell1965gershgorin}, we have
\begin{align}\label{eq.temp_081803}
    \min\eig(\mathbf{B}^{(2k+2)})\geq 1 - \max_{i}\sum_{j\neq i}\mathbf{B}^{(2k+2)}_{ij},
\end{align}
where $\mathbf{B}^{(2k+2)}_{ij}$ denotes the $(i,j)$-th element of $\mathbf{B}^{(2k+2)}$. Notice that
\begin{align*}
    \max_{i}\sum_{j\neq i}\mathbf{B}^{(2k+2)}_{ij}=&\max_i \sum_{j\neq i}\left(\bm{a}_i^{\otimes 2k+2}\right)^T\left(\bm{a}_j^{\otimes 2k+2}\right)\\
    =&\max_i\sum_{j\neq i}\left(\cos\RFangleij\right)^{2k+2}\text{ (by Eq.~\eqref{eq.temp_081701} and Eq.~\eqref{eq.temp_081801})}\\
    \leq & (n-1)\left(\cos\RFanglemin\right)^{2k+2}.
\end{align*}
Note that, when $k\geq k^*\defeq  \frac{\log(2n-2)}{2\log(1/\cos\RFanglemin)}-1$, we have
\begin{align*}
    &2k+2\geq \frac{\log(2n-2)}{\log(1/\cos\RFanglemin)}\\
    \implies & (2k+2)\log(1/\cos\RFanglemin)\geq \log(2n-2)\\
    \implies &(n-1)\left(\cos\RFanglemin\right)^{2k+2}\leq \frac{1}{2}.
\end{align*}
Therefore, we have
\begin{align*}
    \max_{i}\sum_{j\neq i}\mathbf{B}^{(2k+2)}_{ij}\leq \frac{1}{2},\text{ for all } k\geq k^*.
\end{align*}
By Eq.~\eqref{eq.temp_081802} and Eq.~\eqref{eq.temp_081803}, we thus have
\begin{align*}
    \min\eig(\Hinfnormalized)\geq & \frac{p_1}{8d\pi}\sum_{k\geq k^*}\frac{(2k)!}{(k!)^2}\frac{4}{2k+1}\left(\frac{1}{2}\right)^{2k+2}\\
    =&\frac{p_1}{8d\pi}\sum_{k\geq \lceil k^*\rceil}\frac{(2k-1)!!}{(2k)!!}\frac{1}{2k+1}\\
    \geq & \frac{p_1}{8d\pi}\sum_{k\geq \lceil k^*\rceil}\frac{1}{\sqrt{\pi\left(k+\frac{4}{\pi}-1\right)}}\frac{1}{2k+1}\text{ (by Lemma~\ref{le.wallis_inequality})}\\
    \geq & \frac{p_1}{8d\pi}\int_{k^*+1}^\infty \frac{1}{\sqrt{\pi\left(x+\frac{4}{\pi}-1\right)}}\frac{1}{2x+1}dx\\
    \geq & \frac{p_1}{8d\pi}\int_{k^*+1}^\infty \frac{1}{2\sqrt{\pi}}(x+1)^{-\frac{3}{2}}dx\text{ (notice that $\frac{4}{\pi}-1\approx 0.27\leq 1$)}\\
    =&\frac{p_1}{8d\pi}\frac{1}{\sqrt{\pi}}\frac{1}{\sqrt{k^*+2}}.
\end{align*}
Notice that $\frac{1}{\sqrt{\pi}}\geq \frac{1}{2}$ and
\begin{align*}
    k^*+2=\frac{\log(2n-2)}{2\log(1/\cos \RFanglemin)}+1\leq \frac{\log (2n)}{\log(1/\cos\RFanglemin)}+1=\frac{\log(2n/\cos\RFanglemin)}{\log(1/\cos\RFanglemin)}.
\end{align*}
We thus have
\begin{align*}
    \min\eig(\Hinfnormalized)\geq \frac{p_1}{16d\pi}\sqrt{\frac{\log(1/\cos\RFanglemin)}{\log(2n/\cos\RFanglemin)}}.
\end{align*}
\end{proof}

\subsection{Proof of Lemma~\ref{le.min_angle}}\label{app.le_min_angle}
We first show some useful lemmas.

\begin{lemma}\label{le.temp_041001}
For any $\theta\in[0,\ \pi]$, we have
\begin{align*}
    &1-\frac{\sin\theta+(\pi-\theta)\cos\theta}{\pi}\geq \frac{\pi-1}{2\pi}\sin^2\theta,\\
    &\lim_{\theta\to0^+}\frac{1-\frac{\sin\theta+(\pi-\theta)\cos\theta}{\pi}}{\sin^2\theta}=\frac{1}{2}.
\end{align*}
\end{lemma}
\begin{proof}
To prove the first part, we have
\begin{align*}
    \frac{\sin\theta+(\pi-\theta)\cos\theta}{\pi}\leq &\frac{\sin\theta+(\pi-\theta)\sqrt{1-\sin^2\theta}}{\pi}\\
    &\text{ (although $\cos\theta$ could be negative, we always have $\cos\theta\leq \sqrt{1-\sin^2\theta}$)}\\
    \leq &\frac{\sin\theta+(\pi-\theta)\sqrt{1-\sin^2\theta+\frac{1}{4}\sin^4\theta}}{\pi}\\
    =&\frac{\sin\theta+(\pi-\theta)\left(1-\frac{1}{2}\sin^2\theta\right)}{\pi}\\
    \leq & \frac{\sin\theta+(\pi-\sin\theta)\left(1-\frac{1}{2}\sin^2\theta\right)}{\pi}\text{ (because $\sin\theta\leq \theta$)}\\
    =&1-\frac{\pi-\sin\theta}{2\pi}\sin^2\theta\\
    \leq & 1- \frac{\pi-1}{2\pi}\sin^2\theta\text{ (because $\sin\theta \leq 1$)},
\end{align*}
i.e.,
\begin{align*}
    1-\frac{\sin\theta+(\pi-\theta)\cos\theta}{\pi}\geq \frac{\pi-1}{2\pi}\sin^2\theta.
\end{align*}
To prove the second part, we have
\begin{align*}
    \lim_{\theta\to0^+}\frac{1-\frac{\sin\theta+(\pi-\theta)\cos\theta}{\pi}}{\sin^2\theta}= &\lim_{\theta\to0^+}\frac{\frac{\partial}{\partial\theta}\left(1-\frac{\sin\theta+(\pi-\theta)\cos\theta}{\pi}\right)}{\frac{\partial \sin^2\theta}{\partial \theta}}\text{ (by L'Hospital's rule)}\\
    =&\lim_{\theta\to0^+}\frac{-\cos\theta+(\pi-\theta)\sin\theta+\cos\theta}{2\pi\sin\theta\cos\theta}\\
    =&\lim_{\theta\to0^+}\frac{\pi-\theta}{2\pi\cos\theta}\\
    =&\frac{1}{2}.
\end{align*}
\end{proof}

\begin{lemma}\label{le.frac_converge}
Consider $a\geq 0$ and $b>0$. Let $\delta\defeq |b-1|$. If $\delta\in [0,\ 0.5]$, we then have
\begin{align*}
    a-a\delta\leq \frac{a}{b}\leq a+2a\delta.
\end{align*}
Therefore, for any $c\in \mathds{R}$, we have
\begin{align*}
    \left|\frac{a}{b}-c\right|\leq |a-c|+2a\delta.
\end{align*}
Further, if we know the upper bound of $a$, we have the following conclusion: (i) if $a\leq 1$, we must have $\frac{a}{b}\leq a + 2\delta$; (ii) if $a\leq 1.5$, we must have $\left|\frac{a}{b}-c\right|\leq |a-c|+3a$.
\end{lemma}
\begin{proof}
We have
\begin{align*}
    \frac{a}{b}\leq& a\frac{1}{1-\delta}\text{ (because $b\geq 1-|1-b|=1-\delta$)}\\
    \leq &a \frac{1+(1-2\delta)\delta}{1-\delta}\text{ (by $(1-2\delta)\geq 0$ because $\delta\in[0,\ 0.5]$)}\\
    =& a \frac{1+\delta-2\delta^2}{1-\delta}\\
    =&a \frac{(1-\delta)(1+2\delta)}{1-\delta}\\
    =&a+2a\delta.
\end{align*}
We also have
\begin{align*}
    \frac{a}{b}\geq & a\frac{1}{1+\delta}\text{ (because $b\leq 1+|1-b|=1+\delta$)}\\
    \geq & a\frac{1-\delta^2}{1+\delta}\\
    =& a-a\delta.
\end{align*}
The result of this lemma thus follows.
\end{proof}

\begin{lemma}\label{le.temp_082501}
If the condition in Eq.~\eqref{eq.temp_041202} is satisfied, then
\begin{align*}
    \frac{1- 2\Cdq +\frac{2dnq}{\sqrt{p_1}}}{1-\frac{4dnq\sqrt{2d}}{\sqrt{p_1}}-\frac{4d^2n^2q^2}{p_1}}\leq 1-\Cdq .
\end{align*}
\end{lemma}
\begin{proof}
By Eq.~\eqref{eq.temp_041202} and the definition of $C(n,d,q)$ in Eq.~\eqref{eq.def_Cndq}, we have
\begin{align}
    &\frac{10dnq\sqrt{2d}}{\sqrt{p_1}}\leq \frac{\pi-1}{4\pi}\cdot \frac{1}{2}\leq \frac{1}{8}\leq \frac{1}{2}\label{eq.temp_012201}\\
    \implies &\frac{dnq\sqrt{2d}}{\sqrt{p_1}}\leq \frac{1}{20},\ \text{and }\frac{d^2n^2q^2\cdot 2d}{p_1}\leq\left(\frac{1}{20}\right)^2= \frac{1}{400}\nonumber\\
    \implies & \frac{4dnq\sqrt{2d}}{\sqrt{p_1}}+\frac{4d^2n^2q^2}{p_1}\leq \frac{1}{5}+\frac{1}{200d}\leq 0.5.\label{eq.temp_082301}
\end{align}
We also have
\begin{align}
    &\frac{2dnq}{\sqrt{p_1}}+\frac{8dnq\sqrt{2d}}{\sqrt{p_1}}+\frac{8d^2n^2q^2}{p_1}\nonumber\\
    \leq & \frac{dnq\sqrt{2d}}{\sqrt{p_1}}+\frac{8dnq\sqrt{2d}}{\sqrt{p_1}}+\frac{1}{2}\sqrt{\frac{8d^2n^2q^2}{p_1}}\text{ (because $2\leq \sqrt{2d}$ and $\frac{8d^2n^2q^2}{p_1}\leq \frac{1}{100d}\leq \frac{1}{4}$)}\nonumber\\
    =&\frac{9dnq\sqrt{2d}}{\sqrt{p_1}}+\frac{dnq\sqrt{2}}{\sqrt{p_1}}\nonumber\\
    \leq & \frac{10dnq\sqrt{2d}}{\sqrt{p_1}}\nonumber\\
    \leq & \Cdq \text{ (by Eq.~\eqref{eq.temp_041202})}.\label{eq.temp_082303}
\end{align}
Thus, we have
\begin{align}
    &\frac{2dnq}{\sqrt{p_1}}\leq \Cdq \text{ (by Eq.~\eqref{eq.temp_041202})}\nonumber\\
    \implies &1- 2\Cdq +\frac{2dnq}{\sqrt{p_1}}\in [0, 1]\quad \text{ (since $\Cdq\leq \frac{1}{8}$ by Eq.~\eqref{eq.temp_012201})}.\label{eq.temp_082302}
\end{align}
By Eq.~\eqref{eq.temp_082301}, Eq.~\eqref{eq.temp_082302} and applying Lemma~\ref{le.frac_converge}(i) (where $a=1- 2\Cdq +\frac{2dnq}{\sqrt{p_1}}$, $b=1-\frac{4dnq\sqrt{2d}}{\sqrt{p_1}}-\frac{4d^2n^2q^2}{p_1}$, and $\delta=\frac{4dnq\sqrt{2d}}{\sqrt{p_1}}+\frac{4d^2n^2q^2}{p_1}$), we thus have
\begin{align*}
    &\frac{1- 2\Cdq +\frac{2dnq}{\sqrt{p_1}}}{1-\frac{4dnq\sqrt{2d}}{\sqrt{p_1}}-\frac{4d^2n^2q^2}{p_1}}\\
    \leq & 1-2\Cdq  + \frac{2dnq}{\sqrt{p_1}}+ \frac{8dnq\sqrt{2d}}{\sqrt{p_1}}+\frac{8d^2n^2q^2}{p_1}\\
    \leq & 1-\Cdq \text{ (by Eq.~\eqref{eq.temp_082303})}.
\end{align*}
\end{proof}

% \begin{lemma}\label{le.temp_041101}
% For any $a\in[0,\ 1]$ and $b\in[0,\ 0.5]$, we must have
% \begin{align*}
%     &\frac{1-a}{1-b}\leq 1-a+2b,\quad \frac{1+a}{1-b}\leq 1+2a+2b,\quad \frac{1-a}{1+b}\geq 1-a-b.
% \end{align*}
% \end{lemma}
% \begin{proof}
% We have
% \begin{align*}
%     \frac{1-a}{1-b}\leq &(1-a)\frac{1+(1-2b)b}{1-b}\text{ (by $(1-2b)b\geq 0$ because $b\in [0,\ 0.5]$)}\\
%     =& (1-a)\frac{1+b-2b^2}{1-b}\\
%     =& (1-a)(1+2b)\\
%     =& 1-a + (1-a)2b\\
%     \leq & 1-a + 2b\text{ (by $1-a\geq 0$ because $a\in[0,\ 1]$)}.
% \end{align*}
% \begin{align*}
%     \frac{1+a}{1-b}\leq & (1+a)\frac{1+(1-2b)b}{1-b}\text{ (by $(1-2b)b\geq 0$ because $b\in [0,\ 0.5]$)}\\
%     = & (1+a)\frac{1+b-2b^2}{1-b}\\
%     =& (1+a)(1+2b)\\
%     = & 1+a+2b+2ab\\
%     \leq & 1+2a+2b\text{ (because $b\in[0,\ 0.5]$)}.
% \end{align*}
% \begin{align*}
%     \frac{1-a}{1+b}=&(1-a)\frac{1}{1+b}\geq  (1-a)\frac{1-b^2}{1+b}= (1-a)(1-b)= 1-a-b+ab\geq 1-a-b.
% \end{align*}
% \end{proof}

% \begin{corollary}\label{coro.frac_estimate}
% For any $a$ and $b$ that $|a|\in[0,\ 1]$ and $|b|\in [0,\ 0.5]$, we must have
% \begin{align*}
%     \left|\frac{1+a}{1+b}-1\right|\leq 2|a|+2|b|.
% \end{align*}
% \end{corollary}
% \begin{proof}
% We have
% \begin{align*}
%     \frac{1+a}{1+b}\leq \frac{1+|a|}{1-|b|},\quad \frac{1+a}{1+b}\geq \frac{1-|a|}{1+|b|}.
% \end{align*}
% The result of this corollary thus follows by Lemma~\ref{le.temp_041101}.
% \end{proof}

\begin{lemma}\label{le.single_RFnormij}
Given $\XX$, for any $m>0$,
\begin{align*}
    \prob_{\Vzero}\left\{\left|\RFnormij - \frac{p_1}{2d}\right|\geq 2m\sqrt{2p_1 d}+2 m^2 d \right\}\leq \frac{2}{m^2}.
\end{align*}
In other words, given any $\xsmall,\bm{z}\in \sd$ and for any $q>0$,
\begin{align*}
    \prob_{\Vzero}\left\{\left|\|\hxRF\|_2\cdot \|\hzRF\|_2 - \frac{p_1}{2d}\right|\geq 2m\sqrt{2p_1 d}+2 m^2 d \right\}\leq \frac{2}{m^2}.
\end{align*}
\end{lemma}
\begin{proof}
Define
\begin{align*}
    Q_k^i\defeq  (\XX_i^T\Vzero[k])(\XX_i^T\Vzero[k])\indicator{\XX_i^T\Vzero[k]> 0}.
\end{align*}
By Eq.~\eqref{eq.def_hxRF}, we have
\begin{align*}
    \frac{1}{p_1}\|\hXiRF\|_2^2=\frac{1}{p_1}(\hXiRF)^T\cdot \hXiRF=\frac{1}{p_1}\sum_{k=1}^{p_1}Q_k^i.
\end{align*}
Note that
\begin{align*}
    |Q_k^i|\leq \|\XX_i\|_2\cdot \|\Vzero[k]\|_2\cdot \|\XX_i\|_2\cdot\|\Vzero[k]\|=1\text{ (by Assumption~\ref{as.normalize})}.
\end{align*}
Further, note that
\begin{align*}
    \expectation_{\Vzero}[Q_1^i]=&\int_{\sd}(\XX_i^T\bm{v})(\XX_i^T\bm{v})\indicator{\XX_i^T\bm{v}>0}d\vdensity(\bm{v})\text{ (by Eq.~\eqref{eq.def_hxRF})}\\
    =&\frac{\sin 0  + \pi\cos 0}{2d\pi}\text{ (by Lemma~\ref{le.RF_kernel})}\\
    =&\frac{1}{2d}.
\end{align*}
By Lemma~\ref{le.chebyshev}, we thus have
\begin{align}
    \prob_{\Vzero}\left\{\left|\frac{1}{p_1}\|\hXiRF\|_2^2-\frac{1}{2d}\right|\geq \frac{m}{\sqrt{p_1}}\right\}\leq \frac{1}{m^2}.\label{eq.temp_012202}
\end{align}
Notice that
\begin{align}
    \left|\|\hXiRF\|_2^2-\frac{p_1}{2d}\right|=\left(\|\hXiRF\|_2+\sqrt{\frac{p_1}{2d}}\right)\cdot \left|\|\hXiRF\|_2-\sqrt{\frac{p_1}{2d}}\right|\geq \sqrt{\frac{p_1}{2d}}\cdot \left|\|\hXiRF\|_2-\sqrt{\frac{p_1}{2d}}\right|.\label{eq.temp_012203}
\end{align}
Combining Eq.~\eqref{eq.temp_012202} and Eq.~\eqref{eq.temp_012203}, we then have
\begin{align*}
    \prob_{\Vzero}\left\{\left|\|\hXiRF\|_2-\sqrt{\frac{p_1}{2d}}\right|\geq m\sqrt{2d}\right\}\leq \prob_{\Vzero}\left\{\left|\|\hXiRF\|_2^2-\frac{p_1}{2d}\right|\geq m\sqrt{p_1}\right\} \leq \frac{1}{m^2}.
\end{align*}
Finally, notice that
\begin{align*}
    &\left|\RFnormij - \frac{p_1}{2d}\right|\\
    =&\Big|\|\hXiRF\|_2\left(\|\hXjRF\|_2-\sqrt{\frac{p_1}{2d}}\right)+\|\hXjRF\|_2\left(\|\hXiRF\|_2-\sqrt{\frac{p_1}{2d}}\right)\\
    &-\left(\|\hXjRF\|_2-\sqrt{\frac{p_1}{2d}}\right)\left(\|\hXiRF\|_2-\sqrt{\frac{p_1}{2d}}\right)\Big|\\
    \leq &\sqrt{p_1}\left|\|\hXjRF\|_2-\sqrt{\frac{p_1}{2d}}\right|+\sqrt{p_1}\left|\|\hXiRF\|_2-\sqrt{\frac{p_1}{2d}}\right|\\
    +&\left|\|\hXjRF\|_2-\sqrt{\frac{p_1}{2d}}\right|\cdot \left|\|\hXiRF\|_2-\sqrt{\frac{p_1}{2d}}\right|\text{ (by the triangle inequality and Lemma~\ref{le.bound_hx})}.
\end{align*}
Thus, we have
\begin{align*}
    &\left\{\left|\|\hXiRF\|_2-\sqrt{\frac{p_1}{2d}}\right|\geq m\sqrt{2d}\right\} \cup \left\{\left|\|\hXjRF\|_2-\sqrt{\frac{p_1}{2d}}\right|\geq m\sqrt{2d}\right\}\\
    \supseteq & \left\{\left|\RFnormij - \frac{p_1}{2d}\right|\geq 2m\sqrt{2p_1 d}+2 m^2 d\right\}.
\end{align*}
Applying the union bound, we thus have
\begin{align*}
    \prob_{\Vzero}\left\{\left|\RFnormij - \frac{p_1}{2d}\right|\geq 2m\sqrt{2p_1 d}+2 m^2 d \right\}\leq \frac{2}{m^2}.
\end{align*}
\end{proof}

Now we are ready to prove Lemma~\ref{le.min_angle}.
\begin{proof}[Proof of Lemma~\ref{le.min_angle}]
Define three events as
\begin{align*}
    &\mathcal{J}_{1,i,j}\defeq \left\{\frac{2d}{p_1}(\hXiRF)^T\cdot \hXjRF\geq 1- 2\Cdq +\frac{2dnq}{\sqrt{p_1}}\right\},\\
    &\mathcal{J}_{2,i,j}\defeq \left\{\RFnormij\leq \frac{p_1}{2d}-2nq\sqrt{2p_1d}-2n^2q^2d\right\},\\
    &\mathcal{J}_{3,i,j}\defeq \left\{\cos\RFangleij\geq 1-2\Cdq  \right\}.
\end{align*}

We take a few steps as follows to finish the proof.

\textbf{Step 1: estimate $\mathcal{J}_{1,i,j}$.}

Define
\begin{align*}
    Q_k^{i,j}\defeq (\XX_i^T\Vzero[k])(\XX_j^T\Vzero[k])\indicator{\XX_i^T\Vzero[k]> 0,\ \XX_j^T\Vzero[k]>0},\quad k=1,2,\cdots,p_1.
\end{align*}
By Eq.~\eqref{eq.def_hxRF} and the definition of $Q_k^{i,j}$, we have
\begin{align}\label{eq.temp_041201}
    (\hXiRF)^T\cdot \hXjRF=\sum_{k=1}^{p_1}Q_k^{i,j}.
\end{align}
Note that
\begin{align*}
    |Q_k^{i,j}|\leq \|\XX_i\|_2\cdot \|\Vzero[k]\|_2\cdot \|\XX_j\|_2\cdot\|\Vzero[k]\|_2=1\text{ (by Assumption~\ref{as.normalize} and Lemma~\ref{le.matrix_norm})}.
\end{align*}
By Lemma~\ref{le.chebyshev}, we then have
\begin{align*}
    \prob_{\Vzero}\left\{\left|\frac{1}{p_1}\sum_{k=1}^{p_1}Q_k^{i,j}-\expectation_{\Vzero} [Q_1^{i,j}]\right|\geq \frac{m}{\sqrt{p_1}}\right\}\leq \frac{1}{m^2}.
\end{align*}
Let $\theta_{i,j}=\arccos(\XX_i^T\XX_j)\in[0,\pi]$ denote the angle between $\XX_i$ and $\XX_j$, where $i\neq j$ and $ i,j\in\{1,2,\cdots, n\}$.
Notice that
\begin{align*}
    \expectation_{\Vzero}[Q_1^{i,j}]=&\int_{\sd}(\XX_i^T\bm{v})(\XX_j^T\bm{v})\indicator{\XX_i^T\bm{v}> 0, \XX_j^T\bm{v}>0}d\vdensity(\bm{v}) \text{ (by Eq.~\eqref{eq.def_hxRF})}\\
    =&\frac{\sin\theta_{i,j}+(\pi-\theta_{i,j})\cos\theta_{i,j}}{2d\pi}\text{ (by Lemma~\ref{le.RF_kernel})}\\
    \leq & \frac{1}{2d}\left(1-\frac{\pi-1}{2\pi}\sin^2\theta_{i,j}\right)\text{ (by Lemma~\ref{le.temp_041001})}.
\end{align*}
Thus, we have
\begin{align*}
    \prob_{\Vzero}\left\{\frac{1}{p_1}\sum_{k=1}^{p_1}Q_k^{i,j}\geq \frac{1}{2d}\left(1-\frac{\pi-1}{2\pi}\sin^2\theta_{i,j}\right)+\frac{m}{\sqrt{p_1}}\right\}\leq \frac{1}{m^2}.
\end{align*}
For any $\alpha\in [0, 1]$, we have
\begin{align*}
    &\prob_{\XX}\left\{\sin^2\theta_{i,j}\leq \alpha\right\}\\
    =&\prob_{\XX}\left\{\theta_{i,j}\leq \arcsin\left(\sqrt{\alpha}\right)\text{ OR }\pi-\theta_{i,j}\leq\arcsin\left(\sqrt{\alpha}\right) \right\}\\
    \leq &\prob_{\XX}\left\{\theta_{i,j}\leq \arcsin\left(\sqrt{\alpha}\right)\right\}+\prob_{\XX}\left\{\pi-\theta_{i,j}\leq \arcsin\left(\sqrt{\alpha}\right)\right\}\text{ (by the union bound)}\\
    =&I_{\alpha}\left(\frac{d-1}{2},\ \frac{1}{2}\right)\text{ (area of two caps, by Lemma~\ref{le.original_cap} and Assumption~\ref{as.normalize})}\\
    \leq &  \frac{2\sqrt{d}\alpha^{\frac{d-1}{2}}}{(d-1)\sqrt{1-\alpha}}\text{ (by Lemma~\ref{le.estimate_Ix} and Lemma~\ref{le.bound_B})}.
\end{align*}
Further, because
\begin{align*}
    &\left\{\sin^2\theta_{i,j}> \alpha\right\}\cap \left\{ \frac{1}{p_1}\sum_{k=1}^{p_1}Q_k^{i,j}< \frac{1}{2d}\left(1-\frac{\pi-1}{2\pi}\sin^2\theta_{i,j}\right)+\frac{m}{\sqrt{p_1}}\right\}\\
    \subseteq & \left\{\frac{2d}{p_1}\sum_{k=1}^{p_1}Q_k^{i,j}< 1- \frac{\pi-1}{2\pi}\alpha+\frac{2dm}{\sqrt{p_1}}\right\},
\end{align*}
we have
\begin{align*}
    &\left\{\sin^2\theta_{i,j}\leq \alpha\right\} \cup \left\{\frac{1}{p_1}\sum_{k=1}^{p_1}Q_k^{i,j}\geq \frac{1}{2d}\left(1-\frac{\pi-1}{2\pi}\sin^2\theta_{i,j}\right)+\frac{m}{\sqrt{p_1}}\right\}\\
    \supseteq &\left\{\frac{2d}{p_1}\sum_{k=1}^{p_1}Q_k^{i,j}\geq 1- \frac{\pi-1}{2\pi}\alpha+\frac{2dm}{\sqrt{p_1}}\right\}.
\end{align*}
Thus, by the union bound and Eq.~\eqref{eq.temp_041201}, we have
\begin{align}\label{eq.temp_041203}
    \prob_{\Vzero,\XX}\left\{\frac{2d}{p_1}(\hXiRF)^T\cdot \hXjRF\geq 1- \frac{\pi-1}{2\pi}\alpha+\frac{2dm}{\sqrt{p_1}}\right\}\leq \frac{1}{m^2}+\frac{2\sqrt{d}\alpha^{\frac{d-1}{2}}}{(d-1)\sqrt{1-\alpha}}.
\end{align}
By letting
\begin{align*}
    \alpha = \min\left\{\frac{1}{2},\ \left(\frac{(d-1)^2}{8d}\right)^{\frac{1}{d-1}}\left(qn\right)^{-\frac{4}{d-1}}\right\},\ \text{and } m=qn,
\end{align*}
we have
\begin{align*}
    \frac{2\sqrt{d}\alpha^{\frac{d-1}{2}}}{(d-1)\sqrt{1-\alpha}}\leq& \frac{2\sqrt{2}\sqrt{d}\alpha^{\frac{d-1}{2}}}{(d-1)}\text{ (because $\alpha\leq \frac{1}{2}$)}\\
    \leq & \frac{2\sqrt{2}\sqrt{d}\sqrt{\frac{(d-1)^2}{8d}}\frac{1}{q^2 n^2}}{d-1}\text{ (because $\alpha\leq \left(\frac{(d-1)^2}{8d}\right)^{\frac{1}{d-1}}\left(qn\right)^{-\frac{4}{d-1}}$)}\\
    =&\frac{1}{q^2n^2}.
\end{align*}
Thus, by Eq.~\eqref{eq.temp_041203}, we have
\begin{align}\label{eq.temp_041211}
    \prob_{\Vzero,\XX}[\mathcal{J}_{1,i,j}]\leq \frac{2}{q^2n^2}.
\end{align}

\textbf{Step 2: estimate $\mathcal{J}_{2,i,j}$.}
By Lemma~\ref{le.single_RFnormij}, we have
\begin{align*}
    \prob_{\Vzero}\left\{\RFnormij\leq \frac{p_1}{2d}-2m\sqrt{2p_1d}-2m^2d\right\}\leq \frac{2}{m^2}.
\end{align*}
Letting $m=qn$, we then have
\begin{align}\label{eq.temp_041212}
    \prob_{\Vzero}[\mathcal{J}_{2,i,j}]\leq \frac{2}{q^2n^2}.
\end{align}

\textbf{Step 3: prove $\mathcal{J}_{3,i,j}\subseteq \mathcal{J}_{1,i,j}\cup \mathcal{J}_{2,i,j}$.}

In order to show $\mathcal{J}_{3,i,j}\subseteq \mathcal{J}_{1,i,j}\cup \mathcal{J}_{2,i,j}$, it suffices to show $\mathcal{J}_{3,i,j}^c\supseteq \mathcal{J}_{1,i,j}^c\cap \mathcal{J}_{2,i,j}^c$. When $\mathcal{J}_{1,i,j}^c\cap \mathcal{J}_{2,i,j}^c$ happens, we have
\begin{align*}
    &\frac{2d}{p_1}\RFnormij> 1-\frac{2d}{p_1}\cdot 2nq\sqrt{2p_1d}-\frac{2d}{p_1}\cdot 2n^2q^2d=1-\frac{4dnq\sqrt{2d}}{\sqrt{p_1}}-\frac{4d^2n^2q^2}{p_1},\\
    &\frac{2d}{p_1}(\hXiRF)^T\cdot \hXjRF< 1- 2\Cdq +\frac{2dnq}{\sqrt{p_1}}.
\end{align*}
Thus, we have
\begin{align*}
    \cos\RFangleij=&\frac{(\hXiRF)^T\cdot \hXjRF}{\RFnormij}\nonumber\\
    < &\frac{1- 2\Cdq +\frac{2dnq}{\sqrt{p_1}}}{1-\frac{4dnq\sqrt{2d}}{\sqrt{p_1}}-\frac{4d^2n^2q^2}{p_1}}\nonumber\\
    \leq & 1- \Cdq  \text{ (by Lemma~\ref{le.temp_082501})}
\end{align*}

% We have
% \begin{align}
%     \frac{8dnq\sqrt{2d}}{p_1}=&\frac{2dnq}{\sqrt{p_1}}\cdot\frac{4\sqrt{2d}}{\sqrt{p_1}}\nonumber\\
%     \leq & \frac{2dnq}{\sqrt{p_1}}\text{ (because $p_1\geq 32d$ by Eq.~\eqref{eq.temp_041202})}.\label{eq.temp_041210}
% \end{align}
% We also have
% \begin{align}\label{eq.temp_041206}
%     \frac{4dnq}{\sqrt{p_1}}\leq \frac{\pi-1}{8\pi}\text{ (because $p_1\geq \left(\frac{32\pi dnq}{\pi-1}\right)^2$ by Eq.~\eqref{eq.temp_041202})},
% \end{align}
% and
% \begin{align}
%     &\frac{4dnq}{\sqrt{p_1}}\leq \frac{1}{2}\frac{\pi-1}{2\pi}\left(\frac{(d-1)^2}{8d}\right)^{\frac{1}{d-1}}\left(qn\right)^{-\frac{4}{d-1}}\label{eq.temp_041207}\\
%     &\text{ (because $p_1\leq \left(\frac{16\pi dnq}{\pi-1}\right)^2\left(\frac{(d-1)^2}{8d}\right)^{-\frac{2}{d-1}}\left(qn\right)^{\frac{8}{d-1}}$ by Eq.~\eqref{eq.temp_041202})}.\nonumber
% \end{align}
% Combining Eq.~\eqref{eq.temp_041206} and Eq.~\eqref{eq.temp_041207}, we thus have
% \begin{align}\label{eq.temp_041208}
%     \frac{4dnq}{\sqrt{p_1}}\leq \frac{1}{2}\frac{\pi-1}{2\pi}\min\left\{\frac{1}{2},\ \left(\frac{(d-1)^2}{8d}\right)^{\frac{1}{d-1}}\left(qn\right)^{-\frac{4}{d-1}}\right\}.
% \end{align}
% Plugging Eq.~\eqref{eq.temp_041210} and Eq.~\eqref{eq.temp_041208} into Eq.~\eqref{eq.temp_041209}, we thus have
% \begin{align*}
%     \cos\RFangleij< 1-\frac{\pi-1}{4\pi}\min\left\{\frac{1}{2},\ \left(\frac{(d-1)^2}{8d}\right)^{\frac{1}{d-1}}\left(qn\right)^{-\frac{4}{d-1}}\right\},
% \end{align*}
i.e., the event $\mathcal{J}_{3,i,j}^c$ happens. To sum up, we have proven that $\mathcal{J}_{3,i,j}^c\supseteq \mathcal{J}_{1,i,j}^c\cap \mathcal{J}_{2,i,j}^c$, which implies $\mathcal{J}_{3,i,j}\subseteq \mathcal{J}_{1,i,j}\cup \mathcal{J}_{2,i,j}$.

\textbf{Step 4: estimate $\mathcal{J}_{3,i,j}$.}
We have
\begin{align}
    \prob_{\Vzero,\XX}[\mathcal{J}_{3,i,j}]\leq &  \prob_{\Vzero,\XX}[\mathcal{J}_{1,i,j}]+\prob_{\Vzero,\XX}[\mathcal{J}_{2,i,j}]\text{ (by $\mathcal{J}_{3,i,j}\subseteq \mathcal{J}_{1,i,j}\cup \mathcal{J}_{2,i,j}$ and the union bound)}\nonumber\\
    \leq & \frac{4}{q^2n^2}\text{ (by Eq.~\eqref{eq.temp_041211} and Eq.~\eqref{eq.temp_041212})}.\label{eq.temp_041213}
\end{align}

\textbf{Step 5: estimate }$\cos\RFanglemin$.
We have
\begin{align*}
    &\prob_{\Vzero,\XX}\left\{\cos\RFanglemin\geq 1-\Cdq \right\}\\
    =&\prob_{\Vzero,\XX}\left[\bigcup_{i\neq j}\mathcal{J}_{3,i,j}\right]\\
    \leq & n(n-1)\prob_{\Vzero,\XX}[\mathcal{J}_{3,i,j}]\text{ (by the union bound)}\\
    \leq & \frac{4}{q^2}\text{ (by Eq.~\eqref{eq.temp_041213})}.
\end{align*}
The result of Lemma~\ref{le.min_angle} thus follows.
\end{proof}

\subsection{Proof of Lemma~\ref{le.diff_H_Hinf}}\label{app.proof_le_diff_H_Hinf}

We first introduce two useful lemmas.
Define $\Hinf\in \mathds{R}^{n\times n}$ as
\begin{align*}
    \Hinf_{i,j}\defeq \RFnormij\cos(\RFangleij)\cdot\frac{\pi-\RFangleij}{2\pi}.
\end{align*}
\begin{lemma}\label{le.diff_1}
Given $\XX$ and $\Vzero$, for any $q>0$, we have
\begin{align*}
    \prob_{\Wzero}\left\{\max_{i,j}\left|\frac{1}{p_2}(\HHt\HHt^T)_{i,j}-\Hinf_{i,j}\right|\geq \frac{q np_1 }{\sqrt{p_2}}\right\}\leq \frac{1}{q^2}.
\end{align*}
Thus, we also have
\begin{align*}
    \prob_{\Wzero,\XX,\Vzero}\left\{\max_{i,j}\left|\frac{1}{p_2}(\HHt\HHt^T)_{i,j}-\Hinf_{i,j}\right|\geq \frac{q np_1 }{\sqrt{p_2}}\right\}\leq \frac{1}{q^2}.
\end{align*}
\end{lemma}
\begin{proof}
For notation simplicity, given any $i,j\in\{1,2,\cdots,n\}$, we define
\begin{align*}
    Q_k^{i,j}\defeq (\hXiRF)^T\hXjRF\indicator{(\hXiRF)^T\Wzero[k]>0,\ (\hXjRF)^T\Wzero[k]}\text{ for all }k\in \{1,2,\cdots,p_2\}.
\end{align*}
By Eq.~\eqref{eq.def_h3}, we thus have
\begin{align*}
    (\HHt\HHt^T)_{i,j}=(\hhXi)^T\hhXj=&\sum_{k=1}^{p_2}(\hXiRF)^T\hXjRF\indicator{(\hXiRF)^T\Wzero[k]>0,\ (\hXjRF)^T\Wzero[k]}\\
    =&\sum_{k=1}^{p_2}Q_k^{i,j}.
\end{align*}
By Lemma~\ref{le.NTK2_kernel} and recalling Eq.~\eqref{eq.def_Hinf}, we have
\begin{align*}
    \expectation_{\Wzero}[Q_k^{i,j}]=\Hinf_{i,j}.
\end{align*}
By Lemma~\ref{le.bound_hx} and Lemma~\ref{le.matrix_norm}, we have
\begin{align*}
    |Q_k^{i,j}|\leq \|\hXiRF\|_2\cdot \|\hXjRF\|_2\leq  p_1.
\end{align*}
Note that $Q_k^{i,j}$ are independent across $k$. By Lemma~\ref{le.chebyshev}, for any $m>0$, we have
\begin{align}\label{eq.temp_041502}
    \prob_{\Wzero}\left\{\left|\frac{1}{p_2}(\HHt\HHt^T)_{i,j}-\Hinf_{i,j}\right|\geq m\frac{p_1}{\sqrt{p_2}}\right\}\leq \frac{1}{m^2}.
\end{align}
The result of this lemma thus follows by letting $m=qn$ and the union bound, i.e.,
\begin{align*}
    \prob_{\Wzero}\left\{\max_{i,j}\left|\frac{1}{p_2}(\HHt\HHt^T)_{i,j}-\Hinf_{i,j}\right|\geq \frac{q np_1 }{\sqrt{p_2}}\right\}=&\prob_{\Wzero}\left\{\bigcup_{i,j}\left\{\left|\frac{1}{p_2}(\HHt\HHt^T)_{i,j}-\Hinf_{i,j}\right|\geq \frac{q np_1 }{\sqrt{p_2}}\right\}\right\}\\
    \leq & \sum_{i,j}\prob_{\Wzero}\left\{\left|\frac{1}{p_2}(\HHt\HHt^T)_{i,j}-\Hinf_{i,j}\right|\geq \frac{q n p_1}{\sqrt{p_2}}\right\}\\
    &\text{ (by the union bound)}\\
    \leq &\sum_{i,j}\frac{1}{q^2n^2}\text{ (by letting $m=qn$ in Eq.~\eqref{eq.temp_041502})}\\
    =&\frac{1}{q^2}.
\end{align*}
\end{proof}
% Replacing $\XX_i$ and $\XX_j$ by $\xsmall$ and $\bm{z}$ in Eq.~\eqref{eq.temp_041502}, we have the following corollary.
% \begin{corollary}\label{coro.card_to_arccos}
% For any $\xsmall, \bm{z}\in\sd$, for any $m>0$,
% \begin{align*}
%     \prob_{\Wzero}\left\{\left|(\hzRF)^T\hxRF \frac{\left|\Chzhx\right|}{p_2}-(\hzRF)^T\hxRF\frac{\pi-\arccos\left(\frac{(\hzRF)^T\hxRF}{\|\hzRF\|_2\|\hxRF\|_2}\right)}{2\pi}\right|\geq m\frac{p_1}{\sqrt{p_2}}\right\}\leq \frac{1}{m^2}.
% \end{align*}
% \end{corollary}

\begin{lemma}\label{le.diff_2}
Given $\XX$, for any $q>0$, we must have 
\begin{align*}
    \prob_{\Vzero}\left\{\max_{i,j}\left|\Hinf_{i,j}-\Hinfnormalized_{i,j}\right|\geq qn\sqrt{2p_1d}+q^2n^2d\right\}\leq \frac{2}{q^2}.
\end{align*}
Thus, we also have
\begin{align*}
    \prob_{\Vzero,\XX,\Wzero}\left\{\max_{i,j}\left|\Hinf_{i,j}-\Hinfnormalized_{i,j}\right|\geq qn\sqrt{2p_1d}+q^2n^2d\right\}\leq \frac{2}{q^2}.
\end{align*}
\end{lemma}
\begin{proof}
We have
\begin{align*}
    \left|\Hinf_{i,j}-\Hinfnormalized_{i,j}\right|=&\left|\left(\RFnormij-\frac{p_1}{2d}\right)\cos(\RFangleij)\cdot\frac{\pi-\RFangleij}{2\pi}\right|\text{ (by Eq.~\eqref{eq.def_Hinf})}\\
    \leq& \left|\RFnormij-\frac{p_1}{2d}\right|\cdot \left|\cos(\RFangleij)\right|\cdot \left|\frac{\pi-\RFangleij}{2\pi}\right|\text{ (by Lemma~\ref{le.matrix_norm})}\\
    \leq & \frac{1}{2} \left|\RFnormij-\frac{p_1}{2d}\right|\quad \text{ (since $0\leq \RFangleij\leq \frac{\pi}{2}$)}.
\end{align*}
The result of this lemma thus follows by letting $m=qn$ in Lemma~\ref{le.single_RFnormij} and the union bound, i.e.,
\begin{align*}
    &\prob_{\Vzero}\left\{\max_{i,j}\left|\Hinf_{i,j}-\Hinfnormalized_{i,j}\right|\geq qn\sqrt{2p_1d}+q^2n^2d\right\}\\
    \leq &  \prob_{\Vzero}\left\{\max_{i,j}\left|\RFnormij-\frac{p_1}{2d}\right|\geq 2qn\sqrt{2p_1d}+2q^2n^2d\right\}\\
    =& \prob_{\Vzero}\left\{\bigcup_{i,j}\left\{\left|\RFnormij-\frac{p_1}{2d}\right|\geq 2qn\sqrt{2p_1d}+2q^2n^2d\right\}\right\}\\
    \leq & \sum_{i,j}\prob_{\Vzero}\left\{\left|\RFnormij-\frac{p_1}{2d}\right|\geq 2qn\sqrt{2p_1d}+2q^2n^2d\right\}\text{ (by the union bound)}\\
    \leq & \sum_{i,j}\frac{2}{q^2 n^2}\text{ ( by letting $m=qn$ in Lemma~\ref{le.single_RFnormij})}\\
    =& \frac{2}{q^2}.
\end{align*}
\end{proof}

Now we are ready to prove Lemma~\ref{le.diff_H_Hinf}.
\begin{proof}[Proof of Lemma~\ref{le.diff_H_Hinf}]
By the triangle inequality, we have
\begin{align*}
    \left|\frac{1}{p_2}(\HHt\HHt^T)_{i,j}-\Hinfnormalized_{i,j}\right|\leq \left|\Hinf_{i,j}-\Hinfnormalized_{i,j}\right| + \left|\frac{1}{p_2}(\HHt\HHt^T)_{i,j}-\Hinf_{i,j}\right|.
\end{align*}
Thus, we have
\begin{align*}
    &\prob_{\XX,\Vzero,\Wzero}\left\{\max_{i,j}\left|\frac{1}{p_2}(\HHt\HHt^T)_{i,j}-\Hinfnormalized_{i,j}\right|\geq qn\sqrt{2p_1d}+q^2n^2d+\frac{qnp_1}{\sqrt{p_2}}\right\}\\
    \leq &\prob_{\XX,\Vzero,\Wzero}\left\{\left\{\max_{i,j}\left|\Hinf_{i,j}-\Hinfnormalized_{i,j}\right|\geq qn\sqrt{2p_1d}+q^2n^2d\right\}\right.\\
    &\qquad\left.\cup \left\{\max_{i,j}\left|\frac{1}{p_2}(\HHt\HHt^T)_{i,j}-\Hinf_{i,j}\right|\geq \frac{q np_1 }{\sqrt{p_2}}\right\}\right\}\\
    \leq & \prob_{\XX,\Vzero,\Wzero}\left\{\max_{i,j}\left|\Hinf_{i,j}-\Hinfnormalized_{i,j}\right|\geq qn\sqrt{2p_1d}+q^2n^2d\right\}\\
    &+\prob_{\XX,\Vzero,\Wzero}\left\{\max_{i,j}\left|\frac{1}{p_2}(\HHt\HHt^T)_{i,j}-\Hinf_{i,j}\right|\geq \frac{q np_1 }{\sqrt{p_2}}\right\} \text{ (by the union bound)}\\
    \leq & \frac{3}{q^2}\text{ (by Lemma~\ref{le.diff_1} and Lemma~\ref{le.diff_2})}.
\end{align*}
The result of Lemma~\ref{le.diff_H_Hinf} thus follows by Lemma~\ref{le.small_diff_eig} (where $k=n$).
\end{proof}

\section{Proof of Proposition~\ref{prop.term_C} and Proposition~\ref{prop.term_DB}}\label{app.main}
We first provide some useful lemmas.

\begin{lemma}\label{le.sin}
For any $\varphi \in[0,2\pi]$, we must have $\sin\varphi \leq \varphi$. For any $\varphi \in[0,\pi/2]$, we must have $\varphi \leq \frac{\pi}{2}\sin\varphi$.
\end{lemma}
\begin{proof}
See Lemma 41 of \cite{ju2021generalization}.
\end{proof}

\begin{lemma}\label{le.temp_083001}
For any $a_1,a_2\in[-1,1]$ that $|a_1-a_2|\leq 1$, we must have
\begin{align*}
    |\arccos(a_1)-\arccos(a_2)|\leq \frac{\sqrt{2}\pi}{2}\sqrt{|a_1-a_2|}.
\end{align*}
\end{lemma}
\begin{proof}
Without loss of generality, we assume $a_2\geq a_1$ and let $\delta\defeq a_2-a_1\in [0,\ 1]$. Because $\frac{\partial \arccos x}{\partial x}=-\frac{1}{\sqrt{1-x^2}}$, we have
\begin{align*}
    \frac{\partial(\arccos(a_1)-\arccos(a_1+\delta))}{\partial a_1}&=-\frac{1}{\sqrt{1-a_1^2}}+\frac{1}{\sqrt{1-(a_1+\delta)^2}}\\
    &\begin{cases}
    \leq 0, \text{ when }a_1\in[-1,\  -\frac{\delta}{2}]\\
    \geq 0, \text{ when $a_1\in[-\frac{\delta}{2},\ 1-\delta]$}
    \end{cases}.
\end{align*}
Thus, we know the largest value of $\arccos(a_1)-\arccos(a_1+\delta)$ can only be achieved at either $a_1=-1$ or $a_1=1-\delta$, i.e.,
\begin{align}\label{eq.temp_083001}
    \arccos(a_1)-\arccos(a_1+\delta)\leq \max\left\{\pi-\arccos(-1+\delta),\ \arccos(1-\delta)\right\}=\arccos(1-\delta).
\end{align}
(The last equality is because $\arccos(-x)=\pi-\arccos x$.)
It remains to show that $\arccos(1-\delta)\leq \frac{\sqrt{2}\pi}{2}\sqrt{\delta}$. To that end, it suffices to prove $\cos(\frac{\sqrt{2}\pi}{2}\sqrt{\delta})\leq 1-\delta$. Let $\theta\defeq \frac{\sqrt{2}\pi}{2}\sqrt{\delta}$, i.e., $\delta = \frac{2}{\pi^2}\theta^2$. When $\theta > \frac{\pi}{2}$, we have $\cos(\frac{\sqrt{2}\pi}{2}\sqrt{\delta})=\cos\theta < 0 < 1-\delta$ (since $\delta \in[0,\ 1]$). When $\theta\in[0,\ \frac{\pi}{2}]$, we have
\begin{align*}
    \cos(\frac{\sqrt{2}\pi}{2}\sqrt{\delta})=\cos\theta =\sqrt{1-\sin^2\theta}\leq \sqrt{1-\sin^2\theta+\frac{1}{4}\sin^4\theta}=&1-\frac{1}{2}\sin^2\theta\\
    \leq &1-\frac{1}{2}(\frac{2}{\pi}\theta)^2\text{ (by Lemma~\ref{le.sin})}\\
    =&1-\delta.
\end{align*}
Therefore, we have proven that $\arccos(1-\delta)\leq \frac{\sqrt{2}\pi}{2}\sqrt{\delta}$ for all $\delta\in[0,\ 1]$.
By Eq.~\eqref{eq.temp_083001}, the result of this lemma thus follows.
\end{proof}

\begin{lemma}\label{le.arccos_delta}
For any real number $a_1$, $a_2$, $\delta_1$, and $\delta_2$ such that $a_2\in[-1, 1]$, $a_2+\delta_2\in [-1, 1]$, and $|\delta_2|\leq 1$, we must have
\begin{align*}
    \left|(a_1+\delta_1)\frac{\pi-\arccos(a_2+\delta_2)}{2\pi}-a_1\frac{\pi-\arccos(a_2)}{2\pi}\right|\leq \frac{1}{2}|\delta_1|+\frac{\sqrt{2}|a_1|\sqrt{|\delta_2|}}{4}.
\end{align*}
\end{lemma}
\begin{proof}
Define
\begin{align*}
    b\defeq a_1\frac{\pi-\arccos(a_2+\delta_2)}{2\pi}.
\end{align*}
we have
\begin{align*}
    &\left|(a_1+\delta_1)\frac{\pi-\arccos(a_2+\delta_2)}{2\pi}-a_1\frac{\pi-\arccos(a_2)}{2\pi}\right|\\
    =&\left|(a_1+\delta_1)\frac{\pi-\arccos(a_2+\delta_2)}{2\pi}-b+b-a_1\frac{\pi-\arccos(a_2)}{2\pi}\right|\\
    \leq & \left|(a_1+\delta_1)\frac{\pi-\arccos(a_2+\delta_2)}{2\pi}-b\right|+\left|b-a_1\frac{\pi-\arccos(a_2)}{2\pi}\right|\\
    =&|\delta_1|\cdot\left| \frac{\pi-\arccos(a_2+\delta_2)}{2\pi}\right|+|a_1|\cdot \left|\frac{\arccos(a_2+\delta_2)-\arccos(a_2)}{2\pi}\right|\\
    \leq & \frac{1}{2}|\delta_1|+\frac{\sqrt{2}|a_1|\sqrt{|\delta_2|}}{4}\text{ (since $\arccos(\cdot)\in [0,\pi]$ and by Lemma~\ref{le.temp_083001})}.
\end{align*}
\end{proof}

\begin{lemma}\label{le.RF_kernel_bound}
For any $\theta\in [0,\ \pi]$, we have
\begin{align*}
    \frac{\sin\theta+(\pi-\theta)\cos\theta}{\pi}\in[0,\ 1].
\end{align*}
\end{lemma}
\begin{proof}
We have
\begin{align*}
    \frac{\partial (\sin\theta+(\pi-\theta)\cos\theta)}{\partial \theta}=-(\pi-\theta)\sin\theta\leq 0.
\end{align*}
Thus, $\sin\theta+(\pi-\theta)\cos\theta$ is monotone decreasing. The result of this lemma thus follows by plugging $\theta=0$ and $\theta=\pi$ into the expression.
\end{proof}

\begin{lemma}\label{le.arccos_to_arccos}
Recall the definition of $\KThree(\cdot)$ in Eq.~\eqref{eq.def_KThree} and the definition of $\Qpd$ in Eq.~\eqref{eq.def_Q}. When $p_1$ is large enough such that $9d\cdot \Qpd\leq 1$, we must have
\begin{align*}
    &\prob_{\Vzero}\left\{\max_{\xsmall,\bm{z}}\left|\frac{1}{p_1}(\hzRF)^T\hxRF\frac{\pi-\arccos\left(\frac{(\hzRF)^T\hxRF}{\|\hzRF\|_2\cdot \|\hxRF\|_2}\right)}{2\pi}-\KThree(\xsmall^T\bm{z})\right|\geq \sqrt{\frac{\Qpd}{d}}\right\}\\
    &\leq\frac{d^2}{(p_1+1)e^{d+1}}.
\end{align*}
% where
% \begin{align*}
%     &a\defeq \frac{1}{p_1}(\hzRF)^T\hxRF\frac{\pi-\arccos\left(\frac{(\hzRF)^T\hxRF}{\|\hzRF\|_2\cdot \|\hxRF\|_2}\right)}{2\pi}.
%     % &b\defeq \frac{c}{2d}\cdot \frac{\pi-\arccos\left(c\right)}{2\pi},\\
%     % &c\defeq \frac{\sqrt{1-(\xsmall^T\bm{z})^2}+(\xsmall^T\bm{z})\left(\pi-\arccos(\xsmall^T\bm{z})\right)}{\pi}.
% \end{align*}
\end{lemma}
\begin{proof}
Because $9d\cdot \Qpd\leq 1$, we have
\begin{align}\label{eq.temp_042504}
    \Qpd=\sqrt{\Qpd}\sqrt{\Qpd}\leq \sqrt{\Qpd}\sqrt{\frac{1}{9d}}=\sqrt{\frac{\Qpd}{9d}}.
\end{align}
We also have
\begin{align}\label{eq.temp_042506}
    2d\cdot \Qpd\leq \frac{2}{9}\leq 0.5.
\end{align}
Define two events
\begin{align*}
    &\mathcal{J}_1\defeq \left\{\max_{\xsmall,\bm{z}}\left|\frac{1}{p_1}(\hxRF)^T\hzRF-\KRF(\xsmall^T\bm{z})\right|\geq \Qpd\right\},\\
    &\mathcal{J}_2\defeq \left\{\max_{\xsmall,\bm{z}}\left|a-\KThree(\xsmall^T\bm{z})\right|\geq \sqrt{\frac{\Qpd}{d}}\right\}.
\end{align*}
Notice that the randomness of those events is on $\Vzero$.
We first show $\mathcal{J}_1\supseteq \mathcal{J}_2$, i.e., $\mathcal{J}_1^c\subseteq \mathcal{J}_2^c$. To that end, suppose $\mathcal{J}_1^c$ happens.
% Because of $\mathcal{J}_1^c$, we have
% \begin{align}\label{eq.temp_042501}
%     &\max_{\xsmall,\bm{z}}2d\cdot \KRF(\xsmall,\bm{z})\left|\frac{1}{\KRF(\xsmall,\bm{z})p_1}(\hzRF)^T\hxRF-1\right|\leq 2d\cdot C.
% \end{align}
Because of $\mathcal{J}_1^c$, we have
\begin{align}\label{eq.temp_042501}
    \max_{\xsmall,\bm{z}}\left|\frac{2d}{p_1}(\hzRF)^T\hxRF-2d\cdot\KRF(\xsmall^T\bm{z})\right|\leq 2d\cdot \Qpd.
\end{align}
By Eq.~\eqref{eq.def_KRF}, we have
\begin{align}\label{eq.temp_090101}
    \KRF(\xsmall^T\xsmall)=\KRF(1)=\frac{1}{2d}.
\end{align}
Thus, we have
\begin{align}\label{eq.temp_042502}
    &\max_{\xsmall,\bm{z}}\left|\frac{2d}{p_1}\left\|\hxRF\right\|_2\cdot \left\|\hzRF\right\|_2-1\right|\nonumber\\
    =&\max_{\xsmall}\left|\frac{2d}{p_1}\left\|\hxRF\right\|_2^2-1\right|\text{ (the max value is achieved when $\|\hxRF\|_2=\|\hzRF\|_2$)}\nonumber\\
    =&\max_{\xsmall}\left|\frac{2d}{p_1}\left\|\hxRF\right\|_2^2-2d\KRF(\xsmall^T\xsmall)\right|\text{ (by Eq.~\eqref{eq.temp_090101})}\nonumber\\
    \leq &\max_{\xsmall,\bm{z}}\left|\frac{2d}{p_1}\left\|\hxRF\right\|_2\left\|\hzRF\right\|_2 -2d\KRF(\xsmall^T\bm{z})\right|\text{ (since we could set $\bm{z}=\xsmall$ on the right hand side)}\nonumber\\
    \leq & 2d\cdot \Qpd\text{ (because of $\mathcal{J}_1^c$)}.
\end{align}
By Eq.~\eqref{eq.temp_042502}, Eq.~\eqref{eq.temp_042501}, and Eq.~\eqref{eq.temp_042506}, we thus have
\begin{align}\label{eq.temp_090102}
    \left|\frac{2d}{p_1}\left\|\hxRF\right\|_2\cdot \left\|\hzRF\right\|_2-1\right|\leq 0.5\text{ for all $\xsmall$ and $\bm{z}$}.
\end{align}
Thus, we then have $\frac{2d}{p_1}(\hzRF)^T\hxRF\leq \frac{2d}{p_1}\left\|\hxRF\right\|_2\cdot \left\|\hzRF\right\|_2\leq 1.5$. Besides, we have $\frac{2d}{p_1}(\hzRF)^T\hxRF\geq 0$ because all elements of $\hxRF$ and $\hzRF$ are non-negative by Eq.~\eqref{eq.def_hxRF}. In other words, we have
\begin{align}\label{eq.temp_012301}
    \frac{2d}{p_1}(\hzRF)^T\hxRF \in [0, 1.5] \text{ for all $\xsmall$ and $\bm{z}$}.
\end{align}
Therefore, we then have
\begin{align}
    &\max_{\xsmall,\bm{z}}\left|\frac{(\hzRF)^T\hxRF}{\|\hzRF\|_2\cdot \|\hxRF\|_2}-2d\cdot\KRF(\xsmall,\bm{z})\right|\nonumber\\
    =&\max_{\xsmall,\bm{z}} \left|\frac{\frac{2d}{p_1}(\hzRF)^T\hxRF}{\frac{2d}{p_1}\|\hzRF\|_2\cdot \|\hxRF\|_2}-2d\cdot\KRF(\xsmall,\bm{z})\right|\nonumber\\
    \leq &  \max_{\xsmall,\bm{z}}\left|\frac{2d}{p_1}(\hzRF)^T\hxRF-2d\cdot\KRF(\xsmall,\bm{z})\right|+3\max_{\xsmall,\bm{z}}\left|\frac{2d}{p_1}\left\|\hxRF\right\|_2\cdot \left\|\hzRF\right\|_2-1\right|\nonumber\\
    &\text{(by Lemma~\ref{le.frac_converge}(ii) where $a=\frac{2d}{p_1}(\hzRF)^T\hxRF\in [0, 1.5]$ by Eq.~\eqref{eq.temp_012301},}\nonumber\\
    &\text{$b=\frac{2d}{p_1}\left\|\hxRF\right\|_2\cdot \left\|\hzRF\right\|_2$,  and $\delta=\left|\frac{2d}{p_1}\left\|\hxRF\right\|_2\cdot \left\|\hzRF\right\|_2-1\right|\in [0, 0.5]$ by Eq.~\eqref{eq.temp_090102}).}\nonumber\\
    \leq &9d\cdot \Qpd\text{ (by Eq.~\eqref{eq.temp_042501} and Eq.~\eqref{eq.temp_042502})}.\label{eq.temp_041901}
\end{align}
Now we apply Lemma~\ref{le.arccos_delta} by letting $\delta_1=\frac{1}{p_1}(\hxRF)^T\hzRF-\KRF(\xsmall^T\bm{z})$, $\delta_2=\frac{(\hzRF)^T\hxRF}{\|\hzRF\|_2\cdot \|\hxRF\|_2}-2d\cdot \KRF(\xsmall^T\bm{z})$, $a_1=\KRF(\xsmall^T\bm{z})$, and $a_2=2d\cdot\KRF(\xsmall^T\bm{z})$. We first check the conditions required by Lemma~\ref{le.arccos_delta}. By Eq.~\eqref{eq.def_KRF} and Lemma~\ref{le.RF_kernel_bound}, we have
\begin{align*}
    a_2 = 2d\cdot \KRF(\xsmall^T\bm{z})\in [0, 1] \subseteq [-1, 1].
\end{align*}
Because $\left|(\hzRF)^T\hxRF\right|\leq \|\hzRF\|_2\cdot \|\hxRF\|_2$, we have
\begin{align*}
    a_2+\delta_2=\frac{(\hzRF)^T\hxRF}{\|\hzRF\|_2\cdot \|\hxRF\|_2}\in [-1, 1].
\end{align*}
By Eq.~\eqref{eq.temp_041901} and $9d\cdot \Qpd \leq 1$ (the condition of this lemma), we have
\begin{align*}
    |\delta_2|\leq 9d\cdot \Qpd \leq 1.
\end{align*}
Therefore, all conditions of Lemma~\ref{le.arccos_delta} are satisfied. According to Lemma~\ref{le.arccos_delta}, we then have
\begin{align*}
    &\left|(a_1+\delta_1)\frac{\pi-\arccos(a_2+\delta_2)}{2\pi}-a_1\frac{\pi-\arccos(a_2)}{2\pi}\right|\leq \frac{1}{2}|\delta_1|+\frac{\sqrt{2}|a_1|\sqrt{|\delta_2|}}{4}\\
    \implies &
    \left|\frac{1}{p_1}(\hzRF)^T\hxRF\frac{\pi-\arccos\left(\frac{(\hzRF)^T\hxRF}{\|\hzRF\|_2\cdot \|\hxRF\|_2}\right)}{2\pi}\right.\\
    &\left.-\KRF(\xsmall^T\bm{z})\frac{\pi-\arccos(2d\cdot \KRF(\xsmall^T\bm{z}))}{2\pi}\right|\leq\frac{1}{2}\left|\frac{1}{p_1}(\hxRF)^T\hzRF-\KRF(\xsmall^T\bm{z})\right|\\
    &+\frac{\sqrt{2}|\KRF(\xsmall^T\bm{z})|\sqrt{\left|\frac{(\hzRF)^T\hxRF}{\|\hzRF\|_2\cdot \|\hxRF\|_2}-2d\cdot \KRF(\xsmall^T\bm{z})\right|}}{4}.
\end{align*}
By $\mathcal{J}_1^c$ and Eq.~\eqref{eq.temp_041901}, we thus have
\begin{align*}
    &\max_{\xsmall,\bm{z}}\left|\frac{1}{p_1}(\hzRF)^T\hxRF\frac{\pi-\arccos\left(\frac{(\hzRF)^T\hxRF}{\|\hzRF\|_2\cdot \|\hxRF\|_2}\right)}{2\pi}-\KThree(\xsmall^T\bm{z})\right|\\
    \leq& \frac{\Qpd}{2}+\frac{\sqrt{2}|\KRF(\xsmall,\bm{z})|\sqrt{9d\cdot \Qpd}}{4}\\
    \leq&\frac{\Qpd}{2}+\sqrt{\frac{9}{32d}\cdot \Qpd}\text{ (because $\KRF(\xsmall,\bm{z})\in\left[0,\ \frac{1}{2d}\right]$ by Lemma~\ref{le.RF_kernel_bound})}\\
    \leq & \left(\sqrt{\frac{1}{36d}}+\sqrt{\frac{9}{32d}}\right)\sqrt{\Qpd}\text{ (by Eq.~\eqref{eq.temp_042504})}\\
    \leq & \sqrt{\frac{\Qpd}{d}}\text{ (since $\sqrt{\frac{1}{36}}+\sqrt{\frac{9}{32}}\approx \frac{1}{6}+0.53\leq 1$)},
\end{align*}
i.e., $\mathcal{J}_2^c$ happens. We next estimate the probability of $\mathcal{J}_2$. We have
\begin{align*}
    \prob_{\Vzero}[\mathcal{J}_2]\leq &\prob_{\Vzero}[\mathcal{J}_1]\text{ (because $\mathcal{J}_2\subseteq \mathcal{J}_1$)}\\
    \leq & \frac{d^2}{(p_1+1)e^{d+1}}\text{ (by Lemma~\ref{le.max_RF_converge}, noticing that $9d\cdot \Qpd\leq 1 \implies p_1\geq 10$)}.
\end{align*}
The result of this lemma thus follows.
\end{proof}

\begin{lemma}\label{le.HHH}
We have
\begin{align*}
    \|\HHt^T(\HHt\HHt^T)^{-1}\|_2\leq \frac{1}{\sqrt{\min\eig(\HHt\HHt^T)}}.
\end{align*}
\end{lemma}
\begin{proof}
For any $\bm{a}\in\mathds{R}^n$, we have
\begin{align*}
    \|\HHt^T(\HHt\HHt^T)^{-1}\bm{a}\|_2=\sqrt{\left(\HHt^T(\HHt\HHt^T)^{-1}\bm{a}\right)^T\HHt^T(\HHt\HHt^T)^{-1}\bm{a}}&=\sqrt{\bm{a}^T(\HHt\HHt^T)^{-1}\bm{a}}\\
    &\leq \frac{\|\bm{a}\|_2}{\sqrt{\min\eig(\HHt\HHt^T)}}.
\end{align*}
The result of this lemma thus follows.
\end{proof}

% Combining Corollary~\ref{coro.card_to_arccos} and Lemma~\ref{le.arccos_to_arccos} with the union bound, we thus have the following corollary.
% \begin{corollary}
% When $p_1\geq 2d$, for any $\xsmall,\bm{z}\in\sd$, we must have
% \begin{align*}
%     \prob_{\Vzero,\Wzero}\left\{\left|(\hzRF)^T\hxRF \frac{\left|\Chzhx\right|}{p_1p_2}-\KThree(\xsmall,\bm{z})\right|\geq \frac{q}{\sqrt{p_2}}+\frac{q}{2\sqrt{p}_1}+\frac{\sqrt{6q}}{4\sqrt{d}\sqrt{p_1}}\right\}\leq \frac{4}{q^2}.
% \end{align*}
% \end{corollary}

We are now ready to prove Proposition~\ref{prop.term_C} and Proposition~\ref{prop.term_DB}.

\subsection{Proof of Proposition~\ref{prop.term_C}}\label{app.proof_term_C}
\begin{proof}
For $k=1,2,\cdots,p_2$, define
\begin{align}\label{eq.temp_012303}
    K_k=\int_{\sd}\frac{(\hxRF)^T\hzRF}{p_1}\indicator{(\hxRF)^T\Wzero[k]>0,\ (\hzRF)^T\Wzero[k]}g(\bm{z})d\xdensity(\bm{z}).
\end{align}
It is obvious that $K_1,K_2,\cdots,K_{p_2}$ are \emph{i.i.d.} (when randomness is on $\Wzero$).
By Eq.~\eqref{eq.def_pseudoGT} and Eq.~\eqref{eq.def_h3}, we have
\begin{align}\label{eq.temp_042603}
    \pseudoGT(\bm{x})=\frac{1}{p_2}\sum_{k=1}^{p_2}K_k.
\end{align}

Notice that
\begin{align}
    |K_k|\leq & \int_{\sd}\left|\frac{(\hxRF)^T\hzRF}{p_1}\right|\cdot |g(\bm{z})|\ d\xdensity(\bm{z})\nonumber\\
    \leq & \int_{\sd}|g(\bm{z})|\ d\xdensity(\bm{z})\text{ (by Lemma~\ref{le.bound_hx})}\nonumber\\
    =& \|g\|_1.\label{eq.temp_042801}
\end{align}

Thus, by Lemma~\ref{le.chebyshev}, we have
\begin{align}\label{eq.temp_042601}
    \prob_{\Wzero}\left\{\left|\frac{1}{p_2}\sum_{k=1}^{p_2} K_k-\expectation_{\Wzero}[K_1]\right|\geq \frac{q\|g\|_1}{\sqrt{p_2}}\right\}\leq\frac{1}{q^2}.
\end{align}

For any $k\in\{1,2,\cdots,p_2\}$, we have
\begin{align*}
    &\expectation_{\Wzero}[K_k]\\
    =&\int_{\sd}\expectation_{\Wzero}\left[\frac{(\hxRF)^T\hzRF}{p_1}\indicator{(\hxRF)^T\Wzero[k]>0,\ (\hzRF)^T\Wzero[k]}\right]g(\bm{z})d\xdensity(\bm{z})\\
    =&\int_{\sd}\frac{(\hxRF)^T\hzRF}{p_1}\cdot \frac{\pi-\arccos\left(\frac{(\hxRF)^T\hzRF}{\left\|\hxRF\right\|_2\cdot \left\|\hzRF\right\|_2}\right)}{2\pi}g(\bm{z})d\xdensity(\bm{z})\text{ (by Lemma~\ref{le.NTK2_kernel})}\\
    = & \f(\xsmall)+\int_{\sd}\left(\frac{(\hxRF)^T\hzRF}{p_1}\cdot \frac{\pi-\arccos\left(\frac{(\hxRF)^T\hzRF}{\left\|\hxRF\right\|_2\cdot \left\|\hzRF\right\|_2}\right)}{2\pi}-\KThree(\xsmall^T\bm{z})\right)g(\bm{z})d\xdensity(\bm{z})\\
    &\text{ (by $\f=\fg$ and Eq.~\eqref{eq.def_learnableSet})}.
\end{align*}
Thus, we have
\begin{align}
    &\left|\expectation_{\Wzero}[K_k]-\f(\xsmall)\right|\nonumber\\
    \leq &\max_{\xsmall,\bm{z}}\left|\frac{(\hxRF)^T\hzRF}{p_1}\cdot \frac{\pi-\arccos\left(\frac{(\hxRF)^T\hzRF}{\left\|\hxRF\right\|_2\cdot \left\|\hzRF\right\|_2}\right)}{2\pi}-\KThree(\xsmall^T\bm{z})\right|\cdot \|g\|_1.\label{eq.temp_042804}
\end{align}
Applying Lemma~\ref{le.arccos_to_arccos}, we then have
\begin{align}\label{eq.temp_042602}
    \prob_{\Vzero}\left\{\left|\expectation_{\Wzero}[K_k]-\f(\xsmall)\right|\geq \sqrt{\frac{\Qpd}{d}} \|g\|_1\right\}\leq\frac{d^2}{(p_1+1)e^{d+1}}.
\end{align}
Notice that
\begin{align*}
    \left|\pseudoGT(\xsmall)-\f(\xsmall)\right|&=\left|\frac{1}{p_2}\sum_{k=1}^{p_2} K_k-\f(\xsmall)\right|\text{ (by Eq.~\eqref{eq.temp_042603})}\\
    &\leq \left|\frac{1}{p_2}\sum_{k=1}^{p_2} K_k-\expectation_{\Wzero}[K_1]\right|+\left|\expectation_{\Wzero}[K_1]-\f(\xsmall)\right|\text{ (by the triangle inequality)}.
\end{align*}
Combining Eq.~\eqref{eq.temp_042601} and Eq.~\eqref{eq.temp_042602} by the union bound, we thus have
\begin{align}\label{eq.temp_042604}
    &\prob_{\Vzero,\Wzero}\left\{\left|\pseudoGT(\xsmall)-\f(\xsmall)\right|\geq \frac{q\|g\|_1}{\sqrt{p_2}}+ \sqrt{\frac{\Qpd}{d}} \|g\|_1\right\}\nonumber\\
    \leq &\prob_{\Vzero,\Wzero}\left\{\left|\frac{1}{p_2}\sum_{k=1}^{p_2} K_k-\expectation_{\Wzero}[K_1]\right|\geq \frac{q\|g\|_1}{\sqrt{p_2}}\right\}+\prob_{\Vzero,\Wzero}\left\{\left|\expectation_{\Wzero}[K_k]-\f(\xsmall)\right|\geq \sqrt{\frac{\Qpd}{d}} \|g\|_1\right\}\nonumber\\
    \leq &\frac{d^2}{(p_1+1)e^{d+1}} +\frac{1}{q^2}.
\end{align}
\end{proof}

\subsection{Proof of Proposition~\ref{prop.term_DB}}\label{app.proof_term_DB}
\begin{proof}
    For $k=1,2,\cdot,p_2$, define $\mathbf{K}_k\in\mathds{R}^n$ whose $i$-th element is
    \begin{align*}
        \mathbf{K}_{k,i}\defeq \int_{\sd}\frac{(\hXiRF)^T\hzRF}{p_1}\indicator{(\hXiRF)^T\Wzero[k]>0,\ (\hzRF)^T\Wzero[k]}g(\bm{z})d\xdensity(\bm{z}).
    \end{align*}
    Note that $\mathbf{K}_{k,i}$ is similar to $K_k$ in Eq.~\eqref{eq.temp_012303}, with the only difference that the former is defined with respect to $\Xii$ and the latter is defined with respect to $\xsmall$. Thus, we use a similar strategy to work with $\mathbf{K}_{k,i}$.
    By Eq.~\eqref{eq.def_F_pseudoGT} and Eq.~\eqref{eq.def_pseudoGT}, we have
    \begin{align*}
        \FWVg(\XX) = \frac{1}{p_2}\sum_{k=1}^{p_2}\mathbf{K}_k.
    \end{align*}
    Similar to Eq.~\eqref{eq.temp_042801}, we have
    \begin{align*}
        |\mathbf{K}_{k,i}|\leq \|g\|_1,\text{ for all }i=1,2,\cdots,n.
    \end{align*}
    Thus, we have
    \begin{align*}
        \|\mathbf{K}_k\|_2=\sqrt{\sum_{i=1}^n\mathbf{K}_{k,i}^2}\leq \sqrt{n}\|g\|_1.
    \end{align*}
    By Lemma~\ref{le.chebyshev}, we thus have
    \begin{align*}
        \prob_{\Wzero}\left\{\left\|\frac{1}{p_2}\sum_{k=1}^{p_2}\mathbf{K}_k-\expectation_{\Wzero}[\mathbf{K}_1]\right\|_2\geq \frac{q\sqrt{n}\|g\|_1}{\sqrt{p_2}}\right\}\leq \frac{1}{q^2}.
    \end{align*}
    Similar to Eq.~\eqref{eq.temp_042804}, we have
    \begin{align*}
        &\left\|\expectation_{\Wzero}[\mathbf{K}_1]-\F(\XX)\right\|_2\\
        \leq &\sqrt{n}\max_{\xsmall,\bm{z}}\left|\frac{(\hxRF)^T\hzRF}{p_1}\cdot \frac{\pi-\arccos\left(\frac{(\hxRF)^T\hzRF}{\left\|\hxRF\right\|_2\cdot \left\|\hzRF\right\|_2}\right)}{2\pi}-\KThree(\xsmall,\bm{z})\right|\cdot \|g\|_1.
    \end{align*}
    Thus, similar to Eq.~\eqref{eq.temp_042604}, we have
    \begin{align}\label{eq.temp_042805}
        \prob_{\Vzero,\Wzero}\left\{\left\|\F(\XX)-\FWVg(\XX)\right\|_2\geq \frac{q\sqrt{n}\|g\|_1}{\sqrt{p_2}}+ \sqrt{\frac{\Qpd}{d}}\sqrt{n} \|g\|_1\right\}\leq \frac{d^2}{(p_1+1)e^{d+1}} +\frac{1}{q^2}.
    \end{align}
    
    % {\color{blue}
    % Notice that
    % \begin{align*}
    %     \|\F(\XX)-\FWVg(\XX)\|_2=\sqrt{\sum_{i=1}^n \left(\pseudoGT(\XX_i)-\f(\XX_i)\right)^2}\leq \sqrt{n}\cdot\max_{i\in\{1,2,\cdots,n\}}\left|\pseudoGT(\XX_i)-\f(\XX_i)\right|.
    % \end{align*}
    % By Eq.~\eqref{eq.temp_042604} and the union bound, for any $m>0$, we thus have
    % \begin{align*}
    %     &\prob_{\Vzero,\Wzero}\left\{\|\F(\XX)-\FWVg(\XX)\|_2\geq \frac{m\sqrt{n}\|g\|_1}{\sqrt{p_2}}+ 4\sqrt{2}d\left(\frac{2(d+1)\log(p_1+1)}{p_1}\right)^{\frac{1}{4}} \sqrt{n}\|g\|_1\right\}\\
    %     &\leq \frac{3nd^2}{(p_1+1)e^{d+1}} +\frac{n}{m^2}.
    % \end{align*}
    % Replacing $m$ by $\sqrt{n}q$, we thus have
    % \begin{align}
    %     &\prob_{\Vzero,\Wzero}\left\{\|\F(\XX)-\FWVg(\XX)\|_2\geq \frac{nq\|g\|_1}{\sqrt{p_2}}+ 4\sqrt{2}d\left(\frac{2(d+1)\log(p_1+1)}{p_1}\right)^{\frac{1}{4}} \sqrt{n}\|g\|_1\right\}\nonumber\\
    %     &\leq \frac{3nd^2}{(p_1+1)e^{d+1}} +\frac{1}{q^2}.\label{eq.temp_042701}
    % \end{align}
    % }
    We note that
    \begin{align*}
        &\text{term D}+\text{term B of Eq.~\eqref{eq.term_ABC}}\\
        \leq &  \left\|\hhx\right\|_2\left\|\HHt^T(\HHt\HHt^T)^{-1}\right\|_2\cdot\left( \|\F(\XX)-\FWVg(\XX)\|_2+\|\esmall\|_2\right)\text{ (by Lemma~\ref{le.matrix_norm})}\\
        \leq &\frac{\sqrt{p_1 p_2}\cdot \left(\|\F(\XX)-\FWVg(\XX)\|_2+\|\esmall\|_2\right)}{\sqrt{\min\eig(\HHt\HHt^T)}}\text{ (by Lemma~\ref{le.HHH} and Lemma~\ref{le.bound_hx})}.
    \end{align*}
    Combining Eq.~\eqref{eq.temp_042805} and Proposition~\ref{prop.min_eig} by the union bound, we thus have
    \begin{align*}
        &\prob_{\XX,\Vzero,\Wzero}\left\{\text{term D}+\text{term B of Eq.~\eqref{eq.term_ABC}}\geq \frac{\sqrt{n}\|g\|_1\left(\frac{q}{\sqrt{p_2}}+\sqrt{\frac{\Qpd}{d}}\right)+\|\esmall\|_2}{\sqrt{\Jnppdq}} \right\}\\
        &\leq \frac{d^2}{(p_1+1)e^{d+1}} +\frac{8}{q^2}.
    \end{align*}
\end{proof}

% Now we are ready to prove Theorem~\ref{th.main}.

\section{Details Related to Learnable Set}\label{app.details_learnable_set}

In this part, we first restate Proposition~\ref{prop.concise_learnable_set} in a more precise way, i.e.,  Proposition~\ref{prop.finite_sum_in_the_set} in Appendix~\ref{app.set_part_1} and Proposition~\ref{prop.compare_efficient_and_set} in Appendix~\ref{app.set_part_2}. Then, in Appendix~\ref{sec.outside_learnable_set} we discuss the generalization performance of ground-truth functions outside the learnable set.

\subsection{\texorpdfstring{$\learnableSet$}{Learnable set} contains all polynomials with finite degree}\label{app.set_part_1}
By the following proposition, we show that $\learnableSet$ contains all polynomials with finite degree. We formally state it in the following proposition.
\begin{proposition}\label{prop.finite_sum_in_the_set}
Let $k$ be a finite non-negative integer. For any $\f(\xsmall)=\sum_{i=0}^k c_i (\xsmall^T \bm{a}_i)^i$ where  $c_i\in\mathds{R}$ and $\bm{a}_i\in \mathds{R}^d$, we must have $\f\in \learnableSet$.
\end{proposition}
We prove Proposition~\ref{prop.finite_sum_in_the_set} in Appendix~\ref{app.proof_finite_degree}. Although Proposition~\ref{prop.finite_sum_in_the_set} is only for no-bias situation of 3-layer NTK, we can easily prove the similar results for the biased 3-layer NTK with the same proof technique.

% \textbf{Comparison of the learnable sets with 2-layer NTK:}
% {\color{cyan}We now compare the learnable set $\learnableSet$ with that of 2-layer NTK.} 
% Here we can already see a major difference between the generalization performance of 3-layer NTK and 2-layer NTK, i.e., the impact of ReLU bias on the learnable set.

% It is not hard to show that for any ground-truth function $\fg$ that is a polynomials with finite degree (i.e., $\fg(\xsmall)=\sum_{l=0}^k (\xsmall^T\bm{a}_l)^l$ where $k$ is finite), for both $\learnableSet$ and $\learnableSetTwoBias$ the corresponding norm of $g$ must be finite. Therefore, we must have $\fg\in \learnableSet\cap\learnableSetTwoBias$ for any polynomials $\fg$ with finite degree.

\subsection{\texorpdfstring{$\learnableSet$}{F(3)} is a superset of \texorpdfstring{$\learnableSetTwoBias$}{F(2)b} (recall the definition of \texorpdfstring{$\learnableSetTwoBias$}{F(2)b} in Section~\ref{sec.high_d})}\label{app.set_part_2}

% We want to know how the learnable set of these NTK models differs. In the previous part, It remains to consider polynomials with infinite degree.
% In this case, it is hard to determine whether $g$ is finite or infinite, which depends on the 

The learnable sets of both 3-layer and 2-layer NTK models also contain polynomials with \emph{infinite} degree. Notice that not all infinite-degree polynomials belong to the learnable sets, because the norm of the corresponding function $g$ may not be finite. As we mentioned in footnote~\ref{footnote.finite_g_1}, the constrain $\|g\|_\infty<\infty$ can be relaxed to $\|g\|_1<\infty$. However, with $\|g\|_1<\infty$, the comparison among those learnable sets becomes more difficult. For convenience, we just relax the constraint to $\|g\|_2<\infty$ (instead of $\|g\|_1<\infty$) in the following result.
% Instead, we need to carefully examine
% the coefficients of the harmonics expansion of the corresponding kernels. Even so, we can still show that 3-layer NTK at least can learn all learnable functions for 2-layer NTK.

\begin{proposition}\label{prop.compare_efficient_and_set}
Under the constraint of $\|g\|_2<\infty$, the learnable set of the 3-layer NTK (no bias) is at least as large as the 2-layer NTK (both with and without bias) ,i.e., $\learnableSetTwo\cup \learnableSetTwoBias \subseteq \learnableSet$. The learnable set of 2-layer NTK with bias is larger than that of 2-layer NTK without bias i.e., $\learnableSetTwo\subset \learnableSetTwoBias$. The learnable sets of 2-layer NTK with different bias settings are the same i.e., $\learnableSetTwoBiasA=\learnableSetTwoBiasB$ for any $b_1,b_2\in (0,1)$.
% for all $b\in (0, 1)$,  $\learnableSetTwo\subset \learnableSetTwoBias$,  $\learnableSetTwo\cup \learnableSetTwoBias \subseteq \learnableSet$, and the set $\learnableSetTwoBias$ does not change with $b$,.
\end{proposition}
We prove Proposition~\ref{prop.compare_efficient_and_set} in Appendix~\ref{app.proof_compare_set}. An important message conveyed by Proposition~\ref{prop.compare_efficient_and_set} is that, 3-layer NTK can at least learn all learnable functions for 2-layer NTK under the constraint $\|g\|_2<\infty$. We conjecture that the same result may also hold for $\|g\|_1<\infty$, which we leave for future work.

% Mathematically speaking, Proposition~\ref{prop.compare_efficient_and_set} shows that $\learnableSetTwo\subset \learnableSetTwoBias$,  $\learnableSetTwo\cup \learnableSetTwoBias \subseteq \learnableSet$, and the set $\learnableSetTwoBias$ does not change with $b$.

% \footnote{Although we only show $\learnableSetTwoBB=\learnableSetTwoNB$, the result holds for any bias setting of Eq.~\eqref{eq.def_bias_input} as long as $b>0$.}. We provide the proof of Lemma~\ref{le.coefficient_compare} in Appendix~\ref{app.proof_coeff_set}.

\subsection{Generalization performance of ground-truth functions outside the learnable set}\label{sec.outside_learnable_set}
% In this subsection, we will {\color{cyan}use numerical results to further study} how the type of ground-truth functions and bias settings affect the actual generalization performance under different settings.

One may wonder what happens to the generalization performance for functions outside the learnable set.
Notice that although we have proven that ground-truth functions inside the learnable set can be learned, it is possible that some functions outside the learnable set could still be learnable. For 2-layer NTK models without bias, \cite{ju2021generalization} shows that if a ground-truth function has a positive distance away from the learnable set, then such distance becomes the lower bound of the generalization error. Such ground-truth functions with positive distance exist for 2-layer NTK, e.g., $(\xsmall^T\bm{e}_1)^3$, because $\learnableSetTwo$ does not contain odd power polynomials except linear functions. However, for 2-layer NTK with bias or 3-layer NTK, there do not exist such ground-truth functions with a positive distance away from the learnable set. In other words, functions outside the learnable set is still in the closure of the corresponding learnable set. Thus, it is unclear whether or not those functions have a very different generalization performance compared with functions inside the learnable set.

\begin{figure}[t!]
    \centering
    \includegraphics[width=0.7\textwidth]{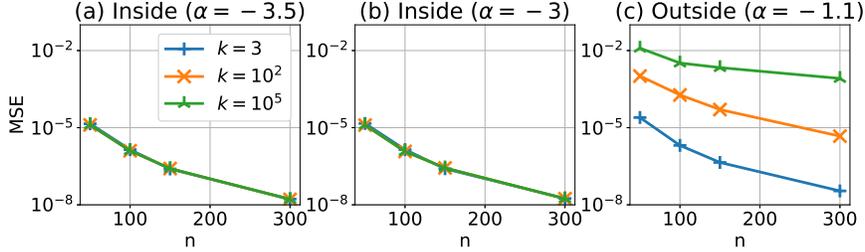}
    \caption{Curves of test MSE of 2-layer NTK with normal bias with respect to $n$ for the ground-truth functions $\f_{k, \alpha}(\xsmall)$ where $p_1,p_2\to\infty$, $d=2$, and $\esmall=\bm{0}$. Every curve is the average of 10 simulation runs.
    }\label{fig.boundary_check_x_axis_is_n}
\end{figure}

We now use simulation results in Fig.~\ref{fig.boundary_check_x_axis_is_n} to show that functions outside the learnable set may indeed exhibit qualitatively different generalization performance (and thus Proposition~\ref{prop.compare_efficient_and_set} will be meaningful in capturing ground-truth functions with good generalization performance). 
We construct an example of functions inside and outside the learnable set (in the sense of finite $\|g\|_2$, consistent with Proposition~\ref{prop.compare_efficient_and_set}). 
For simplicity, we focus on $\learnableSetTwoNB$, which is the learnable set for the 2-layer NTK with normal bias. We then consider a specific type of normalized ground-truth functions $\f_{k,\alpha}\defeq \bar{f}_{k,\alpha}/\|\bar{f}_{k,\alpha}\|_2$ where $\bar{f}_{k,\alpha}(\xsmall)\defeq \sum_{i=1}^k i^\alpha \left(\xsmall^T\bm{e}_d\right)^i$. By previous discussion, we have already known that if $k$ is finite, then $f_{k,\alpha}\in \learnableSetTwoNB$. However, when $k=\infty$, then whether $f_{\infty,\alpha}\in \learnableSetTwoNB$ or not is determined by the value of $\alpha$. We let $d=2$ and choose the value of $\alpha$ to be $-3.5$, $-3$, and $-1.1$, respectively. It can be verified that $f_{k,\alpha}\in \learnableSetTwoNB$ when $\alpha=-3.5$ or $\alpha=-3$, while $f_{k,\alpha}\notin \learnableSetTwoNB$ when $\alpha=-1.1$.
In numerical experiments, it is difficult to directly calculate $f_{\infty,\alpha}$, as we do not know the close form of $f_{\infty,\alpha}$. Therefore, we use $f_{k,\alpha}$ to approach $f_{\infty,\alpha}$ by increasing $k$. In Fig.~\ref{fig.boundary_check_x_axis_is_n}(a), we let $\alpha=-3.5$ and plot the test MSE with respect to $n$ when $k=3$ (blue curve), $k=10^2$ (orange curve), and $k=10^5$ (green curve), respectively. We can see that these three curves almost overlap with each other, which implies that increasing $k$ does not alter the test error significantly. (Similar phenomenon also appears in Fig.~\ref{fig.boundary_check_x_axis_is_n}(b) where $\alpha=-3$.) In contrast, when we let $\alpha=-1.1$ in Fig.~\ref{fig.boundary_check_x_axis_is_n}(c), larger $k$ leads to a much flatter curve. This phenomenon suggests that when $k\to \infty$, providing more training data becomes less effective in lowering the test error. Besides, by comparing the curve of $k=10^5$ in Fig.~\ref{fig.boundary_check_x_axis_is_n}(a) and (c), we can see that the curve in Fig.~\ref{fig.boundary_check_x_axis_is_n}(c) is higher than the one in Fig.~\ref{fig.boundary_check_x_axis_is_n}(a) by several orders of magnitude. Therefore, we can tell that the functions inside and outside the learnable set could have very different generalization performance. 

% Further simulations can be found in Appendix~\ref{app.more_simu_for_outside_set}.

% \subsection{About the conjecture in Section~\ref{sec.outside_learnable_set}}\label{app.more_simu_for_outside_set}

\begin{figure}[ht!]
\centering
\includegraphics[width=4in]{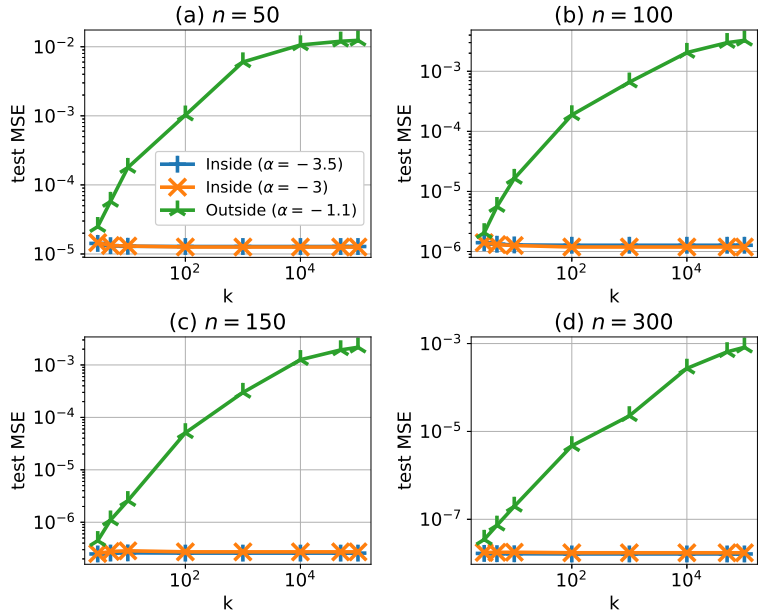}
\caption{Curves of test MSE of 2-layer NTK with normal bias with respect to $k$ for the ground-truth functions $\f_{k, \alpha}(\xsmall)$ where $p_1,p_2\to\infty$, $d=2$, and $\esmall=\bm{0}$. Every curve is the average of 10 simulation runs.}\label{fig.boundary_check}
\end{figure}

% We provide some additional simulation results (in addition to Fig.~\ref{fig.boundary_check}) to validate our conjecture in Section~\ref{sec.outside_learnable_set} that the functions inside and outside the learnable set could have very different generalization performance. Note that in Section~\ref{sec.outside_learnable_set}, the x-axis of Fig.~\ref{fig.boundary_check_x_axis_is_n} is $n$. Here, we also plot the curves with x-axis being $k$.

The setup of Fig.~\ref{fig.boundary_check} is the same as that of Fig.~\ref{fig.boundary_check_x_axis_is_n} except that here we let x-axis be $k$. In Fig.~\ref{fig.boundary_check}, we can see that the curves of $\alpha = -3.5$ and $\alpha=-3$ (finite $\|g\|_2$) in all sub-figures~(a)(b)(c)(d) are almost flat with respect to $k$. In contrast, the curves of $\alpha =-1.1$ (infinite $\|g\|_2$) keep increasing with respect to $k$, and have much higher generalization error when $k$ is large than those with finite $\|g\|_2$. This also validates our conjecture that the functions inside and outside the learnable set could have very different generalization performance.

\begin{comment}
    Specifically, we have the following result.
    \begin{proposition}\label{prop.boundary_condition}{ \normalfont (proof in Appendix~\ref{app.proof_one_point_five})}
    Consider $g_\alpha$ such that $\int_{\sd}\KTwo(\frac{d}{d+1}\xsmall^T\bm{z}+\frac{1}{d+1})g_\alpha(\bm{z})d\xdensity(\bm{z}) =f_{\infty, \alpha}(\xsmall)$. Then $\|g_{\alpha}\|_\infty<\infty$ if and only if $\alpha < -1.5$. 
    % Besides, $\|g_{\alpha}\|_1<\infty$ if and only if $\alpha \leq -1.5$.
    % \begin{align*}
    %     \|g_{\alpha}\|_\infty \begin{cases}
    %     <\infty, \text{ if }\alpha < -1.5\\
    %     =\infty, \text{ otherwise}
    %     \end{cases}, \quad \text{and }\|g_\alpha\|_1\begin{cases}
    %     <\infty, \text{ if }\alpha \leq -1.5\\
    %     =\infty, \text{ otherwise}
    %     \end{cases}.
    % \end{align*}
    \end{proposition}
    By Proposition~\ref{prop.boundary_condition}, we know that $\alpha=-1.5$ is the boundary value to determine whether $f_{\infty,\alpha}\in\learnableSetTwoNB$. 
\end{comment}

% Specifically, we consider the setting of the  two-layer NTK model  (e.g., \cite{ju2021generalization}).
% We define
% \begin{align*}
%     \hft(\xsmall)\defeq\KTwo(\xsmall,\bm{e}).
% \end{align*}
% Then, we can do the similar decomposition like in Section~\ref{subsec.three_contain_all}. By Proposition~3 in \cite{ju2021generalization}, we have
% \begin{align*}
%     c_{\hft}(l,\bm{0})\begin{cases}=0, \text{ when }l=3,5,7,\cdots,\\
%     \neq 0,\text{ when }l=0,1,2,4,6,\cdots.\end{cases}
% \end{align*}

\section{Proof of Proposition~\ref{prop.finite_sum_in_the_set}}\label{app.proof_finite_degree}
% In the rest of this subsection, 
\begin{proof}
We prove Proposition~\ref{prop.finite_sum_in_the_set} by using similar methods as in \cite{ju2021generalization}.
For any $\fg\in\learnableSet$, we have
\begin{align}
    &\fg(\xsmall) =g\circledast \hf(\xsmall)\defeq\int_{\mathsf{SO}(d)}g(\mathbf{S}\bm{e}_d)\hf(\mathbf{S}^{-1}\xsmall)d\mathbf{S},\label{eq.convolution}\\
    &\hf(\xsmall)\defeq \KThree(\xsmall^T\bm{e}_d),\label{eq.h_in_convolution}
\end{align}
where $\bm{e}_d\defeq [0\ 0\ \cdots\ 0\ 1]^T\in\mathds{R}^d$,
% \begin{align*}
%     \hf(\xsmall)\defeq \xsmall^T\bm{e}\frac{\pi-\arccos (\xsmall^T\bm{e})}{2\pi},\quad \bm{e}\defeq [0\ 0\ \cdots\ 0\ 1]^T\in\mathds{R}^d,
% \end{align*}
and $\mathbf{S}$ is a $d\times d$ orthogonal matrix that denotes a rotation in $\sd$, chosen from the set $\mathsf{SO}(d)$ of all rotations.
An important property of the convolution Eq.~\eqref{eq.convolution} is that it corresponds to multiplication in the frequency domain, similar to Fourier coefficients. To define such a transformation to the frequency domain, we use a set of  hyper-spherical harmonics $\Xi_{\mathbf{K}}^l$ \citep{vilenkin1968special, dokmanic2009convolution} when $d\geq 3$, which forms an orthonormal basis for functions on $\sd$. These harmonics are indexed by $l$ and $\mathbf{K}$, where $\mathbf{K}=(k_1,k_2,\cdots,k_{d-2})$ and $l=k_0\geq k_1\geq k_2\geq \cdots \geq k_{d-2}\geq 0$ (those $k_i$'s and $l$ are all non-negative integers). Any function $f\in L^2(\sd\mapsto\mathds{R})$ (including even  $\delta$-functions \citep{li2013integral}) can be decomposed uniquely into these harmonics, i.e., $\f(\xsmall)=\sum_{l}\sum_{\mathbf{K}}c_{\f}(l,\mathbf{K})\Xi_{\mathbf{K}}^l(\xsmall)$,
where $c_{\f}(\cdot,\cdot)$ are projections of $\f$ onto the basis function.

In Eq.~\eqref{eq.convolution}, let $c_g(\cdot,\cdot)$ and $c_h(\cdot,\cdot)$ denote the coefficients corresponding to the decompositions of $g$ and $h$, respectively. Then, we must have \citep{dokmanic2009convolution}
\begin{align}\label{eq.freq_product}
    c_{\fg}(l,\mathbf{K})=\Lambda\cdot c_{g}(l,\mathbf{K})c_{\hf}(l,\bm{0}),
\end{align}
where $\Lambda$ is some normalization constant.

Eq.~\eqref{eq.freq_product} describes an interesting ``filtering'' interpretation on $\learnableSet$. Specifically, $\hf$ and $c_{\hf}$ work like a channel or a filter in a wireless communication system, where $c_{g}$ denotes the transmitted signal and $c_{\fg}$ denotes the received signal. Therefore, for any basis function $f(\xsmall)= \Xi_{\mathbf{K}}^l(\xsmall)$, as long as $c_{\hf}(l,\bm{0})\neq 0$, we must have $f=\fg\in\learnableSet$ where the corresponding $g(\cdot)$ can simply be chosen as $g(\bm{z})=\frac{\Xi_{\mathbf{K}}^l(\bm{z})}{\Lambda c_{\hf}(l,\bm{0})}$. Indeed, we have the following proposition about values of $c_{\hf}(l,\bm{0})$.

\begin{proposition}\label{prop.coefficients}
$c_{\hf}(l,\bm{0})>0$ for all $l=0,1,2,\cdots$.
\end{proposition}
We provide its proof in Appendix~\ref{proof.positive}.

By Proposition~\ref{prop.coefficients}, we know that all harmonics $\Xi_{\mathbf{K}}^l\in \learnableSet$. Notice that the set $\learnableSet$ is invariant under addition and scale operation\footnote{Specifically, if $f_{g_1},f_{g_2}\in\learnableSet$, then $f_{g_1+g_2}\defeq f_{g_1}+f_{g_2}\in\learnableSet$ and $f_{\alpha g_1}\defeq \alpha f_{g_1}\in \learnableSet$.}. Therefore, any finite sum of $\Xi_{\mathbf{K}}^l$ belongs to $\learnableSet$. Notice that for any non-negative integer $i$ and a real-valued vector $\bm{a}\in\mathds{R}^d$, a polynomial $(\xsmall^T \bm{a})^i$ consists of a finite sum of harmonic basis. Thus, $\learnableSet$ contains any  polynomials $(\xsmall^T\bm{a})^l$ for all $l=0,1,2,\cdots$. Proposition~\ref{prop.finite_sum_in_the_set} thus follows.
\end{proof}

\subsection{Proof of Proposition~\ref{prop.coefficients}}\label{proof.positive}
It is relatively easy to prove the result when $d=2$, which is omitted here. We focus on the general case when $d\geq 3$.
By Eq.~(115) of \cite{ju2021generalization}, the harmonics $\Xi_{\bm{0}}^l$ can be expressed by
\begin{align}
    \Xi_{\bm{0}}^l(\xsmall)=A_{\bm{0}}^l\sum_{k=0}^{\lfloor\frac{l}{2}\rfloor}(-1)^k\frac{\Gamma(l-k+\frac{d-2}{2})}{\Gamma(\frac{d-2}{2})k!(l-2k)!}(2\xsmall^T\bm{e}_d)^{l-2k}\label{eq.temp_011305},
\end{align}
where $A_{\bm{0}}^l$ is a positive number as the normalization factor of $\Xi_{\bm{0}}^l$.
We give a few examples of $\Xi_{\mathbf{0}}^l$ as follows.
\begin{align*}
    &\Xi_{\mathbf{0}}^0(\xsmall)=A_{\bm{0}}^0,\\
    &\Xi_{\mathbf{0}}^1(\xsmall)=A_{\bm{0}}^1(d-2)\xsmall^T\bm{e}_d,\\
    &\Xi_{\mathbf{0}}^2(\xsmall)=A_{\bm{0}}^2\frac{d-2}{2}\left(d(\xsmall^T\bm{e}_d)^2-1\right),\\
    &\Xi_{\mathbf{0}}^3(\xsmall)=A_{\bm{0}}^3\frac{d-2}{2}\cdot d\cdot \left((\xsmall^T\bm{e}_d)^3-\xsmall^T\bm{e}_d\right).
\end{align*}
% Using the harmonics defined in \cite{dokmanic2009convolution}, the basis $\Xi_{\bm{0}}^l$ for $(l,\bm{0})$ turns out to have the form
% \begin{align}\label{eq.K_zero_form}
%     \Xi_{\bm{0}}^l(\xsmall)=\sum_{k=0}^{\lfloor\frac{l}{2}\rfloor}(-1)^k\cdot a_{l,k}\cdot (\xsmall^T\bm{e})^{l-2k},
% \end{align}
% where $a_{l,k}$ are positive constants.

Recalling Eq.~\eqref{eq.def_KThree}, we perform a Taylor expansion of $\KThree(\cdot )$.
Let $\taylorCoeff_0,\taylorCoeff_1,\cdots$ denote the Taylor expansion coefficients of $2d\cdot \KThree$, i.e.,
\begin{align}\label{eq.temp_051301}
    2d\cdot \KThree(a)=\sum_{k=0}^\infty \taylorCoeff_k a^k.
\end{align}
The following lemma shows that all coefficients in Eq.~\eqref{eq.temp_051301} are positive.
\begin{lemma}\label{le.ck_positive}
For all $k=0,1,2,\cdots$, we have $\taylorCoeff_k>0$ in Eq.~\eqref{eq.temp_051301}.
\end{lemma}
\begin{proof}
% By Eq.~\eqref{eq.temp_051301} and Eq.~\eqref{eq.def_KThree}, we have
% \begin{align*}
%     \taylorCoeff_0=&2d\cdot \KThree(\bm{a},\bm{e})\bigg|_{\bm{a}^T\bm{e}=0}=\frac{1}{\pi}\cdot \frac{\pi-\arccos(1/\pi)}{2\pi}\approx0.096,\\
%     \taylorCoeff_1=&\frac{\partial \left(2d\cdot \KThree(\bm{a},\bm{e})\right)}{\partial (\bm{a}^T\bm{e})}\bigg|_{\bm{a}^T\bm{e}=0}\approx 0.177.
% \end{align*}
% Notice that
% \begin{align*}
%     &A\defeq \taylorCoeff_0+\taylorCoeff_1+\taylorCoeff_2+\taylorCoeff_3+\cdots=2d\cdot \KThree(\bm{a},\bm{e})\bigg|_{\bm{a}^T\bm{e}=1}=0.5,\\
%     &B\defeq \taylorCoeff_0-\taylorCoeff_1+\taylorCoeff_2-\taylorCoeff_3+\cdots=2d\cdot \KThree(\bm{a},\bm{e})\bigg|_{\bm{a}^T\bm{e}=-1}=0.
% \end{align*}
% Thus, we have
% \begin{align*}
%     &\taylorCoeff_0+\taylorCoeff_2+\taylorCoeff_4+\cdots = \frac{A+B}{2}=0.25,\\
%     &\taylorCoeff_1+\taylorCoeff_3+\taylorCoeff_5+\cdots = \frac{A-B}{2}=0.25.
% \end{align*}
% It remains to prove $\taylorCoeff_k>0$ for all $k=0,1,2,\cdots$. To that end,
% We let $\theta=\arccos(\bm{a}^T\bm{e})$. By Eq.~\eqref{eq.def_KRF}, we thus have
By Lemma~\ref{le.taylor_of_kernel}, for any $a,b\in [0, 1]$, we have
\begin{align}
    &2d\cdot \KRF(a)=\frac{1}{\pi}\left(1+\frac{\pi}{2}a+\sum_{k=0}^\infty\frac{2(2k)!}{(k+1)(2k+1)(k!)^2}\left(\frac{a}{2}\right)^{2k+2}\right),\label{eq.temp_050201}\\
    &\KTwo(b)=\frac{b}{4}+\frac{1}{2\pi}\sum_{k=0}^\infty \frac{(2k)!}{(k!)^2}\frac{4}{2k+1}\left(\frac{b}{2}\right)^{2k+2}.\label{eq.temp_111001}
\end{align}
By Lemma~\ref{le.RF_kernel_bound}, we know that $2d\cdot \KRF(a)\in [0, 1]$. Thus, we can let $b=2d\cdot \KRF(a)$ in Eq.~\eqref{eq.temp_111001} and then apply Eq.~\eqref{eq.temp_050201}, i.e.,
\begin{align}
    &\KTwo(2d\cdot \KRF(a))\nonumber\\
    =&\frac{1}{4\pi}\left(1+\frac{\pi}{2}a+\sum_{k=0}^\infty\frac{2(2k)!}{(k+1)(2k+1)(k!)^2}\left(\frac{a}{2}\right)^{2k+2}\right)\nonumber\\
    &+\frac{1}{2\pi}\sum_{l=0}^{\infty}\frac{(2l)!}{(l!)^2}\frac{4}{2l+1}\left(\frac{1}{2\pi}\left(1+\frac{\pi}{2}a+\sum_{k=0}^\infty\frac{2(2k)!}{(k+1)(2k+1)(k!)^2}\left(\frac{a}{2}\right)^{2k+2}\right)\right)^{2l+2}.\label{eq.temp_111002}
\end{align}
By Eq.~\eqref{eq.def_KThree} and Eq.~\eqref{eq.temp_051301}, we know that $\taylorCoeff_k$ is the coefficient of $a^k$ in Eq.~\eqref{eq.temp_111002}. In order to know the sign of $\taylorCoeff_k$, it remains to combine similar terms in Eq.~\eqref{eq.temp_111002}.
To that end, we apply Lemma~\ref{le.expansion} and have
\begin{align*}
    \taylorCoeff_0=&\frac{1}{4\pi}+\frac{1}{2\pi}\sum_{l=0}^{\infty}\frac{(2l)!}{(l!)^2}\frac{4}{2l+1}\left(\frac{1}{2\pi}\right)^{2l+2},\\
    \taylorCoeff_1=&\frac{1}{8}+\frac{1}{2\pi}\sum_{l=0}^{\infty}\frac{(2l)!}{(l!)^2}\frac{4}{2l+1}\left(\frac{1}{2\pi}\right)^{2l+2}(2l+2)\left(\frac{\pi}{2}\right)^{2l+1},\\
    \taylorCoeff_{2i+1}=&\frac{1}{2\pi}\sum_{l=0}^{\infty}\frac{(2l)!}{(l!)^2}\frac{4}{2l+1}\left(\frac{1}{2\pi}\right)^{2l+2}\cdot\sum_{\substack{k_0+k_1+k_2+k_4+\cdots+k_{2i}=2l+2\\k_1+2k_2+4k_4+\cdots+2ik_{2i}=2i+1}}\\
    &(k_0,k_1,k_2,k_4,\cdots,k_{2i})!\left(\frac{\pi}{2}\right)^{k_1}\prod_{j=0}^{i-1}\left(\frac{2(2j)!}{(j+1)(2j+1)(j!)^2}\left(\frac{1}{2}\right)^{2j+2}\right)^{k_{2j+2}},\\
    \taylorCoeff_{2i+2}=&\frac{1}{4\pi}\frac{2(2i)!}{(i+1)(2i+1)(i!)^2}\left(\frac{1}{2}\right)^{2i+2}\\
    &+\frac{1}{2\pi}\sum_{l=0}^{\infty}\frac{(2l)!}{(l!)^2}\frac{4}{2l+1}\left(\frac{1}{2\pi}\right)^{2l+2}\cdot\sum_{\substack{k_0+k_1+k_2+k_4+\cdots+k_{2i}=2l+2\\k_1+2k_2+4k_4+\cdots+2ik_{2i}=2i+2}}\\
    &(k_0,k_1,k_2,k_4,\cdots,k_{2i})!\left(\frac{\pi}{2}\right)^{k_1}\prod_{j=0}^{i-1}\left(\frac{2(2j)!}{(j+1)(2j+1)(j!)^2}\left(\frac{1}{2}\right)^{2j+2}\right)^{k_{2j+2}}.
\end{align*}
As we can see, every term in those expressions of $\taylorCoeff_0,\taylorCoeff_1,\cdots$ is positive, which implies that $\taylorCoeff_k>0$ for all $k=0,1,\cdots$.
\end{proof}

From Eq.~\eqref{eq.temp_051301}, we have $2d\cdot \KThree(\xsmall^T \bm{e}_d)=\sum_{k=0}^\infty u_k (\xsmall^T \bm{e}_d)^k$. We now consider the decomposition of each $(\xsmall^T \bm{e}_d)^k$ into harmonics.
\begin{lemma}\label{le.positive_component}\label{le.Q_positive}
Let $a$ and $b$ be two non-negative integers. Define the function
\begin{align}\label{eq.def_Qab}
    Q(a,b)\defeq\int_{\sd}(\bm{x}^T\bm{e}_d)^a\cdot \Xi_{\bm{0}}^{b}(\xsmall)d\xdensity(\xsmall).
\end{align}
We must have
\begin{align}\label{eq.temp_012401}
    Q(2k, 2m+1)=Q(2k+1,2m)=0,
\end{align}
\begin{align}\label{eq.temp_111103}
    Q(2k,2m)\begin{cases}
    >0,\text{ if }m\leq k,\\
    =0,\text{ if }m>k,
    \end{cases}
\end{align}
and
\begin{align}\label{eq.temp_111101}
    Q(2k+1,2m+1)\begin{cases}
    >0,\text{ if }m\leq k,\\
    =0,\text{ if }m>k.
    \end{cases}
\end{align}
\end{lemma}
\begin{proof}
By Eq.~\eqref{eq.temp_011305}, we have $\Xi_{\bm{0}}^b(-\xsmall)=(-1)^b \Xi_{\bm{0}}^b(\xsmall)$. Thus, when $a+b$ is odd, the function $(\bm{x}^T\bm{e}_d)^a\cdot \Xi_{\bm{0}}^{b}(\xsmall)$ is an odd function with respect to $\xsmall$. By symmetry of $\sd$, we then have $Q(a,b)=0$ when $a+b$ is odd, i.e., Eq.~\eqref{eq.temp_012401} holds.
Eq.~\eqref{eq.temp_111103} has been proved in Lemma~53 of \cite{ju2021generalization} by mathematical induction. Here we prove Eq.~\eqref{eq.temp_111101}.
% Here we prove the result for $Q(2k+1, 2m+1)$ by using the similar method.

% {
% \newcommand{\xed}[1]{(\bm{x}^T\bm{e}_d)^{#1}}
% \newcommand{\intsd}{\int_{\sd}}
% \newcommand{\xib}[1]{\Xi_{\bm{0}}^{#1}(\xsmall)}
% \newcommand{\Ab}[1]{A_{\bm{0}}^{#1}}
% We have
% \begin{align*}
%     Q(2k+1, 1)=\intsd\xed{2k+1}\xib{1} \xdensity(\xsmall)= \Ab{1}(d-2)\intsd \xed{2k+2}\xdensity(\xsmall)>0.
% \end{align*}
% Therefore, it only remains to prove the result for $Q(2k+1, 2m+1)$ when $m\geq 1$ in Eq.~\eqref{eq.temp_111101}. When $m\geq 1$, we have
% \begin{align*}
%     &Q(1, 2m+1)=\intsd \xed{1}\xib{2m+1}\xdensity(\xsmall)=\frac{1}{\Ab{1}(d-2)}\intsd \xib{1}\xib{2m+1}\xdensity(\xsmall)=0,\\
%     &\text{(by the orthogonality of the basis)}.
% \end{align*}
% Thus, Eq.~\eqref{eq.temp_111101} holds for all $m$ when $k=0$. Now we do mathematical induction on $k$. Suppose that Eq.~\eqref{eq.temp_111101} holds when $k=i$. To complete the mathematical induction, we need to show that Eq.~\eqref{eq.temp_111101} also holds for $k=i+1$. 
By Eq.~(118) of \cite{ju2021generalization}, for any $a$, we have
\begin{align}\label{eq.temp_111102}
    Q(a+1, l+1)=q_{l,1}\cdot Q(a, l+2)+q_{l,2}\cdot Q(a, l),
\end{align}
where $q_{l,1}>0$ and $q_{l,2}>0$.
Applying Eq.~\eqref{eq.temp_111102} for $a=2k$ and $l=2m$, we have
\begin{align*}
    Q(2k+1, 2m+1)=q_{2m,1}\cdot Q(2k, 2m+2)+q_{2m,2}\cdot Q(2k, 2m).
\end{align*}
By Eq.~\eqref{eq.temp_111103}, the result of Eq.~\eqref{eq.temp_111101} thus follows.
% }
\end{proof}

Notice that 
\begin{align*}
    c_{\hf}(l,\bm{0})=&\int_{\sd} \hf(\xsmall) \Xi_{\bm{0}}^l(\xsmall) d\xdensity(\xsmall)\\
    =&\frac{1}{2d}\int_{\sd}2d\cdot\KThree(\xsmall^T\bm{e}_d)\Xi_{\bm{0}}^l(\xsmall) d \xdensity(\xsmall)\text{ (by Eq.~\eqref{eq.h_in_convolution})}\\
    =&\frac{1}{2d}\sum_{k=0}^\infty \taylorCoeff_k \cdot Q(k, l)\text{ (by Eq.~\eqref{eq.temp_051301} and Eq.~\eqref{eq.def_Qab})}\\
    >&0\text{ (by Lemma~\ref{le.ck_positive} and Lemma~\ref{le.Q_positive})}.
\end{align*}
The result of Proposition~\ref{prop.coefficients} thus follows.
\section{Proof of Proposition~\ref{prop.compare_efficient_and_set}}\label{app.proof_compare_set}
\begin{proof}
Using similar decomposition in Eq.~\eqref{eq.convolution}, we define filter functions $\hft$ (for 2-layer NTK, no-bias) and $\hftb$ (for 2-layer NTK, with bias).
% $\hftnb$ (2-layer NTK, normal-bias), and $\hftbb$ (2-layer NTK, balanced-bias). 
% The corresponding harmonic coefficients are denoted by $c_{\hft}$, $c_{\hftnb}$, and $c_{\hftbb}$.
The corresponding harmonic coefficients are denoted by $c_{\hft}$ and $c_{\hftb}$. We have the following result about the magnitude of those harmonic coefficients.

\begin{lemma}\label{le.coefficient_compare}
For any $\bOne,\bTwo\in(0, 1)$, we must have  $c_{\hftbA}(2k,\bm{0})=\Theta\left(c_{\hft}(2k,\bm{0})\right)$, $c_{\hftbA}(k,\bm{0})=\Theta\left(c_{\hftbB}(k,\bm{0})\right)$, and  $c_{\hf}(k,\bm{0})=\Omega\left(\harmonicCoeff_{\hftbA}(k,\bm{0})\right)$. Here, $\Theta(\cdot)$ and $\Omega(\cdot)$ denote the orders as $k$ becomes large.
\end{lemma}
We prove Lemma~\ref{le.coefficient_compare} in Appendix~\ref{app.proof_coeff_set}.

Lemma~\ref{le.coefficient_compare} has the following implications for the magnitude of the harmonics coefficients when the leading index of harmonics is large (i.e., $k$ in Lemma~\ref{le.coefficient_compare} is large). The first statement states that, for 2-layer NTK, the setting with bias and the setting without bias have the same order of harmonics coefficients for even terms. (For odd terms, recall that for 2-layer NTK without bias, the coefficients of odd terms except linear term are zero. In contrast, for 2-layer NTK with bias, the coefficients of odd terms are not zero \citep{ju2021generalization}. Hence, the first statement does not hold for odd terms). The second statement states that, the coefficients of harmonics for 2-layer NTK have the same order with respect to $k$ for all non-zero bias. The third statement states that, the coefficients for 3-layer NTK even without bias is not smaller (in order) than 2-layer NTK with bias.

By comparing the magnitude of these filter coefficients, we can then compare whether polynomials with infinite degree belong to each of the learnable sets. Specifically, consider an infinite-degree polynomial with the form $\fg(\xsmall)=\sum_{l,\mathbf{K}}\alpha_{l,\mathbf{K}}\cdot \Xi_{\mathbf{K}}^l (\xsmall)$. By Eq.~\eqref{eq.freq_product}, we have $g(\bm{z})=\sum_{l,\mathbf{K}} \frac{\alpha_{l,\mathbf{K}}\cdot \Xi_{\mathbf{K}}^l(\bm{z})}{\Lambda\cdot c_{h}(l,\bm{0})}$, where $h$ can be $\hf$, $\hft$, or $\hftb$. Thus, the magnitude of $c_h(l,\bm{0})$ determines the norm of $g$. Specifically, we have $\|g\|_2=\sum_{l,\mathbf{K}} \frac{\alpha_{l,\mathbf{K}}}{\Lambda\cdot c_{h}(l,\bm{0})}$ due to the orthogonality of harmonics $\Xi_{\mathbf{K}}^l$. Note that $c_{\hftbA}(k,\bm{0})=\Theta\left(c_{\hftbB}(k,\bm{0})\right)$ by Lemma~\ref{le.coefficient_compare}. Thus, if $\|g\|_2$ is finite for 2-layer NTK with bias $b_1>0$, then it must also be finite for 2-layer NTK with a different bias $b_2>0$. This implies that The learnable sets of 2-layer NTK with different bias settings are the same i.e., $\learnableSetTwoBiasA=\learnableSetTwoBiasB$ for any $b_1,b_2\in (0,1)$. Similarly, for 3-layer NTK, by Lemma~\ref{le.coefficient_compare}, we can also show that $\learnableSetTwo\cup \learnableSetTwoBias \subseteq \learnableSet$ and $\learnableSetTwo\subset \learnableSetTwoBias$. Therefore, the result of Proposition~\ref{prop.compare_efficient_and_set} follows.
\end{proof}
% \section{Proofs Related to Decomposition}\label{app.decomposition}

\section{Proof of Lemma~\ref{le.coefficient_compare}}\label{app.proof_coeff_set}
The following lemma shows the relationship between harmonic coefficients and Taylor coefficients.
\begin{lemma}\label{le.coeff_same_taylor}
Consider two polynomial functions $h_{\alpha}(\xsmall)\defeq \sum_{k=0}^\infty \taylorCoeff_{\alpha,k} (\xsmall^T\bm{e}_d)^k$ and $h_{\beta}(\xsmall)\defeq \sum_{k=0}^\infty \taylorCoeff_{\beta,k} (\xsmall^T\bm{e}_d)^k$ where $\taylorCoeff_{\alpha,k}\geq 0$ and $\taylorCoeff_{\beta,k}\geq 0$ for all $k$. Let $\harmonicCoeff_{h_{\alpha}}$ and $\harmonicCoeff_{h_{\beta}}$ denote their harmonic coefficients. If $\taylorCoeff_{\alpha,k}=O(\taylorCoeff_{\beta,k})$ (where $O(\cdot)$ denotes the order when $k$ is large), then $\harmonicCoeff_{h_{\alpha}}(l,\bm{0})=O(\harmonicCoeff_{h_{\beta}}(l,\bm{0}))$ for large $l$. The same is true if we restrict to only even harmonics, i.e., if $\taylorCoeff_{\alpha,2k}=O(\taylorCoeff_{\beta,2k})$, then $\harmonicCoeff_{h_{\alpha}}(2l,\bm{0})=O(\harmonicCoeff_{h_{\beta}}(2l,\bm{0}))$ for large $l$. This lemma also holds if $O(\cdot)$ is replaced by $\Omega(\cdot)$ or $\Theta(\cdot)$.
% Consider a function $h(\xsmall)\defeq \sum_{k=0}^\infty \taylorCoeff_{k} (\xsmall^T\bm{e}_d)^k$ where $\taylorCoeff_{k}>0$.
% Let $\taylorCoeff_{\hft,k}$ denote the Taylor coefficients of $\hft(\xsmall)$, i.e., $\hft(\xsmall)=\sum_{k=0}^\infty \taylorCoeff_{\hft,k}(\xsmall^T\bm{e}_d)^k$.
% If $\taylorCoeff_{2k}=O(\taylorCoeff_{\hft,2k})$ and $\taylorCoeff_{2k+1}=O(\taylorCoeff_{\hft,2k})$ with respect to $k$, then $\harmonicCoeff_{h}(2k,\bm{0})=O(\harmonicCoeff_{\hft}(2k,\bm{0}))$ and $\harmonicCoeff_{h}(2k+1,\bm{0})=O(\harmonicCoeff_{\hft}(2k,\bm{0}))$ with respect to $k$. It is also true if $O(\cdot)$ is replaced by $\Omega(\cdot)$ or $\Theta(\cdot)$. 
\end{lemma}
\begin{proof}
Notice that
\begin{align*}
    \harmonicCoeff_{h_{\alpha}}(l,\bm{0})=&\int_{\sd} h(\xsmall) \Xi_{\bm{0}}^{l}(\xsmall) d\xdensity(\xsmall)\\
    =&\sum_{k=0}^\infty \taylorCoeff_{\alpha,k} \cdot Q(k,l)\text{ (by Eq.~\eqref{eq.def_Qab})}\\
    =& \sum_{k=l}^\infty \taylorCoeff_{\alpha,k}\cdot Q(k, l)\text{ (by Lemma~\ref{le.Q_positive})}.
\end{align*}
and
\begin{align*}
    \harmonicCoeff_{h_{\alpha}}(2l,\bm{0})=&\int_{\sd} h(\xsmall) \Xi_{\bm{0}}^{2l}(\xsmall) d\xdensity(\xsmall)\\
    =&\sum_{k=0}^\infty \taylorCoeff_{\alpha,k} \cdot Q(k,2l)\text{ (by Eq.~\eqref{eq.def_Qab})}\\
    =& \sum_{k=l}^\infty \taylorCoeff_{\alpha,2k}\cdot Q(2k, 2l)\\
    &\text{ ($Q(k,2l)$ is non-zero only when $k$ is even and not smaller than $2l$ by Lemma~\ref{le.Q_positive})}.
\end{align*}
Similarly, we have
\begin{align*}
    &\harmonicCoeff_{h_{\beta}}(l,\bm{0})=\sum_{k=l}^\infty \taylorCoeff_{\beta,k}\cdot Q(k,l),\\
    &\harmonicCoeff_{h_{\beta}}(2l,\bm{0})=\sum_{k=l}^\infty \taylorCoeff_{\beta,2k}\cdot Q(2k,2l).
\end{align*}
Notice that all $Q(\cdot,\cdot)$ and $\taylorCoeff_{\cdot,\cdot}$ are all non-negative. Thus, we have
\begin{align*}
    \frac{\harmonicCoeff_{h_{\alpha}}(l,\bm{0})}{\harmonicCoeff_{h_{\beta}}(l,\bm{0})}\in\left[\min_{k\geq l} \frac{\taylorCoeff_{\alpha,k}}{\taylorCoeff_{\beta,k}},\ \max_{k\geq l} \frac{\taylorCoeff_{\alpha,k}}{\taylorCoeff_{\beta,k}}\right],\qquad \frac{\harmonicCoeff_{h_{\alpha}}(2l,\bm{0})}{\harmonicCoeff_{h_{\beta}}(2l,\bm{0})}\in\left[\min_{k\geq l} \frac{\taylorCoeff_{\alpha,2k}}{\taylorCoeff_{\beta,2k}},\ \max_{k\geq l} \frac{\taylorCoeff_{\alpha,2k}}{\taylorCoeff_{\beta,2k}}\right].
\end{align*}
The result of this lemma thus follows.
% Similarly, we have
% \begin{align*}
%     \harmonicCoeff_{\hft}(2k,\bm{0})=\sum_{i=k}^\infty \taylorCoeff_{\hft,2i}\cdot Q(2i,2k),
% \end{align*}
% and
% \begin{align*}
%     \harmonicCoeff_{h}(2k+1,\bm{0})=\sum_{i=k}^\infty \taylorCoeff_{2i+1}\cdot Q(2i+1,2k+1).
% \end{align*}
% By Eq.~\eqref{eq.temp_111201}, we have
% $Q(2i+1,2k+1)=q_{2k,1}\cdot Q(2i,2k+2)+q_{2k,2}\cdot Q(2i,2k)$. Thus, we have
% \begin{align*}
%     \harmonicCoeff_{h}(2k+1,\bm{0})=&\sum_{i=k}^\infty \taylorCoeff_{2i+1}\left(q_{2k,1}\cdot Q(2i,2k+2)+q_{2k,2}\cdot Q(2i,2k)\right)\\
%     =& \taylorCoeff_{2k+1}q_{2k,2}Q(2k,2k)+\sum_{i=k}^\infty 
% \end{align*}
\end{proof}

With Lemma~\ref{le.coeff_same_taylor}, in order to show Lemma~\ref{le.coefficient_compare}, it is equivalent to compare Taylor coefficients of the expression of different kernels $\hf$, $\hft$, and $\hftb$. Specifically, we are looking at the Taylor coefficients of the following expression:
\begin{align}\label{eq.temp_012403}
    T(x)\defeq K(x)\frac{\pi-\arccos(K(x))}{2\pi}.
\end{align}
For 2-layer NTK, $K(x)=(1-a)x+a$ where $a\in [0, 1)$ corresponds to different choices of bias ($a=0$ corresponds to no bias). For 3-layer NTK (no bias), we have $K(x)=2d\cdot \KRF(x)$ (we neglect the constant $1/(2d)$ in $\KThree$, which does not change its order.) In other words, we have
\begin{align*}
    T(x)=\begin{cases}
    \KTwo(x),&\text{ if }K(x)=x\text{ (i.e., 2-layer NTK no bias)},\\
    \KTwo((1-a)+a),&\text{ if }K(x)=(1-a)x+a\text{ (i.e., 2-layer NTK with bias $a>0$)},\\
    2d\cdot\KThree(x),&\text{ if }K(x)=2d\cdot \KRF(x)\text{ (i.e., 3-layer NTK no bias)}.
    \end{cases}
\end{align*}
When $K(x)=x$, we already have the exact form of the Taylor expansion of $T(x)$ by Lemma~\ref{le.taylor_of_kernel}. However, when $K(x)$ is a polynomial, it is not easy to get the close form of Taylor coefficients. We will first estimate the Taylor coefficients when $K(x)=(1-a)x+a$ in Appendix~\ref{app.bias_Taylor_expansion}. Second, we will estimate the Taylor coefficients when $K(x)$ is a polynomial with finite degree in Appendix~\ref{app.finite_Taylor_expansion}. Last, we will estimate the case of $K(x)=2d\cdot \KRF(x)$ in Appendix~\ref{app.inf_Taylor_expansion}. (By Lemma~\ref{le.taylor_of_kernel}, we know that $\KRF(x)$ is a polynomial with infinite degree.) The result of Lemma~\ref{le.coefficient_compare} then follows from these estimates.

\subsection{Harmonic coefficients for 2-layer NTK with bias}\label{app.bias_Taylor_expansion}
After adding bias, the kernel of 2-layer NTK changes from $x\frac{\pi-\arccos x}{2\pi}$
to $\left((1-a)x+a\right)\frac{\pi-\arccos\left((1-a)x+a\right)}{2\pi}$.
Here $a>0$ denotes the bias setting ($a=0$ corresponds to the no-bias setting).
By Lemma~\ref{le.coeff_same_taylor}, we only need to investigate the relationship between the Taylor coefficients.
We define $\taylorCoeff_{a,m}$ as the Taylor coefficients under the bias setting, i.e.,
\begin{align*}
    \left((1-a)x+a\right)\frac{\pi-\arccos\left((1-a)x+a\right)}{2\pi}=\sum_{m=0}^\infty \taylorCoeff_{a,m}\cdot x^m.
\end{align*}
When $a=0$, $\taylorCoeff_{a,m}$ becomes $\taylorCoeff_{0, m}$ and corresponds to the no-bias setting.
\begin{lemma}\label{le.\taylorCoeff_a_over_\taylorCoeff_0}
For any $k\in\{2,3,\cdots\}$ and any $a\in[0,1)$, we must have $\frac{\taylorCoeff_{a,2k}}{\taylorCoeff_{0,2k}}\in$
\begin{align*}
     \left[\frac{1}{\left(2+2\left(\frac{2a}{1-a}+1\right)\right)^2} \frac{1}{1 + \frac{\left(\frac{1+a}{2}\right)^2}{1-\left(\frac{1+a}{2}\right)^2}}\frac{1}{1-a^2}\cdot \frac{1+\left(\frac{1-a}{1+a}\right)^{2k+1}}{1+\frac{1-a}{1+a}},\ \frac{1}{1-a^2}\cdot \frac{1+\left(\frac{1-a}{1+a}\right)^{2k+1}}{1+\frac{1-a}{1+a}}\right],
\end{align*}
and $\frac{\taylorCoeff_{a,2k-1}}{\taylorCoeff_{0,2k}}\in$
\begin{align*}
     \left[\frac{1}{\left(2+2\left(\frac{2a}{1-a}+1\right)\right)^2} \frac{1}{1 + \frac{\left(\frac{1+a}{2}\right)^2}{1-\left(\frac{1+a}{2}\right)^2}}\frac{1}{1-a^2}\cdot \frac{1-\left(\frac{1-a}{1+a}\right)^{2k}}{1+\frac{1-a}{1+a}},\ \frac{1}{1-a^2}\cdot \frac{1-\left(\frac{1-a}{1+a}\right)^{2k}}{1+\frac{1-a}{1+a}}\right].
\end{align*}
\end{lemma}

We prove Lemma~\ref{le.\taylorCoeff_a_over_\taylorCoeff_0} in Appendix~\ref{proof.\taylorCoeff_a_over_\taylorCoeff_0}.

Note that when $k\to\infty$, the terms that depend on $k$ (i.e., $\left(\frac{1-a}{1+a}\right)^{2k+1}$ and $\left(\frac{1-a}{1+a}\right)^{2k}$) all approach $0$. In other words, as $k$ becomes larger, $\taylorCoeff_{a,2k}$ (as well as $\taylorCoeff_{a,2k+1}$) approaches (approximately) a constant (that only depends on $a$) multiple of $\taylorCoeff_{0,2k}$. Therefore,
by Lemma~\ref{le.\taylorCoeff_a_over_\taylorCoeff_0}, we can then conclude that $\taylorCoeff_{b_1,2k}=\Theta(\taylorCoeff_{0,2k})$ and $\taylorCoeff_{b_1,k}=\Theta(\taylorCoeff_{b_2,k})$ when $k$ is large for any $b_1\in (0,1)$ and $b_2\in (0,1)$. By Lemma~\ref{le.coeff_same_taylor}, it immediately implies that $c_{\hftbA}(2k,\bm{0})=\Theta\left(c_{\hft}(2k,\bm{0})\right)$, $c_{\hftbA}(k,\bm{0})=\Theta\left(c_{\hftbB}(k,\bm{0})\right)$. This proves the first and second statements of Lemma~\ref{le.coefficient_compare}.

\subsubsection{Proof of Lemma~\ref{le.\taylorCoeff_a_over_\taylorCoeff_0}}\label{proof.\taylorCoeff_a_over_\taylorCoeff_0}

We first write the form of $\taylorCoeff_{0,l}$, i.e., the Taylor coefficients under no-bias setting.
By Lemma~\ref{le.taylor_of_kernel}, we have
\begin{align*}
    x\frac{\pi-\arccos x}{2\pi}=\frac{x}{4}+\frac{1}{2\pi}\sum_{k=0}^\infty \frac{(2k)!}{(k!)^2}\frac{4}{2k+1}\left(\frac{x}{2}\right)^{2k+2}.
\end{align*}
Thus, for $k\geq 1$, we have
\begin{align}
    &\taylorCoeff_{0,2k}=\frac{1}{2\pi}\frac{(2k-2)!}{((k-1)!)^2}\frac{4}{2k-1}\frac{1}{2^{2k}},\label{eq.temp_061806}\\
    &\taylorCoeff_{0,2k+1}=0.\label{eq.temp_061805}
\end{align}

Next, we write the expression of $\taylorCoeff_{a,l}$. To that end, we define
\begin{align}
    &d_{a,2k,i}\defeq \binom{2k+2i}{2k}(1-a)^{2k}a^{2i},\label{eq.def_d_a_2k_i}\\
    &d_{a,2k+1, i}\defeq \binom{2k+2i+2}{2k+1}(1-a)^{2k+1}a^{2i+1}.\label{eq.def_d_a_2k1_i}
\end{align}
The following lemma provides the expression of $\taylorCoeff_{a,l}$.
\begin{lemma}\label{le.\taylorCoeff_d}
For any $k\geq 1$, we must have
\begin{align*}
    &\taylorCoeff_{a,2k}=\sum_{i=0}^\infty \taylorCoeff_{0,2(k+i)}d_{a,2k,i},\\
    &\taylorCoeff_{a,2k+1}=\sum_{i=0}^\infty \taylorCoeff_{0,2(k+i+1)}d_{a,2k+1,i}.
\end{align*}
\end{lemma}
We prove Lemma~\ref{le.\taylorCoeff_d} in Appendix~\ref{proof.\taylorCoeff_d}.

Although we have the expression of $\taylorCoeff_{a,k}$ by Lemma~\ref{le.\taylorCoeff_d}, it is not easy to directly estimate its value because terms like $\taylorCoeff_{0,2(k+i)}d_{a,2k,i}$ have a very complicated form. Fortunately, some properties of $\taylorCoeff_{0,2k}$ is very helpful. Specifically, by Eq.~\eqref{eq.temp_061806}, we have $\frac{\taylorCoeff_{0,2(k+1)}}{\taylorCoeff_{0,2k}}=\frac{(2k-1)^2 2k}{k^2(2k+1)\cdot 4}=\frac{(2k-1)^2}{2k(2k+1)}$, whose value approaches $1$ when $k\to\infty$.
In other words, $\taylorCoeff_{0,2k}$ has a very slow changing speed when $k$ is large. Therefore, we can approximate the tail of $\sum_{i=0}^\infty \taylorCoeff_{0,2(k+i)}d_{a,2k,i}$ by treating $\taylorCoeff_{0,2(k+i)}$ as a constant. This allows us to focus our attention on estimating $\sum_{i=0}^\infty d_{a,2k,i}$ (and its tail $\sum_{i=l}^\infty d_{a,2k,i}$), whose value can be calculated by examining the coefficients of the Taylor expansion of $\frac{1}{1-\left((1-a)x+a\right)^2}$ (i.e., the sum of a geometric sequence $1, \left((1-a)x+a\right)^2, , \left((1-a)x+a\right)^4,\cdots$). The latter is much easier to study. We show these steps in detail as the following lemmas.

Define
\begin{align}\label{eq.def_l}
    l\defeq \max\left\{k,\ \left\lceil\frac{2a}{1-a}k\right\rceil\right\}.
\end{align}
The following lemma estimates the target ratio $\frac{\taylorCoeff_{a,2k}}{\taylorCoeff_{0,2k}}$ and $\frac{\taylorCoeff_{a,2k-1}}{\taylorCoeff_{0,2k}}$ in terms of $\sum_{i=0}^l d_{a,2k,i}$ and $\sum_{i=0}^\infty d_{a,2k,i}$.

\begin{lemma}\label{le.\taylorCoeff_a_\taylorCoeff_0}
For any $k\in\{2,3,\cdots\}$, we must have
\begin{align*}
    &\frac{\taylorCoeff_{a,2k}}{\taylorCoeff_{0,2k}}\in\left[\frac{1}{\left(2+2\left(\frac{2a}{1-a}+1\right)\right)^2} \sum_{i=0}^l d_{a,2k,i},\ \sum_{i=0}^\infty d_{a,2k,i}\right],\\
    &\frac{\taylorCoeff_{a,2k-1}}{\taylorCoeff_{0,2k}}\in\left[\frac{1}{\left(2+2\left(\frac{2a}{1-a}+1\right)\right)^2} \sum_{i=0}^l d_{a,2k-1,i},\ \sum_{i=0}^\infty d_{a,2k-1,i}\right].
\end{align*}
\end{lemma}
We prove Lemma~\ref{le.\taylorCoeff_a_\taylorCoeff_0} in Appendix~\ref{app.proof_le_coeff_a_0}.

In order to finish the proof of Lemma~\ref{le.\taylorCoeff_a_over_\taylorCoeff_0}, it only remains to estimate $\sum_{i=0}^l d_{a,2k,i}$, $\sum_{i=0}^l d_{a,2k-1,i}$, $\sum_{i=0}^\infty d_{a,2k,i}$, $\sum_{i=0}^\infty d_{a,2k-1,i}$, which are shown by the following two lemmas.

\begin{lemma}\label{le.d_inf_sum}
For any $k\in\{2,3,\cdots\}$, we must have
\begin{align*}
    \sum_{i=0}^\infty d_{a,2k,i}=\frac{1}{1-a^2}\cdot \frac{1+\left(\frac{1-a}{1+a}\right)^{2k+1}}{1+\frac{1-a}{1+a}},\quad \sum_{i=0}^\infty d_{a,2k-1,i}=\frac{1}{1-a^2}\cdot\frac{1-\left(\frac{1-a}{1+a}\right)^{2k}}{1+\frac{1-a}{1+a}}.
\end{align*}
\end{lemma}
We prove Lemma~\ref{le.d_inf_sum} in Appendix~\ref{proof.d_inf_sum}

\begin{lemma}\label{le.d_l_d_inf}
Recall that $l$ is defined in Eq.~\eqref{eq.def_l}. For any $k\in\{2,3,\cdots\}$, we must have
\begin{align*}
    \frac{\sum_{i=0}^l d_{a,2k,i}}{\sum_{i=0}^\infty d_{a,2k,i}},\ \frac{\sum_{i=0}^l d_{a,2k-1,i}}{\sum_{i=0}^\infty d_{a,2k-1,i}}\in\left[\frac{1}{1 + \frac{\left(\frac{1+a}{2}\right)^2}{1-\left(\frac{1+a}{2}\right)^2}},\ 1\right].
\end{align*}
\end{lemma}
We prove Lemma~\ref{le.d_l_d_inf} in Appendix~\ref{proof.d_l_d_inf}

The result of Lemma~\ref{le.\taylorCoeff_a_over_\taylorCoeff_0} follows by combining Lemma~\ref{le.d_inf_sum}, Lemma~\ref{le.d_l_d_inf}, and Lemma~\ref{le.\taylorCoeff_a_\taylorCoeff_0}.

\subsubsection{Proof of Lemma~\ref{le.\taylorCoeff_d}}\label{proof.\taylorCoeff_d}
\begin{proof}
We have
\begin{align*}
    &\left((1-a)x+a\right)\frac{\pi-\arccos\left((1-a)x+a\right)}{2\pi}\\
    =&\sum_{m=0}^\infty \taylorCoeff_{0,m}\left((1-a)x+a\right)^m\\
    =&\sum_{m=0}^\infty \taylorCoeff_{0,m}\sum_{i=0}^m \binom{m}{i} (1-a)^i a^{m-i} x^i\\
    =&\sum_{i=0}^\infty \sum_{j=0}^\infty \taylorCoeff_{0,i+j} \binom{i+j}{i}(1-a)^i a^j x^i\text{ (replace $m$ by $i+j$ and reorganize terms)}.
\end{align*}
Thus, we have
\begin{align*}
    \taylorCoeff_{a, i}=\sum_{j=0}^\infty \taylorCoeff_{0,i+j}\binom{i+j}{i}(1-a)^i a^j.
\end{align*}
Letting $i=2k$, we have
\begin{align*}
    \taylorCoeff_{a,2k}=&\sum_{j=0}^\infty \taylorCoeff_{0,2k+j}\binom{2k+j}{2k}(1-a)^{2k} a^j\\
    =&\sum_{i=0}^\infty \taylorCoeff_{0,2(k+i)}\binom{2k+2i}{2k}(1-a)^{2k} a^{2i}\text{ (by Eq.~\eqref{eq.temp_061805} and letting $j=2i$)}\\
    =&\sum_{i=0}^\infty \taylorCoeff_{0,2(k+i)}d_{a,2k,i}\text{ (by Eq.~\eqref{eq.def_d_a_2k_i})}.
\end{align*}
Similarly, we have
\begin{align*}
    \taylorCoeff_{a,2k+1}=&\sum_{j=0}^\infty \taylorCoeff_{0,2k+1+j}\binom{2k+1+j}{2k+1}(1-a)^{2k+1} a^j\\
    =&\sum_{i=0}^\infty \taylorCoeff_{0,2(k+i+1)}\binom{2k+2i+2}{2k+1}(1-a)^{2k+1} a^{2i+1}\text{ (by Eq.~\eqref{eq.temp_061805} and letting $j=2i+1$)}\\
    =&\sum_{i=0}^\infty \taylorCoeff_{0,2(k+i+1)}d_{a,2k+1,i}\text{ (by Eq.~\eqref{eq.def_d_a_2k1_i})}.
\end{align*}
The result of this lemma thus follows.
\end{proof}

\subsubsection{Proof of Lemma~\ref{le.\taylorCoeff_a_\taylorCoeff_0}}\label{app.proof_le_coeff_a_0}
\begin{proof}
By Eq.~\eqref{eq.temp_061806}, we have
\begin{align}
    \frac{\taylorCoeff_{0,2(k+1)}}{\taylorCoeff_{0,2k}}=\frac{(2k-1)^2 2k}{k^2(2k+1)\cdot 4}&=\frac{(2k-1)^2}{2k(2k+1)}\label{eq.temp_070501}\\
    &\leq 1.\nonumber
\end{align}
By iterating the above inequality, we have $\taylorCoeff_{0,2(k+i)}\leq \taylorCoeff_{0,2k}$ for all $i\geq 0$.
By Lemma~\ref{le.\taylorCoeff_d}, we thus have
\begin{align*}
    \frac{\taylorCoeff_{a,2k}}{\taylorCoeff_{0,2k}}=\frac{\sum_{i=0}^\infty \taylorCoeff_{0,2(k+i)}d_{a,2k,i}}{\taylorCoeff_{0,2k}}\leq\sum_{i=0}^\infty d_{a,2k,i}.
\end{align*}
Similarly, we have
\begin{align*}
    \frac{\taylorCoeff_{a,2k-1}}{\taylorCoeff_{0,2k}}=\frac{\sum_{i=0}^\infty \taylorCoeff_{0,2(k+i)}d_{a,2k-1,i}}{\taylorCoeff_{0,2k}}\leq\sum_{i=0}^\infty d_{a,2k-1,i}.
\end{align*}
These prove the upper bounds in Lemma~\ref{le.\taylorCoeff_a_\taylorCoeff_0}. To prove the lower bounds, note that
for any $m\in\{0,1,\cdots,l\}$ (recall that $l$ is defined in Eq.~\eqref{eq.def_l}), we must have
\begin{align*}
    \frac{\taylorCoeff_{0,2(k+m)}}{\taylorCoeff_{0, 2k}}=&\prod_{i=0}^{m-1} \frac{(2k+2i-1)^2}{(2k+2i)(2k+2i+1)}\text{ (by Eq.~\eqref{eq.temp_070501})}\\
    \geq& \prod_{i=0}^{m-1} \frac{(2k+2i-1)^2}{(2k+2i+1)^2}\\
    =&\frac{(2k-1)^2}{(2k+2m-1)^2}\\
    \geq &\frac{k^2}{(2k+2m)^2}\text{ (using $2k-1\geq k$, which is true because $k\geq 1$)}\\
    =&\frac{1}{\left(2+\frac{2m}{k}\right)^2}\\
    \geq & \frac{1}{\left(2+2\left(\frac{2a}{1-a}+1\right)\right)^2}\text{ (because $m\leq l\leq \frac{2a}{1-a}k+k$ for $k\geq 2$)}.
\end{align*}
Thus, we have
\begin{align*}
    \frac{\taylorCoeff_{a,2k}}{\taylorCoeff_{0,2k}}=& \frac{\sum_{i=0}^\infty \taylorCoeff_{0,2(k+i)}d_{a,2k,i}}{\taylorCoeff_{0,2k}}\text{ (by Lemma~\ref{le.\taylorCoeff_d})}\\
    \geq &\frac{\sum_{i=0}^l \taylorCoeff_{0,2(k+i)}d_{a,2k,i}}{\taylorCoeff_{0,2k}}\\
    \geq & \frac{1}{\left(2+2\left(\frac{2a}{1-a}+1\right)\right)^2} \sum_{i=0}^l d_{a,2k,i}.
\end{align*}
Similarly, we have
\begin{align*}
    \frac{\taylorCoeff_{a,2k-1}}{\taylorCoeff_{0,2k}}\geq \frac{1}{\left(2+2\left(\frac{2a}{1-a}+1\right)\right)^2} \sum_{i=0}^l d_{a,2k-1,i}.
\end{align*}
The result of this lemma thus follows.
\end{proof}

\subsubsection{Proof of Lemma~\ref{le.d_inf_sum}}\label{proof.d_inf_sum}

We first state a useful fact.

\begin{lemma}\label{le.geo_sum}
For any $|r|<1$, we have
\begin{align*}
    \frac{1}{1-r}=\sum_{i=0}^\infty r^i.
\end{align*}
\end{lemma}
\begin{proof}
The result of this lemma directly follows the sum of a geometric series (noticing that $\lim_{i\to\infty}r^i = 0$ when $|r|<1$).
\end{proof}

\begin{proof}[Proof of Lemma~\ref{le.d_inf_sum}]
The proof idea is to express the coefficients of the Taylor expansion of $\frac{1}{1-\left((1-a)x+a\right)^2}$ (where $|x|<1$) in two different ways.
On the one hand, we have
\begin{align}
    &\frac{1}{1-\left((1-a)x+a\right)^2}\nonumber\\
    =&\sum_{m=0}^\infty \left((1-a)x+a\right)^{2m}\text{ (by letting $r=(1-a)x+a$ in Lemma~\ref{le.geo_sum})}\nonumber\\
    =&\sum_{m=0}^\infty \sum_{j=0}^{2m} \binom{2m}{j}(1-a)^j a^{2m-j} x^j\nonumber\\
    =& \sum_{k=0}^\infty\left(\left(\sum_{i=0}^\infty \binom{2k+2i}{2k}(1-a)^{2k} a^{2i} \right) x^{2k}+\left(\sum_{i=0}^\infty\binom{2k+2i+2}{2k+1}(1-a)^{2k+1}a^{2i+1}\right)x^{2k+1}\right)\nonumber\\
    &\text{ (by letting $j=2k$, $2m=2k+2i$ for $x^{2k}$ and letting $j=2k+1$, $2m=2k+2i+2$ for $x^{2k+1}$)}\nonumber\\
    =&\sum_{k=0}^\infty\left(\left(\sum_{i=0}^\infty d_{a,2k,i}\right)x^{2k}+\left(\sum_{i=0}^\infty d_{a,2k+1,i}\right) x^{2k+1}\right)\text{ (by Eq.~\eqref{eq.def_d_a_2k_i} and Eq.~\eqref{eq.def_d_a_2k1_i})}.\label{eq.temp_061801}
\end{align}
One the other hand, we have
\begin{align}
    &\frac{1}{1-\left((1-a)x+a\right)^2}\nonumber\\
    =&\frac{1}{1-\left((1-a)x+a\right)}\cdot\frac{1}{1+\left((1-a)x+a\right)}\nonumber\\
    =&\frac{1}{1-a}\cdot\frac{1}{1-x}\cdot \frac{1}{1+a}\cdot \frac{1}{1+\frac{1-a}{1+a}x}\nonumber\\
    =&\frac{1}{1-a^2}\left(\sum_{i=0}^\infty x^i\right)\left(\sum_{j=0}^{\infty}\left(-\frac{1-a}{1+a}x\right)^j\right)\text{ (by Lemma~\ref{le.geo_sum})}\nonumber\\
    =&\frac{1}{1-a^2}\sum_{i=0}^\infty\sum_{j=0}^\infty \left(-\frac{1-a}{1+a}\right)^j x^{i+j}\nonumber\\
    =&\frac{1}{1-a^2}\sum_{m=0}^{\infty}\left(\sum_{j=0}^m \left(-\frac{1-a}{1+a}\right)^j\right)x^m\text{ (combine terms of $x^{i+j}$ with $i+j=m$)}\nonumber\\
    =&\frac{1}{1-a^2}\sum_{m=0}^\infty\frac{1-\left(-\frac{1-a}{1+a}\right)^{m+1}}{1+\frac{1-a}{1+a}}x^m.\label{eq.temp_061802}
\end{align}
By comparing the coefficients in Eq.~\eqref{eq.temp_061801} and Eq.~\eqref{eq.temp_061802}, the result of this lemma thus follows.
\end{proof}

\subsubsection{Proof of Lemma~\ref{le.d_l_d_inf}}\label{proof.d_l_d_inf}
We first prove a useful lemma.
\begin{lemma}\label{le.temp_061401}
For any $a>b>c>0$, we have
\begin{align*}
    \frac{a}{b}< \frac{a-c}{b-c}.
\end{align*}
\end{lemma}
\begin{proof}
Because $a>b>c>0$, we have
\begin{align*}
    b<a\implies bc<ac \implies ab-ac<ab-bc\implies a(b-c)<b(a-c)\implies \frac{a}{b}<\frac{a-c}{b-c}.
\end{align*}
\end{proof}

Now we are ready to prove Lemma~\ref{le.d_l_d_inf}.
\begin{proof}[Proof of Lemma~\ref{le.d_l_d_inf}]
Recall that $l$ is defined in Eq.~\eqref{eq.def_l}.
For any $i\geq l$, we have
\begin{align}\label{eq.temp_061803}
    \frac{k+i}{i}=\frac{k}{i}+1\leq \frac{k}{l}+1\leq \frac{1+a}{2a}\text{ (because $l\geq \frac{2a}{1-a}k$ by Eq.~\eqref{eq.def_l})}.
\end{align}
Thus, by Eq.~\eqref{eq.def_d_a_2k_i}, we have
\begin{align*}
    \frac{d_{a,2k,i+1}}{d_{a,2k,i}}=\frac{\binom{2k+2i+2}{2k}}{\binom{2k+2i}{2k}}a^2=&\frac{(2k+2i+1)(2k+2i+2)}{(2i+1)(2i+2)}a^2\\
    \leq& \frac{(2k+2i)(2k+2i)}{(2i)\cdot (2i)}a^2\text{ (by Lemma~\ref{le.temp_061401})}\\
    \leq& \left(\frac{1+a}{2}\right)^2\text{ (by Eq.~\eqref{eq.temp_061803})}.
\end{align*}
Similarly, by Eq.~\eqref{eq.def_d_a_2k1_i}, we have
\begin{align*}
    \frac{d_{a,2k-1,i+1}}{d_{a,2k-1,i}}=\frac{\binom{2k+2i+2}{2k-1}}{\binom{2k+2i}{2k-1}}a^2=&\frac{(2k+2i+2)(2k+2i+1)}{(2i+2)(2i+3)}a^2\\
    \leq & \frac{(2k+2i+2)(2k+2i+3)}{(2i+2)(2i+3)}a^2\\
    \leq & \frac{(2k+2i)(2k+2i)}{(2i)\cdot (2i)}a^2\text{ (by Lemma~\ref{le.temp_061401})}\\
    \leq & \left(\frac{1+a}{2}\right)^2\text{ (by Eq.~\eqref{eq.temp_061803})}.
\end{align*}
Iterating the above inequalities, we have
\begin{align*}
    \frac{d_{a,2k, l+j}}{d_{a, 2k, l}}\leq \left(\frac{1+a}{2}\right)^{2j},\quad \text{and }\frac{d_{a,2k-1, l+j}}{d_{a, 2k-1, l}}\leq \left(\frac{1+a}{2}\right)^{2j}.
\end{align*}
Thus, we have
\begin{align*}
    \frac{\sum_{i=l+1}^\infty d_{a,2k,i}}{\sum_{i=0}^l  d_{a,2k,i}}\leq\frac{\sum_{i=l+1}^\infty d_{a,2k,i}}{d_{a,2k,l}}=\sum_{j=1}^\infty \frac{d_{a,2k,l+j}}{d_{a,2k,l}}\leq \sum_{j=1}^\infty \left(\frac{1+a}{2}\right)^{2j}=\frac{\left(\frac{1+a}{2}\right)^2}{1-\left(\frac{1+a}{2}\right)^2}.
\end{align*}
We then have
\begin{align*}
    \frac{\sum_{i=0}^\infty d_{a,2k,i}}{\sum_{i=0}^l d_{a,2k,i}}= \frac{\sum_{i=0}^l d_{a,2k,i}+\sum_{i=l+1}^\infty d_{a,2k,i}}{\sum_{i=0}^l d_{a,2k,i}}\leq 1 + \frac{\left(\frac{1+a}{2}\right)^2}{1-\left(\frac{1+a}{2}\right)^2}.
\end{align*}
Therefore, we conclude that
\begin{align*}
    \frac{\sum_{i=0}^l d_{a,2k,i}}{\sum_{i=0}^\infty d_{a,2k,i}}\in\left[\frac{1}{1 + \frac{\left(\frac{1+a}{2}\right)^2}{1-\left(\frac{1+a}{2}\right)^2}},\ 1\right].
\end{align*}
Similarly, we have
\begin{align*}
    \frac{\sum_{i=0}^l d_{a,2k-1,i}}{\sum_{i=0}^\infty d_{a,2k-1,i}}\in\left[\frac{1}{1 + \frac{\left(\frac{1+a}{2}\right)^2}{1-\left(\frac{1+a}{2}\right)^2}},\ 1\right].
\end{align*}
\end{proof}

\subsection{Expansion for a finite-degree polynomial}\label{app.finite_Taylor_expansion}

We plan to show the third statement of Lemma~\ref{le.coefficient_compare}, which was presented in Appendix~\ref{app.proof_compare_set}. This is more difficult because $\KRF(x)$ is an infinite-degree polynomial (in contrast, Appendix~\ref{app.bias_Taylor_expansion} deals with $K(x)=(1-a)x+a$, which is much simpler). To make progress, we first consider a finite-degree polynomial, and study the expansion. Then, in Appendix~\ref{app.inf_Taylor_expansion}, we will extend to $\KRF(x)$ which has infinite degree. Note that Appendix~\ref{app.bias_Taylor_expansion} is a special case of Appendix~\ref{app.finite_Taylor_expansion}. However, since Appendix~\ref{app.bias_Taylor_expansion} is much simpler and easy to understand, we retain the proof there, and use the result in Appendix~\ref{app.finite_Taylor_expansion} only as a preparation for Appendix~\ref{app.inf_Taylor_expansion}.

Recall the definition of $T(x)$ in Eq.~\eqref{eq.temp_012403}. We denote $K(x)$ as a polynomial, i.e.,
\begin{align}\label{def.coeff_a}
    K(x)= \sum_{i=0}^\infty a_i x^i,
\end{align}
where $a_i$ denote the coefficient of $x^i$ in $K(x)$.

Define $\taylorCoeff_m(\cdot)$ as a function that projects a polynomial in $x$ to a real value such that
\begin{align}\label{def.coeff_c}
    \tilde{K}(x)\frac{\pi- \arccos\left(\tilde{K}(x)\right)}{2\pi}=\sum_{m=0}^\infty \taylorCoeff_m\left(\tilde{K}(x)\right)\cdot  x^m,
\end{align}
where $\tilde{K}(x)$ is any polynomial of $x$. In other words, $\taylorCoeff_m(\tilde{K}(x))$ is the Taylor coefficient of $x^m$ in $T(x)$ when $K(x)=\tilde{K}(x)$.

In this subsection, we let the number of terms of $K(x)$ be finite, i.e., there exists $s$ such that $a_i=0$ for all $i>s$. Further, we impose the following conditions.

\begin{condition}\label{condition.finite_poly}
(i) All coefficients of $K(x)$ are non-negative, i.e., $a_i\geq 0$ for all $i\in\mathds{Z}_{\geq 0}$. (ii) The sum of all coefficients equals to $1$, i.e., $\sum_{i=0}^\infty a_i= K(1)=1$. (iii) $a_0>0$ and $a_1>0$.
\end{condition}

The following lemma shows that when $K(x)$ is a polynomial with finite terms, the Taylor coefficients are on the same order as that of the even-power Taylor coefficients when $K(x)=x$.
Note that, according to Eq.~\eqref{def.coeff_c}, when $K(x)=x$, $\taylorCoeff_m(x)$ recovers the Taylor coefficients of the polynomial expansion of the function $x\frac{\pi-\arccos(x)}{2\pi}$.
\begin{lemma}\label{le.finite_poly_conclusion}
Under Condition~\ref{condition.finite_poly} and when $a_i=0$ for all $i>s$, we must have
\begin{align*}
    \frac{\taylorCoeff_j\left(K(x)\right)}{\taylorCoeff_{2\left\lceil j/2\right\rceil}(x)}\in \left[\underline{C},\ \overline{C}\right],\text{ for all }j=1,2,\cdots,
\end{align*}
where $\overline{C}>\underline{C}>0$ are constants that only depends on $K(x)$ and are independent of $j$.
\end{lemma}
We prove Lemma~\ref{le.finite_poly_conclusion} in Appendix~\ref{app.finite_poly_conclusion}.
Note that Lemma~\ref{le.finite_poly_conclusion} can be seen as a generalization of Lemma~\ref{le.\taylorCoeff_a_over_\taylorCoeff_0}, since $K(x)=(1-a)+a$ satisfies Condition~\ref{condition.finite_poly} when $a\in(0, 1)$.

\subsubsection{Proof of Lemma~\ref{le.finite_poly_conclusion}}\label{app.finite_poly_conclusion}

We introduce some extra notations.
Let $b_i$ be the coefficients of $x^i$ of $\left(K(x)\right)^2$, i.e.,
\begin{align}\label{def.coeff_b}
    \left(K(x)\right)^2=\sum_{i=0}^{2s} b_i x^i,\text{ which implies that }b_i=\sum_{j+k=i}a_ja_k\text{ for all }i\in \mathds{Z}_{\geq 0}.
\end{align}
As in Lemma~\ref{le.expansion}, for all $j\in\mathds{Z}_{\geq 0}$, we define
\begin{align}
    &t(m_0,m_1,\cdots,m_j)\defeq (m_0,m_1,\cdots,m_j)!\cdot a_0^{m_0}a_1^{m_1}\cdots a_j^{m_j}\text{ (we let $a_i=0$ if $i>s$)},\label{eq.def_tm}\\
    &\mathcal{T}_{i,j}\defeq \left\{(m_0,m_1,\cdots,m_j)\ \bigg|\ \substack{m_0+m_1+\cdots+m_j=i\nonumber\\
    m_1+2m_2+\cdots+j\cdot m_j=j\\m_0,m_1,\cdots,m_j\in\mathds{Z}_{\geq 0}}\right\},\nonumber\\
    &d_{i,j}\defeq \sum_{(m_0,m_1,\cdots,m_j)\in \mathcal{T}_{i,j}}t(m_0,m_1,\cdots,m_j).\label{def.coeff_d}
\end{align}

By Lemma~\ref{le.expansion}, we have
\begin{align*}
    \frac{1}{1-\left(K(x)\right)^2}=\sum_{i=0}^\infty\left(K(x)\right)^{2i}=&\sum_{i=0}^\infty\left( \sum_{j=0}^\infty d_{2i,j} x^j\right)\\
    =&\sum_{j=0}^\infty\left( \sum_{i=\left\lceil\frac{j}{2s}\right\rceil}^\infty d_{2i,j}\right) x^j \text{ (since $d_{2i,j}=0$ when $j>2i\cdot s$)}.
\end{align*}

Define 
\begin{align}\label{def.coeff_e}
    e_j \text{ for all }j\in \mathds{Z}_{\geq 0}\text{ such that }\frac{1}{1-\left(K(x)\right)^2}=\sum_{j=0}^\infty e_j x^j,
\end{align}
i.e.,
\begin{align}\label{def.coeff_e_shown_by_d}
    e_j=\sum_{i=\left\lceil\frac{j}{2s}\right\rceil}^\infty d_{2i,j}.
\end{align}

We summarize those definitions in Table~\ref{table.def_coeff}.

\begin{table}[t!]
\centering
\renewcommand{\arraystretch}{2}
\begin{tabular}{ |c|c|c| }
\hline
\textbf{Notation} & \textbf{Description} & \textbf{Definition/Expression} \\
\hline
$a_i$ & coefficients of $x^i$ in $K(x)$ & Eq.~\eqref{def.coeff_a} \\
\hline
$b_i$ & coefficients of $x^i$ in $\left(K(x)\right)^2$ & Eq.~\eqref{def.coeff_b} \\ 
\hline
% $\tilde{b}_i$ & $b_i/(1-b_0)$ & Eq.~\eqref{def.coeff_tilde_b}\\
% \hline
$\taylorCoeff_i(\tilde{K}(x))$ & coefficients of $x^i$ in $\tilde{K}(x)(\pi- \arccos(\tilde{K}(x)))/(2\pi)$ & Eq.~\eqref{def.coeff_c} and Eq.~\eqref{eq.represent_\taylorCoeff_by_cd} \\
\hline
$d_{i,j}$ & the coefficient of $x^j$ in $\left(K(x)\right)^i$ & Eq.~\eqref{def.coeff_d}\\
\hline
$e_i$ & the coefficient of $x^i$ in $1/(1-(K(x))^2)$ & Eq.~\eqref{def.coeff_e} and Eq.~\eqref{def.coeff_e_shown_by_d}\\
\hline
% $s_i$ & the coefficient of $x^i$ in $\frac{1-x}{1-\left(K(x)\right)^2}$ & Eq.~\eqref{def.s_for_e_inf}\\
% \hline
\end{tabular}
\caption{Summary of the notations of various coefficients.}\label{table.def_coeff}
\renewcommand{\arraystretch}{1}
\end{table}

\begin{lemma}\label{le.requirement_b}
Under Condition~\ref{condition.finite_poly}, we must have 
\begin{align}
    &\sum_{i=0}^{2s}b_i = 1,\label{eq.temp_062502}\\
    &b_i\in [0, 1]\text{ for all }i\in\{0,1,\cdots,2s\},\ b_0<1,\text{ and } b_1>0.\nonumber
\end{align}
\end{lemma}
\begin{proof}
By Eq.~\eqref{def.coeff_b} and Condition~\ref{condition.finite_poly}, we have $\sum_{i=0}^{2s}b_i=\left(K(1)\right)^2=1$. Because $a_1>0$ and $\sum_{i=0}^s a_i=1$, we have $a_0 < 1$. Thus, we have $b_0=a_0^2<1$ and $b_1=2 a_0 a_1>0$ (as $a_0$ is also positive).
\end{proof}

\begin{lemma}\label{le.bounded_e}
There exist $\overline{c}\geq \underline{c}>0$ such that for all $k\in\mathds{Z}_{\geq 0}$, we must have $e_k\in[\underline{c},\ \overline{c}]$.
\end{lemma}
\begin{proof}
Because
\begin{align*}
    \frac{1}{1-\left(K(x)\right)^2}=1+\left(K(x)\right)^2\cdot \frac{1}{1-\left(K(x)\right)^2},
\end{align*}
we have
\begin{align}\label{eq.temp_062501}
    \sum_{j=0}^\infty e_j x^j=&1+\left(\sum_{i=0}^s a_i x^i\right)^2\cdot \left(\sum_{j=0}^\infty e_j x^j\right)=1+\left(\sum_{i=0}^{2s}b_i x^i\right)\cdot \left(\sum_{j=0}^\infty e_j x^j\right).
\end{align}
Comparing the coefficient of $x^{j+2s}$ on both sides, we have
\begin{align*}
    e_{j+2s} = \sum_{i=0}^{2s} e_{j+i}  b_{2s-i}\text{ for all }j=0,1,\cdots.
\end{align*}
This is equivalent to
\begin{align*}
    e_{j+2s} = b_0 e_{j+2s} + \sum_{i=0}^{2s-1} e_{j+i}  b_{2s-i}\text{ for all }j=0,1,\cdots.
\end{align*}
Thus, we have
\begin{align*}
    e_{j+2s}=\sum_{i=0}^{2s-1}\frac{b_{2s-i}}{1-b_0}e_{j+i}\text{ for all }j=0,1,\cdots.
\end{align*}
It implies that
\begin{align*}
    &e_{j+2s}\in \left[\left(\min_{i\in\{0,1,\cdots,2s-1\}}e_{j+i}\right)\cdot \sum_{k=1}^{2s}\frac{b_k}{1-b_0},\ \left(\max_{i\in\{0,1,\cdots,2s-1\}}e_{j+i}\right)\cdot\sum_{k=1}^{2s}\frac{b_k}{1-b_0}\right]\\
    &\text{ for all }j=0,1,\cdots.
\end{align*}
By Eq.~\eqref{eq.temp_062502}, we have $\sum_{k=1}^{2s}b_k=1-b_0$, which implies that $\sum_{k=1}^{2s}\frac{b_k}{1-b_0}=1$. Thus, we have
\begin{align*}
    e_{j+2s}\in \left[\min_{i\in\{0,1,\cdots,2s-1\}}e_{j+i},\ \max_{i\in\{0,1,\cdots,2s-1\}}e_{j+i}\right],\text{ for all }j=0,1,\cdots.
\end{align*}
Iteratively applying the above bounds, we then have
\begin{align*}
    &e_{2s}, e_{2s+1},\cdots, e_{4s-1}\in \left[\min_{i\in\{0,1,\cdots,2s-1\}}e_{i},\ \max_{i\in\{0,1,\cdots,2s-1\}}e_{i}\right],\\
    &e_{4s}, e_{4s+1},\cdots, e_{6s-1}\in \left[\min_{i\in\{2s,2s+1,\cdots,4s-1\}}e_{i},\ \max_{i\in\{2s,2s+1,\cdots,4s-1\}}e_{i}\right]\in \left[\min_{i\in\{0,1,\cdots,2s-1\}}e_{i},\right.\\
    &\hspace{4cm}\left. \max_{i\in\{0,1,\cdots,2s-1\}}e_{i}\right],\\
    &\vdots\\
    &e_{2ks}, e_{2ks+1},\cdots, e_{2ks+2s-1}\in\cdots\in \left[\min_{i\in\{0,1,\cdots,2s-1\}}e_{i},\ \max_{i\in\{0,1,\cdots,2s-1\}}e_{i}\right].
\end{align*}
In other words,
\begin{align}
    e_{k}\in\left[\min_{i\in\{0,1,\cdots,2s-1\}}e_{i},\ \max_{i\in\{0,1,\cdots,2s-1\}}e_{i}\right],\text{ for all }k=2s, 2s+1,\cdots.\label{eq.temp_062504}
\end{align}
By Eq.~\eqref{eq.temp_062501}, we have
\begin{align*}
    &e_0=1+b_0 e_0,\\
    &e_1=b_1 e_0+b_0 e_1,\\
    &e_2=b_2 e_0+b_1e_1+b_0e_2,\\
    &\quad \vdots\\
    &e_{2s}=b_{2s} e_0+b_{2s-1}e_1+\cdots+b_0 e_{2s}.
\end{align*}
Thus, we have
\begin{align*}
    &e_0=\frac{1}{1-b_0},\\
    &e_1=\frac{b_1 e_0}{1-b_0},\\
    &e_2=\frac{b_2 e_0+b_1 e_1}{1-b_0},\\
    &\quad \vdots\\
    &e_{2s}=\frac{b_{2s} e_0+b_{2s-1}e_1+\cdots + b_1 e_{2s-1}}{1-b_0}.
\end{align*}
By Lemma~\ref{le.requirement_b} and using induction, we thus have
\begin{align*}
    e_i>0,\text{ for all }i=0,1,\cdots, 2s-1.
\end{align*}
Thus, by Eq.~\eqref{eq.temp_062504}, we have
\begin{align*}
    e_{k}\in \left[\min_{i\in\{0,1,\cdots,2s-1\}}e_{i},\ \max_{i\in\{0,1,\cdots,2s-1\}}e_{i}\right],\text{ for all }k=0, 1,\cdots.
\end{align*}
\end{proof}

\begin{lemma}\label{le.bijection}
When $i\geq j$, there exists a bijection between $\mathcal{T}_{i,j}$ and $\mathcal{T}_{i+1,j}$. Specifically, this bijection is $\mathcal{T}_{i,j}\longleftrightarrow \mathcal{T}_{i+1,j}$: $(m_0,m_1,\cdots,m_j)\longleftrightarrow (m_0+1,m_1,\cdots,m_j)$.
\end{lemma}
\begin{proof}
It suffices to show that for any $(m_0,m_1,\cdots,m_j)\in \mathcal{T}_{i+1,j}$, we must have $m_0\geq 1$.
To that end, note that when $i\geq j$, for any $(m_0,m_1,\cdots,m_s)\in \mathcal{T}_{i+1,j}$, we have
\begin{align*}
    \sum_{k=0}^j m_k=i+1,\quad  \sum_{k=1}^j k m_k=j.
\end{align*}
Thus, we have
\begin{align*}
    m_0=\sum_{k=0}^j m_k-\sum_{k=1}^j k m_k+\sum_{k=1}^j (k-1) m_k=(i+1-j)+\sum_{k=1}^j (k-1) m_k\geq 1\text{ (because $i\geq j$)}.
\end{align*}
The result of this lemma thus follows.
\end{proof}

\begin{lemma}\label{le.d_to_d_ratio}
If $2i\geq j$, then
\begin{align*}
    \frac{d_{2i+2,j}}{d_{2i,j}}\leq \left(\frac{2i }{2i-j}a_0\right)^2.
\end{align*}
(Notice that when $2i=j$, the right hand side is infinite. Nonetheless, this lemma still holds.)
\end{lemma}
\begin{proof}
We have
\begin{align*}
    \frac{d_{2i+2,j}}{d_{2i,j}}=\frac{\sum_{\bm{m}\in \mathcal{T}_{2i+2,j}}t(\bm{m})}{\sum_{\bm{m}\in \mathcal{T}_{2i,j}}t(\bm{m})}.
\end{align*}
Let $(m_0^{(k)},m_1^{(k)},\cdots,m_j^{(k)})$ denote the $k$-th element in $\mathcal{T}_{2i,j}$.
Because $m_0+m_1+\cdots+m_j=2i$ and $m_1+2m_2+\cdots+j\cdot m_j=j$, we have $m_0\geq 2i-j$.
Thus, using the definition of $t(\cdots)$ in Eq.~\eqref{eq.def_tm}, we have
\begin{align*}
    &\frac{t\left(m_0^{(k)}+2, m_1^{(k)},\cdots,m_j^{(k)}\right)}{t\left(m_0^{(k)}, m_1^{(k)},\cdots,m_j^{(k)}\right)}\\
    =&\frac{(2i+1)(2i+2)}{(m_0+1)(m_0+2)}a_0^2\\
    \leq &\frac{(2i+1)(2i+2)}{(2i-j+1)(2i-j+2)}a_0^2\text{ (because $m_0\geq 2i-j$)}\\
    \leq & \left(\frac{2i}{2i-j}a_0\right)^2\text{ (because $\frac{2i+1}{2i-j+1}\leq \frac{2i}{2i-j}$ and $\frac{2i+2}{2i-j+2}\leq \frac{2i}{2i-j}$)}.
\end{align*}
By Lemma~\ref{le.bijection} and $2i\geq j$, we thus have
\begin{align*}
    \frac{d_{2i+2,j}}{d_{2i,j}}=&\frac{t\left(m_0^{(1)}+2, m_1^{(1)},\cdots,m_j^{(1)}\right)+\cdots+t\left(m_0^{(|\mathcal{T}_{2i,j}|)}+2,m_1^{(|\mathcal{T}_{2i,j}|)},\cdots,m_j^{(|\mathcal{T}_{2i,j}|)}\right)}{t\left(m_0^{(1)}, m_1^{(1)},\cdots,m_j^{(1)}\right)+\cdots+t\left(m_0^{(|\mathcal{T}_{2i,j}|)},m_1^{(|\mathcal{T}_{2i,j}|)},\cdots,m_j^{(|\mathcal{T}_{2i,j}|)}\right)}\\
    \leq & \left(\frac{2i}{2i-j}a_0\right)^2.
\end{align*}
\end{proof}

\begin{lemma}\label{le.part_sum_d_not_too_small}
For any $j=1,2,\cdots$, we must have
\begin{align*}
    \sum_{i=\left\lceil\frac{j}{2s}\right\rceil}^{i^*}d_{2i,j}\geq \frac{1-\left(\frac{1+a_0}{2}\right)^2}{2-\left(\frac{1+a_0}{2}\right)^2}\sum_{i=\left\lceil\frac{j}{2s}\right\rceil}^\infty d_{2i,j},
\end{align*}
where
\begin{align*}
    i^*\defeq \left\lceil\frac{1+a_0}{2(1-a_0)}j\right\rceil.
\end{align*}
\end{lemma}
\begin{proof}
For all $i\geq i^*$, we have $2i\geq j$ and
\begin{align*}
    \frac{2i}{2i-j}a_0\leq \frac{2i^*}{2i^*-j}a_0\leq \frac{\frac{1+a_0}{1-a_0}j}{\frac{1+a_0}{1-a_0}j-j}a_0=\frac{1+a_0}{2} \text{ (by Lemma~\ref{le.temp_061401} and $2i\geq 2i^*\geq \frac{1+a_0}{1-a_0}j$)}.
\end{align*}
By Lemma~\ref{le.d_to_d_ratio}, we thus have
\begin{align*}
    \frac{d_{2i+2,j}}{d_{2i,j}}\leq \left(\frac{1+a_0}{2}\right)^2\text{ for all }i\geq i^*.
\end{align*}
Because $d_{2i,j}\geq 0$ for all $i$ and $j$, we have
\begin{align*}
    \sum_{i=i^*+1}^{\infty}d_{2i,j}\leq \sum_{i=i^*}^{\infty}d_{2i,j}\leq d_{2i^*,j}\sum_{k=0}^\infty \left(\frac{1+a_0}{2}\right)^{2k}=\frac{d_{2i^*,j}}{1-\left(\frac{1+a_0}{2}\right)^2}\leq \frac{1}{1-\left(\frac{1+a_0}{2}\right)^2}\sum_{i=\left\lceil\frac{j}{2s}\right\rceil}^{i^*} d_{2i,j}.
\end{align*}
Therefore, we have
\begin{align*}
    \left(1+\frac{1}{1-\left(\frac{1+a_0}{2}\right)^2}\right)\sum_{i=\left\lceil\frac{j}{2s}\right\rceil}^{i^*}d_{2i,j}\geq \sum_{i=0}^{i^*}d_{2i,j}+\sum_{i=i^*+1}^\infty d_{2i,j}=\sum_{i=\left\lceil\frac{j}{2s}\right\rceil}^\infty d_{2i,j},
\end{align*}
i.e.,
\begin{align*}
    \sum_{i=\left\lceil\frac{j}{2s}\right\rceil}^{i^*}d_{2i,j}\geq \frac{1-\left(\frac{1+a_0}{2}\right)^2}{2-\left(\frac{1+a_0}{2}\right)^2}\sum_{i=\left\lceil\frac{j}{2s}\right\rceil}^\infty d_{2i,j}.
\end{align*}
\end{proof}

Recall that $\taylorCoeff_m(x)$ denotes the Taylor coefficients of $x\frac{\pi-\arccos(x)}{2\pi}$.
The following lemma states that $\taylorCoeff_{2k}(x)$ is monotone decreasing with respect to $k$. Further, it estimates the decreasing speed. We draw the curve of $\taylorCoeff_k(x)$ with respect to $k$ in Fig.~\ref{fig.\taylorCoeff_k}.
\begin{lemma}\label{le.\taylorCoeff_k_estimate}
When $k\geq 1$, we have
\begin{align*}
    \taylorCoeff_{2i}(x) \geq \taylorCoeff_{2k}(x) \text{ for all }i\in \{1,2,\cdots,k\},
\end{align*}
and
\begin{align*}
    \frac{\taylorCoeff_{2(k+m)}(x)}{\taylorCoeff_{2k}(x)}\geq \frac{1}{\left(2+\frac{2m}{k}\right)^2}\text{ for all }m\in \mathds{Z}_{\geq 0}.
\end{align*}
\end{lemma}
\begin{proof}
By Lemma~\ref{le.taylor_of_kernel}, we have $\taylorCoeff_0(x)=0$, $\taylorCoeff_1(x)=\frac{1}{4}$, and
for all $k\geq 1$, we have
\begin{align}
    &\taylorCoeff_{2k}(x)=\frac{1}{2\pi}\frac{(2k-2)!}{((k-1)!)^2}\frac{4}{2k-1}\frac{1}{2^{2k}},\label{eq.temp_070101}\\
    &\taylorCoeff_{2k+1}(x)=0.\label{eq.temp_070102}
\end{align}
(We also plot the curve of $\taylorCoeff_k(x)$ with respect to $k$ in Fig.~\ref{fig.\taylorCoeff_k}, so that we can observe the general trend that is consistent with the statement of this lemma. We continue with the precise proof of this lemma below.)

By Eq.~\eqref{eq.temp_070101}, we have
\begin{align}\label{eq.temp_072301}
    \frac{\taylorCoeff_{2k+2}(x)}{\taylorCoeff_{2k}(x)}=\frac{2k(2k-1)(2k-1)}{k^2(2k+1)}\cdot\frac{1}{4}=\frac{(2k-1)^2}{2k\cdot (2k+1)}.
\end{align}
Because $\frac{(2k-1)^2}{2k\cdot (2k+1)}\leq 1$, we know that $\taylorCoeff_{2k}(x)$ is monotone decreasing with respect to $k$. Therefore, we have 
\begin{align*}
    \taylorCoeff_{2i}(x) \geq \taylorCoeff_{2k}(x) \text{ for all }i\in \{1,2,\cdots,k\}.
\end{align*}
Iterating Eq.~\eqref{eq.temp_072301}, we have
\begin{align*}
    \frac{\taylorCoeff_{2(k+m)}(x)}{\taylorCoeff_{2k}(x)}=&\prod_{i=0}^{m-1} \frac{(2k+2i-1)^2}{(2k+2i)(2k+2i+1)}\\
    \geq& \prod_{i=0}^{m-1} \frac{(2k+2i-1)^2}{(2k+2i+1)^2}\\
    =&\frac{(2k-1)^2}{(2k+2m-1)^2}\\
    \geq &\frac{k^2}{(2k+2m)^2}\text{ (because $2k-1\geq k$ due to $k\geq 1$)}\\
    =&\frac{1}{\left(2+\frac{2m}{k}\right)^2}.
\end{align*}
\end{proof}

\begin{lemma}\label{le.not_very_small}
Under Condition~\ref{condition.finite_poly}, for any $j=1,2,\cdots$, we must have
\begin{align*}
    \frac{\taylorCoeff_j\left(K(x)\right)}{\taylorCoeff_{2\left\lceil j/2\right\rceil}(x)}\geq \left(\frac{1}{2}\cdot\frac{1-a_0}{3-a_0}\right)^2 \frac{1-\left(\frac{1+a_0}{2}\right)^2}{2-\left(\frac{1+a_0}{2}\right)^2}e_j.
\end{align*}
\end{lemma}
\begin{proof}
Consider $i^*$ defined in Lemma~\ref{le.part_sum_d_not_too_small}, i.e., $i^*=\myCeil{\frac{1+a_0}{2(1-a_0)}j}$. Let $k=\left\lceil\frac{j}{2}\right\rceil$ and $m=i^*-k$.
% Thus, we have
% \begin{align*}
%     m\geq\left\lceil\frac{1+a_0}{1-a_0}k\right\rceil-k= \left\lceil\frac{2a_0}{1-a_0}k\right\rceil.
% \end{align*}
% We then have
% \begin{align*}
%     \frac{2m}{k}\leq \frac{\frac{2a_0}{1-a_0}k+k}{k}=\frac{1+a_0}{1-a_0}. 
% \end{align*}
We have
\begin{align*}
    \frac{2m}{k}=\frac{2i^*}{k}-2=&\frac{2\myCeil{\frac{1+a_0}{2(1-a_0)}j}}{\myCeil{\frac{j}{2}}}-2\\
    \leq &\frac{2\left(\frac{1+a_0}{2(1-a_0)}j+1\right)}{\frac{j}{2}}-2\quad \text{ (because $\myCeil{\alpha}\in [\alpha, \alpha+1]$)}\\
    =&\frac{2(1+a_0)}{1-a_0}+\frac{4}{j}-2\\
    \leq &\frac{2(1+a_0)}{1-a_0}+4-2\quad\text{ (because $j\geq 1$)}\\
    =&\frac{4}{1-a_0}.
\end{align*}
By Lemma~\ref{le.\taylorCoeff_k_estimate}, we then have
\begin{align*}
    \frac{\taylorCoeff_{2i^*}(x)}{\taylorCoeff_{2\left\lceil j/2\right\rceil}(x)}=\frac{\taylorCoeff_{2(k+m)}(x)}{\taylorCoeff_{2k}(x)}\geq & \frac{1}{\left(2+\frac{2m}{k}\right)^2}\geq \frac{1}{\left(2+\frac{4}{1-a_0}\right)^2}=\left(\frac{1}{2}\cdot\frac{1-a_0}{3-a_0}\right)^2.
\end{align*}
By the first part of Lemma~\ref{le.\taylorCoeff_k_estimate}, $\taylorCoeff_{2i}\geq \taylorCoeff_{2i^*}$ for all $i=1,2,\cdots, i^*$. We thus have
\begin{align}\label{eq.temp_012405}
    \frac{\taylorCoeff_{2i}(x)}{\taylorCoeff_{2\left\lceil j/2\right\rceil}(x)}\geq \left(\frac{1}{2}\cdot \frac{1-a_0}{3-a_0}\right)^2\text{ for all }i=1,2,\cdots,i^*.
\end{align}

Notice that
\begin{align*}
    &K(x)\frac{\pi- \arccos\left(K(x)\right)}{2\pi}\\
    =&\taylorCoeff_1(x)K(x)+\sum_{i=1}^\infty \taylorCoeff_{2i}(x) \left(K(x)\right)^{2i}\\
    =&\taylorCoeff_1(x)K(x)+\sum_{i=1}^\infty \taylorCoeff_{2i}(x) \sum_{j=0}^\infty d_{2i,j}x^j\text{ (by Lemma~\ref{le.expansion})}\\
    =&\frac{1}{4}K(x)+\sum_{j=0}^\infty\left( \sum_{i=\left\lceil\frac{j}{2s}\right\rceil}^\infty \taylorCoeff_{2i}(x)\cdot d_{2i,j}\right) x^j \text{ (since $d_{2i,j}=0$ when $j>2i\cdot s$)}.
\end{align*}
Therefore,
\begin{align}\label{eq.represent_\taylorCoeff_by_cd}
    \taylorCoeff_j\left(K(x)\right)=\frac{a_j}{4}+\sum_{i=\left\lceil\frac{j}{2s}\right\rceil}^\infty \taylorCoeff_{2i}(x)\cdot d_{2i,j}\text{ for all }j\in \mathds{Z}_{\geq 0}.
\end{align}

By Eq.~\eqref{eq.represent_\taylorCoeff_by_cd}, we thus have
\begin{align*}
    \frac{\taylorCoeff_j\left(K(x)\right)}{\taylorCoeff_{2\left\lceil j/2\right\rceil}(x)}\geq & \frac{1}{\taylorCoeff_{2\left\lceil j/2\right\rceil}(x)}\sum_{i=\left\lceil\frac{j}{2s}\right\rceil}^\infty \taylorCoeff_{2i}(x) d_{2i,j}\text{ (notice that $a_i\geq 0$ for all $i\in \mathds{Z}_{\geq 0}$ by Condition~\ref{condition.finite_poly})}\\
    \geq&\frac{1}{\taylorCoeff_{2\left\lceil j/2\right\rceil}(x)}\sum_{i=\left\lceil\frac{j}{2s}\right\rceil}^{i^*} \taylorCoeff_{2i}(x) d_{2i,j}\\
    \geq & \left(\frac{1}{2}\cdot\frac{1-a_0}{3-a_0}\right)^2 \sum_{i=\left\lceil\frac{j}{2s}\right\rceil}^{i^*}  d_{2i,j}\text{ (by Eq.~\eqref{eq.temp_012405})}\\
    \geq & \left(\frac{1}{2}\cdot\frac{1-a_0}{3-a_0}\right)^2 \frac{1-\left(\frac{1+a_0}{2}\right)^2}{2-\left(\frac{1+a_0}{2}\right)^2}\sum_{i=\left\lceil\frac{j}{2s}\right\rceil}^\infty  d_{2i,j}\text{ (by Lemma~\ref{le.part_sum_d_not_too_small})}\\
    =&\left(\frac{1}{2}\cdot\frac{1-a_0}{3-a_0}\right)^2 \frac{1-\left(\frac{1+a_0}{2}\right)^2}{2-\left(\frac{1+a_0}{2}\right)^2}e_j\text{ (by Eq.~\eqref{def.coeff_e_shown_by_d})}.
\end{align*}
\end{proof}

\begin{figure}[t!]
\centering
\includegraphics[width=3.5in]{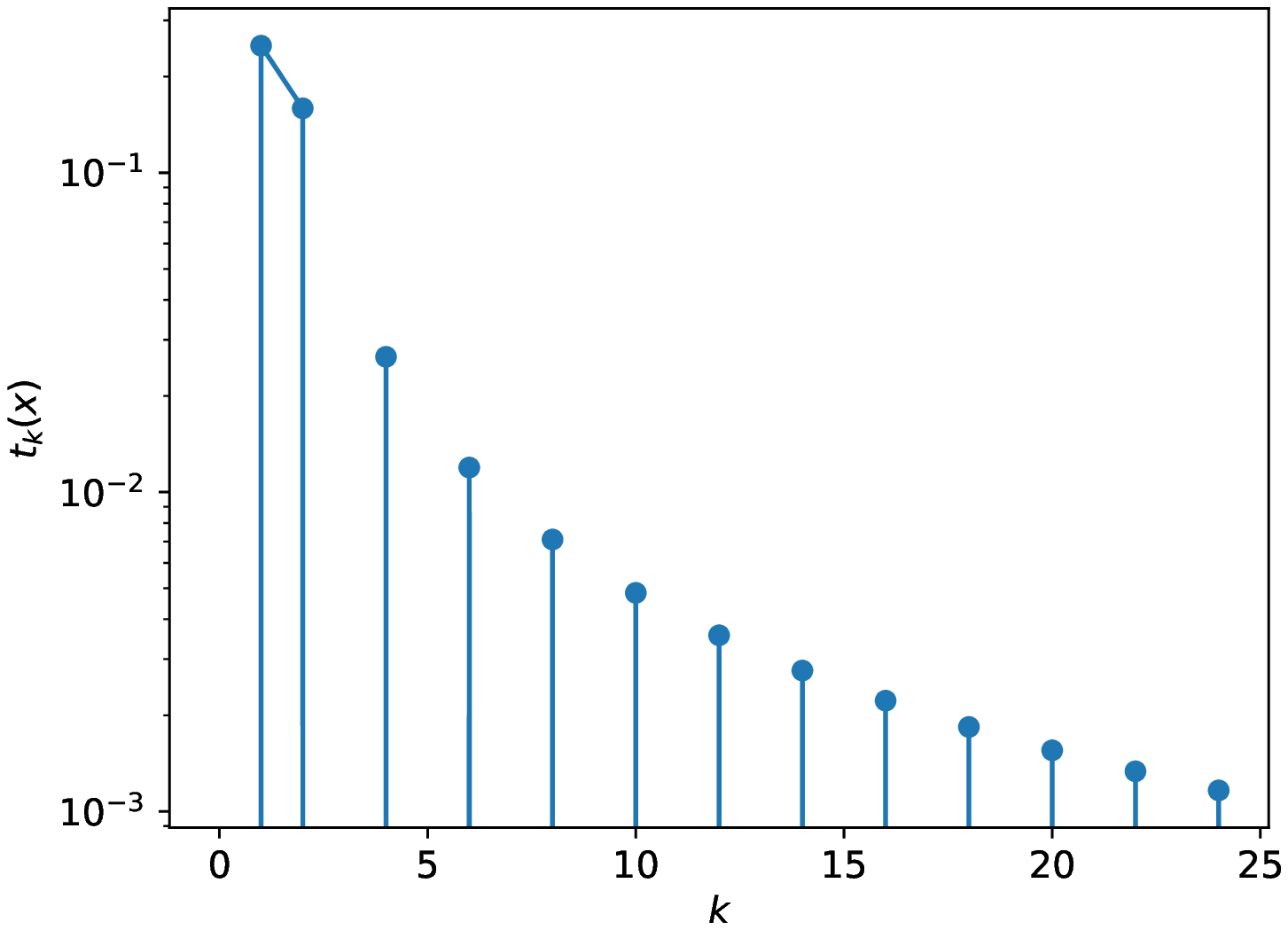}
\caption{The curve of $\taylorCoeff_k(x)$ with respect to $k$. Notice that $\taylorCoeff_k(x)=0$ when $k=0,3,5,7,\cdots.$}\label{fig.\taylorCoeff_k}
\end{figure}

By Lemma~\ref{le.not_very_small} and Lemma~\ref{le.bounded_e}, we can conclude the lower bound in Lemma~\ref{le.finite_poly_conclusion}. Next we will prove the upper bound in Lemma~\ref{le.finite_poly_conclusion}.

\begin{lemma}\label{le.not_very_large}
Under Condition~\ref{condition.finite_poly}, for any $j=1,2,\cdots$, we must have
\begin{align*}
    \frac{\taylorCoeff_j\left(K(x)\right)-\frac{a_j}{4}}{\taylorCoeff_{2\left\lceil j/2\right\rceil}(x)}\leq (2+2s)^2 e_j.
\end{align*}
\end{lemma}
\begin{proof}
Let $k=\left\lceil\frac{j}{2s}\right\rceil$ and $m=\left\lceil\frac{j}{2}\right\rceil-\left\lceil\frac{j}{2s}\right\rceil$. Thus, we have
\begin{align*}
    &\frac{m}{k}=\frac{\left\lceil\frac{j}{2}\right\rceil}{\left\lceil\frac{j}{2s}\right\rceil}-1\leq \frac{\frac{j}{2}+1}{\left\lceil\frac{j}{2s}\right\rceil}-1= \frac{\frac{j}{2}}{\left\lceil\frac{j}{2s}\right\rceil}+\left(\frac{1}{\left\lceil\frac{j}{2s}\right\rceil}-1\right)\leq s \\
    &\quad \text{ (noting that $\myCeil{\frac{j}{2s}}\geq 1$ since $j\geq 1$)}.
\end{align*}
By Lemma~\ref{le.\taylorCoeff_k_estimate}, we then have
\begin{align*}
    \frac{\taylorCoeff_{2\left\lceil j/2\right\rceil}(x)}{\taylorCoeff_{2\left\lceil j/(2s)\right\rceil}(x)}=\frac{\taylorCoeff_{2(k+m)}(x)}{\taylorCoeff_{2k}(x)}\geq & \frac{1}{\left(2+\frac{2m}{k}\right)^2}\text{ (by Lemma~\ref{le.\taylorCoeff_k_estimate})}\\
    \geq & \frac{1}{(2+2s)^2}.
\end{align*}
By the first part of Lemma~\ref{le.\taylorCoeff_k_estimate}, we further have
\begin{align*}
    \taylorCoeff_{2i}(x)\leq (2+2s)^2 \taylorCoeff_{2\left\lceil j/2\right\rceil}(x) \text{ for all } 2i\geq 2\left\lceil\frac{j}{2s}\right\rceil.
\end{align*}
Thus, we have
\begin{align*}
    \taylorCoeff_j\left(K(x)\right)-\frac{a_j}{4}=&\sum_{i=\left\lceil\frac{j}{2s}\right\rceil}^\infty \taylorCoeff_{2i}(x)\cdot d_{2i,j}\text{ (by Eq.~\eqref{eq.represent_\taylorCoeff_by_cd})}\\
    \leq & (2+2s)^2 \taylorCoeff_{2\left\lceil j/2\right\rceil}(x) \sum_{i=\left\lceil\frac{j}{2s}\right\rceil}^\infty  d_{2i,j}\\
    =&(2+2s)^2 \taylorCoeff_{2\left\lceil j/2\right\rceil}(x) e_j\text{ (by Eq.~\eqref{def.coeff_e_shown_by_d})}.
\end{align*}
The result of the lemma thus follows.
\end{proof}

Combining Lemma~\ref{le.not_very_small}, Lemma~\ref{le.not_very_large}, and Lemma~\ref{le.bounded_e}, the result of Lemma~\ref{le.finite_poly_conclusion} thus follows.

\subsection{Expansion for an infinite-degree polynomial}\label{app.inf_Taylor_expansion}
We now return to the proof of the third statement of Lemma~\ref{le.coefficient_compare}, where the polynomial $K(x)$ has infinite terms. We inherit notations $a_i$, $b_i$, $\taylorCoeff_i$, $d_i$, $e_i$ for the finite polynomial case in Appendix~\ref{app.finite_Taylor_expansion}. Further, we introduce some additional notations as follows.

Define 
\begin{align}\label{def.coeff_tilde_b}
    \tilde{b}_i\defeq \frac{b_i}{1-b_0},\text{ for all } i\in \mathds{Z}_{\geq 0},
\end{align}
and
\begin{align}\label{def.T_k}
    &T_k\defeq \sum_{i=k+1}^\infty \tilde{b}_i,\quad k\in \mathds{Z}_{\geq 0}.
\end{align}
By Condition~\ref{condition.finite_poly}, Eq.~\eqref{def.coeff_b}, and Eq.~\eqref{def.coeff_tilde_b}, we have
\begin{align}\label{eq.temp_080301}
    T_k\geq 0\text{ for all }k\in \mathds{Z}_{\geq 0}.
\end{align}
Define for any $k\in \{1,2,\cdots\}$,
\begin{align}\label{def.L_k}
    L_k\defeq \min\ \{e_j\ |\ j=0,1,\cdots,k-1\}\cup \{1\} - \sum_{i=k}^\infty T_i.
\end{align}

\begin{lemma}\label{le.lower_bound_of_e_k}
For any $i\in \mathds{Z}_{\geq 0}$ and any $k\in \{1,2,\cdots\}$, we must have $e_i\geq L_k$. Notice that the indices $i$ and $k$ are not required to be equal.
\end{lemma}
We prove Lemma~\ref{le.lower_bound_of_e_k} in Appendix~\ref{app.proof_lower_bound_of_e_k}.

In order to prove the third statement of Lemma~\ref{le.coefficient_compare}, by Lemma~\ref{le.coeff_same_taylor}, we only need to lower bound $\frac{\taylorCoeff_j\left(K(x)\right)}{\taylorCoeff_{2\left\lceil j/2\right\rceil}(x)}$. Notice that Lemma~\ref{le.not_very_small} still holds when $s\to\infty$ (the proof of it will be exactly the same after replacing $\myCeil{j/2s}$ by zero). Therefore, we only need to prove that $e_i$ is lower bounded by a positive constant. By Lemma~\ref{le.lower_bound_of_e_k}, if we can find $L_k>0$ for some $k$, then we are done. By Eq.~\eqref{def.L_k}, In order to calculate the exact value of $L_k$, we need to find a way to calculate the exact value of $\sum_{k=0}^\infty T_k$, which is provided by the following lemma.

\begin{lemma}\label{le.estimate_T_k}
% The value of $\sum\limits_{k=0}^\infty T_k$ must exists (i.e., the limit $\lim\limits_{l\to\infty}\sum\limits_{k=0}^l T_k$ converges) and 
$\sum\limits_{k=0}^\infty T_k=\frac{2}{1-a_0^2}\frac{\partial K(x)}{\partial x}\big\vert_{x=1}$.
\end{lemma}
\begin{proof}
% In order to show that $\lim_{l\to\infty}\sum_{k=0}^l T_k$ converges, we only need to prove that $\lim_{k\to\infty}T_k=0$. Indeed, we have
% \begin{align*}
%     T_k=&\frac{1}{1-b_0}\sum_{i=k+1}^\infty b_i\text{ (by Eq.~\eqref{def.T_k} and Lemma~\ref{le.e_is_bounded})}\\
%     =&\frac{1}{1-b_0}\left(1-\sum_{l=0}^k\sum_{i+j=l}a_i a_j\right)\text{ (by Eq.~\eqref{def.coeff_b})}\\
%     =&\frac{1}{1-b_0}\sum_{i+j>l }a_i a_j\text{ (since $\sum_{i,j}a_i a_j=\left(\sum_{i=0}^\infty a_i\right)^2=1$ by Condition~\ref{condition.finite_poly})}\\
%     \leq & \frac{1}{1-b_0}\sum_{i>\frac{l}{2}\text{ or }j>\frac{l}{2}}a_i a_j\text{ (either $i> \frac{k}{2}$ or $j>\frac{k}{2}$ or both because $i+j>l$)}\\
%     \leq&\frac{2}{1-b_0}\sum_{i>\frac{
%     k}{2}}a_i\sum_{j=0}^\infty a_j\\
%     =&\frac{2}{1-b_0}\sum_{i>\frac{k}{2}}a_i\text{ (since $\sum_{j=0}^\infty a_j=1$ by Condition~\ref{condition.finite_poly})}.
% \end{align*}
% Thus, the value $\sum_{k=0}^\infty T_k$ must exists. Next we calculate its value.
We have
\begin{align*}
    \sum_{k=0}^\infty T_k = & \sum_{k=0}^\infty \sum_{i=k+1}^\infty \tilde{b}_i\text{ (by Eq.~\eqref{def.T_k})}\\
    =&\sum_{i=1}^\infty i\cdot \tilde{b}_i\\
    =&\frac{1}{1-b_0}\sum_{i=1}^\infty i\cdot b_i\text{ (by Eq.~\eqref{def.coeff_tilde_b})}\\
    =&\frac{1}{1-a_0^2}\sum_{i=1}^\infty i\cdot b_i\text{ (notice that $b_0=a_0^2$ by Eq.~\eqref{def.coeff_b})}\\
    =&\frac{1}{1-a_0^2}\frac{\partial \left(\sum_{i=0}^\infty b_i x^i\right)}{\partial x}\bigg|_{x=1}\\
    =&\frac{1}{1-a_0^2}\frac{\partial\left(K(x)\right)^2}{\partial x}\bigg|_{x=1}\text{ (by Eq.~\eqref{def.coeff_b})}\\
    =& \frac{2K(1)}{1-a_0^2}\frac{\partial K(x)}{\partial x}\bigg|_{x=1}\\
    =&\frac{2}{1-a_0^2}\frac{\partial K(x)}{\partial x}\bigg|_{x=1}\text{ (by Condition~\ref{condition.finite_poly})}.
\end{align*}
\end{proof}

Now we consider the case  of 3-layer without bias, i.e., the case when the polynomial $K(x)$ is $2d\cdot \KRF(x)= \frac{\sqrt{1-x^2}+(\pi-\arccos(x))x}{\pi}$. After calculation (details in Appendix~\ref{app.calculate_L}), we have $L_3\approx 0.069>0$, which completes the proof of Proposition~\ref{prop.compare_efficient_and_set}.

% \begin{figure}[ht!]
% \centering
% \includegraphics[width=3.5in]{figs/e_k.eps}
% \caption{\TODO Quantities when $K(x)=\KRF(x)$.}\label{fig.converge_e}
% \end{figure}

\subsubsection{Proof of Lemma~\ref{le.lower_bound_of_e_k}}\label{app.proof_lower_bound_of_e_k}

We first prove some useful lemmas.

\begin{lemma}\label{le.e_is_bounded}
(i) $\sum_{i=1}^\infty \tilde{b}_i=1$. (ii) $e_k=\sum_{i=0}^{k-1} \tilde{b}_{k-i}e_i$ for all $k\in \{1, 2, \cdots\}$. (iii) $e_k\leq e_0= \frac{1}{1-b_0}$ for all $k\in \mathds{Z}_{\geq 0}$.
\end{lemma}
\begin{proof}
Note that
\begin{align*}
    \sum_{i=1}^\infty \tilde{b}_i=&\frac{\sum_{i=1}^\infty b_i}{1-b_0}\text{ (by Eq.~\eqref{def.coeff_tilde_b})}\\
    =&1\quad \text{ (since $\sum_{i=1}^\infty b_i=1-b_0$ because $b_0+\sum_{i=1}^\infty b_i = (K(1))^2=1$)}.
\end{align*}
Because
\begin{align*}
    \frac{1}{1-\left(K(x)\right)^2}=1+\left(K(x)\right)^2\cdot \frac{1}{1-\left(K(x)\right)^2},
\end{align*}
we have
\begin{align*}
    \sum_{j=0}^\infty e_j x^j=&1+\left(\sum_{i=0}^\infty a_i x^i\right)^2\cdot \left(\sum_{j=0}^\infty e_j x^j\right)=1+\left(\sum_{i=0}^{\infty}b_i x^i\right)\cdot \left(\sum_{j=0}^\infty e_j x^j\right).
\end{align*}
Comparing the coefficient of $x^{j}$ on both sides, we have
\begin{align*}
    &e_0=1+b_0 e_0,\\
    &e_1=b_1 e_0+b_0 e_1,\\
    &e_2=b_2 e_0+b_1e_1+b_0e_2,\\
    &\quad \vdots
\end{align*}
We thus have
\begin{align}\label{eq.temp_012406}
    \begin{cases}
    &e_0=\frac{1}{1-b_0},\\
    &e_1=\tilde{b}_1 e_0,\\
    &e_2=\tilde{b}_2 e_0+\tilde{b}_1 e_1=\left(\tilde{b}_2+\tilde{b}_1^2\right)e_0,\\
    &e_3=\tilde{b}_3 e_0+\tilde{b}_2 e_1+\tilde{b}_1 e_2=\left(\tilde{b}_3+2\tilde{b}_2\tilde{b}_1+\tilde{b}_1^3\right)e_0,\\
    &e_4=\tilde{b}_4 e_0+\tilde{b}_3 e_1+\tilde{b}_2 e_2+\tilde{b}_1 e_3=\left(\tilde{b}_4+2\tilde{b}_1\tilde{b}_3+\tilde{b}_2^2+3\tilde{b}_1^2\tilde{b}_2+\tilde{b}_1^4\right)e_0,\\
    &\quad \vdots\\
    &e_k=\sum_{i=0}^{k-1} \tilde{b}_{k-i}e_i,\\
    % =\left(\sum_{m_1+2m_2+\cdots+km_k=k}(m_1,m_2,\cdots,m_k)!\cdot \tilde{b}_1^{m_1} \tilde{b}_2^{m_2}\cdots \tilde{b}_k^{m_k}\right)e_0,\\
    &\quad \vdots
    \end{cases}
\end{align}
% Because $\sum_{i=0}^\infty b_0=1$, we then have $1-b_0=b_1+b_2+\cdots$, i.e., $\sum_{i=1}^\infty\tilde{b}_i=1$. 
We now prove that $e_k\leq e_0$ for all $k\in\mathds{Z}_{\geq 0}$ by mathematical induction. Suppose that for all $i\leq k$, we already have $e_i\leq e_0$ (which is obviously true when $k=0$). Thus, we have
\begin{align*}
    e_{k+1}=\sum_{i=0}^k \tilde{b}_{k+1-i}e_i \leq e_0\sum_{i=0}^{k} \tilde{b}_{k+1-i}\leq e_0\sum_{i=0}^\infty \tilde{b}_i=e_0.
\end{align*}
By mathematical induction, we thus have $e_i\leq e_0$ for all $i\in \mathds{Z}_{\geq 0}$.
\end{proof}

For any given real number sequence $\bm{\Delta}\defeq(\Delta_1,\Delta_2,\cdots)$, we define a sequence $(e^{\bm{\Delta}}_0, e^{\bm{\Delta}}_1, e^{\bm{\Delta}}_2,\cdots)$ by
\begin{align}\label{def.tilde_e}
    e^{\bm{\Delta}}_0\defeq e_0,\quad e^{\bm{\Delta}}_k\defeq \Delta_k+\sum_{i=0}^{k-1}\tilde{b}_{k-i}e^{\bm{\Delta}}_i,\ k\in\{1,2,\cdots\}.
\end{align}

\begin{lemma}\label{le.diff_e_e}
For any $k\in \{1,2,\cdots\}$ and any $\bm{\Delta}$, we must have
\begin{align}\label{eq.temp_073001}
    e^{\bm{\Delta}}_k-e_k=\sum_{i=1}^k \Delta_i\cdot \frac{e_{k-i}}{e_0}.
\end{align}
\end{lemma}
\begin{proof}
We prove Eq.~\eqref{eq.temp_073001} by mathematical induction. When $k=1$, we have $e^{\bm{\Delta}}_1=\Delta_1+\tilde{b}_1 e_0$ (by Eq.~\eqref{def.tilde_e}) and $e_1=\tilde{b}_1 e_0$ (by Lemma~\ref{le.e_is_bounded}). Thus, Eq.~\eqref{eq.temp_072301} holds when $k= 1$. Suppose that for all $k\in \{1,2,\cdots,l\}$, Eq.~\eqref{eq.temp_073001} holds. We thus have
\begin{align}
    e^{\bm{\Delta}}_{l+1}-e_{l+1}=&\Delta_{l+1}+\sum_{i=0}^{l} \tilde{b}_{l+1-i} (e^{\bm{\Delta}}_{i}-e_i)\text{ (by Eq.~\eqref{def.tilde_e} and Lemma~\ref{le.e_is_bounded})}\nonumber\\
    =&\Delta_{l+1}+\sum_{i=1}^{l} \tilde{b}_{l+1-i} (e^{\bm{\Delta}}_{i}-e_i)\text{ (notice that $e^{\bm{\Delta}}_0=e_0$)}\nonumber\\
    =& \Delta_{l+1}+\sum_{i=1}^{l} \tilde{b}_{l+1-i} \sum_{j=1}^i \Delta_j \frac{e_{i-j}}{e_0}\text{ (applying Eq.~\eqref{eq.temp_073001} by induction hypothesis)}.\label{eq.temp_012407}
\end{align}
Notice that
\begin{align*}
    \sum_{i=1}^{l}  \sum_{j=1}^i \tilde{b}_{l+1-i} \cdot \Delta_j \cdot e_{i-j}=&\sum_{j=1}^l \Delta_j \sum_{i=j}^l \tilde{b}_{l+1-i} e_{i-j} \text{ (by re-organizing terms)}\\
    =&\sum_{j=1}^l \Delta_j \sum_{i=0}^{l-j} \tilde{b}_{l+1-i-j} e_{i}\text{ (replacing $i-j$ by $i$)}.
\end{align*}
Plugging the above equation into Eq.~\eqref{eq.temp_012407}, we then have
\begin{align*}
    e^{\bm{\Delta}}_{l+1}-e_{l+1}=& \Delta_{l+1}+\frac{1}{e_0}\sum_{j=1}^l \Delta_j \sum_{i=0}^{l-j} \tilde{b}_{l-j+1-i}e_i\\
    =&\Delta_{l+1}+\frac{1}{e_0}\sum_{j=1}^l \Delta_j\cdot e_{l+1-j} \text{ (by Lemma~\ref{le.e_is_bounded})}\\
    = & \frac{1}{e_0} \sum_{j=1}^{l+1}\Delta_j\cdot e_{l+1-j},
\end{align*}
i.e., Eq.~\eqref{eq.temp_073001} also holds for $k=l+1$. Thus, the mathematical induction is completed and the result of this lemma thus follows.
\end{proof}

Now we are ready to prove Lemma~\ref{le.lower_bound_of_e_k}.
\begin{proof}[Proof of Lemma~\ref{le.lower_bound_of_e_k}]
Let
\begin{align*}
    \Delta_i=\begin{cases}
    0, &\text{ if }i<k,\\
    T_i, &\text{ if }i \geq k.
    \end{cases}
\end{align*}
We first prove by mathematical induction that
\begin{align}\label{eq.temp_080101}
    e^{\bm{\Delta}}_i\geq \min \{e_j\ |\ j=0,1,\cdots,k-1\}\cup \{1\}\text{ for all }i\in \mathds{Z}_{\geq 0}.
\end{align}
Towards this end, note that because $\Delta_i=0$ for all $i< k$, we know from Lemma~\ref{le.diff_e_e} that $e^{\bm{\Delta}}_i=e_i$. Hence, Eq.~\eqref{eq.temp_080101} trivially holds for all $i\in \{0,1,\cdots,k-1\}$. Suppose that Eq.~\eqref{eq.temp_080101} holds for all $i\leq l\in \mathds{Z}_{\geq 0}$, where $l\geq k-1$ denotes the index of the induction hypothesis. In order to finish the mathematical induction, we only need to prove that Eq.~\eqref{eq.temp_080101} holds for $i=l+1$. To this end, we have
\begin{align*}
    e^{\bm{\Delta}}_{l+1}=&T_{l+1} +\sum_{i=0}^{l}\tilde{b}_{l+1-i}e^{\bm{\Delta}}_i\text{ (by Eq.~\eqref{def.tilde_e})}\\
    =&\sum_{i=l+2}^\infty \tilde{b}_i + \sum_{i=1}^{l+1}\tilde{b}_i e^{\bm{\Delta}}_{l+1-i}\text{ (by Eq.~\eqref{def.T_k})}\\
    \geq &  \left(\min \{e^{\bm{\Delta}}_j\ |\ j=0,1,\cdots,l\}\cup \{1\}\right)\cdot \sum_{i=1}^\infty \tilde{b}_i\\
    =&\min \{e^{\bm{\Delta}}_j\ |\ j=0,1,\cdots,l\}\cup \{1\}\text{ (by Lemma~\ref{le.e_is_bounded})}\\
    \geq  &\min \{e_j\ |\ j=0,1,\cdots,k-1\}\cup \{1\}\text{ (by induction hypothesis)}.
\end{align*}
Thus, Eq.~\eqref{eq.temp_080101} holds by mathematical induction.
We thus have
\begin{align*}
    e_i=& e^{\bm{\Delta}}_i-\sum_{j=1}^i \Delta_j \cdot \frac{e_{i-j}}{e_0}\text{ (by Lemma~\ref{le.diff_e_e})}\\
    \geq & e^{\bm{\Delta}}_i-\sum_{j=1}^i\Delta_j\text{ (since $e_{i-j}\leq e_0$ by Lemma~\ref{le.e_is_bounded})}\\
    =& e^{\bm{\Delta}}_i-\sum_{j=k}^i T_i\\
    \geq & \min \{e_j\ |\ j=0,1,\cdots,k-1\}\cup \{1\} -\sum_{j=k}^\infty T_j\text{ (by Eq.~\eqref{eq.temp_080101} and Eq.~\eqref{eq.temp_080301})}\\
    =&L_k.
\end{align*}
The result of this lemma thus follows.
\end{proof}

\subsubsection{Calculate \texorpdfstring{$L_3$}{L3} for 3-layer without bias}\label{app.calculate_L}

Coefficients of Taylor expansion of $K(x)=2d\cdot \KRF(x)= \frac{\sqrt{1-x^2}+(\pi-\arccos(x))x}{\pi}$ can be derived from Lemma~\ref{le.taylor_of_kernel}, i.e.,
\begin{align}\label{eq.temp_080103}
    K(x)=\frac{1}{\pi}\left(1+\frac{\pi}{2}x+\sum_{k=0}^\infty\frac{2(2k)!}{(k+1)(2k+1)(k!)^2}\left(\frac{x}{2}\right)^{2k+2}\right).
\end{align}
% Then, we calculate the value of $L_k$ by Eq.~\eqref{def.L_k}, Lemma~\ref{le.estimate_T_k}, and Eq.~\eqref{eq.temp_080103}. We have
By Eq.~\eqref{eq.temp_080103}, We can calculate values of $a_i$, $b_i$, and $\tilde{b}_i$ for $i=0,1,2$ by their definitions.
\begin{align*}
    & a_0=\frac{1}{\pi},\ a_1=\frac{1}{2},\ a_2=\frac{1}{2\pi},\\
    & b_0=a_0^2=\frac{1}{\pi^2},\ b_1=2 a_0 a_1= \frac{1}{\pi},\ b_2=2a_2 a_0+a_1^2=\frac{1}{\pi^2}+\frac{1}{4},\\
    & \tilde{b}_0=\frac{1}{\pi^2-1},\ \tilde{b}_1=\frac{\pi}{\pi^2-1},\ \tilde{b}_2=\frac{1}{\pi^2-1}+\frac{\pi^2}{4(\pi^2-1)}.\\
\end{align*}
Then, we calculate the values of $e_i$ by Eq.~\eqref{eq.temp_012406}.
\begin{align*}
    & e_0=\frac{\pi^2}{\pi^2-1}\approx 1.11,\ e_1=\frac{\pi^3}{(\pi^2-1)^2}\approx 0.39,\ e_2=\frac{(2\pi^2-1)\pi^2}{(\pi^2-1)^3}+\frac{\pi^4}{4(\pi^2-1)^2}\approx 0.57.
\end{align*}
Next, we calculate $\sum_{i=0}^\infty T_i$ by Lemma~\ref{le.estimate_T_k}.
\begin{align*}
    \sum_{i=0}^\infty T_i =&\frac{2}{1-a_0^2} \frac{\partial \frac{\sqrt{1-x^2}+(\pi-\arccos(x))x}{\pi}}{\partial x}\bigg\vert_{x=1}\\
    =&\frac{2}{1-a_0^2} \frac{1}{\pi} \left(-\frac{x}{\sqrt{1-x^2}}+(\pi-\arccos(x))+\frac{x}{\sqrt{1-x^2}}\right)\\
    &\text{ (notice that $\frac{\partial \sqrt{1-x^2}}{\partial x}=-\frac{x}{\sqrt{1-x^2}}$ and  $\frac{\partial}{\partial x}\arccos(x)=-\frac{1}{\sqrt{1-x^2}}$)}\\
    =&\frac{2}{1-a_0^2}\frac{\pi-\arccos(x)}{\pi}\big\vert_{x=1}=\frac{2\pi^2}{\pi^2-1}.
\end{align*}
By Eq.~\eqref{def.T_k} and Lemma~\ref{le.e_is_bounded}(i), we thus have $T_j=\sum_{i=1}^\infty \tilde{b}_i-\sum_{i=1}^{j}\tilde{b}_i=1-\sum_{i=1}^{j}\tilde{b}_i$. Therefore, we have
\begin{align*}
    & T_0 = 1,\ T_1 = 1-\frac{\pi}{\pi^2-1},\ T_2 = 1-\frac{\pi}{\pi^2-1} - \left(\frac{1}{\pi^2-1}+\frac{\pi^2}{4(\pi^2-1)}\right).
\end{align*}
Now we are ready to calculate $L_3$ by Eq.~\eqref{def.L_k}.
\begin{align*}
    L_3=&e_1 + T_0+T_1+T_2 - \sum_{i=0}^\infty T_i\\
    =&\frac{\pi^3}{(\pi^2-1)^2}+3 - \frac{2\pi}{\pi^2-1}-\left(\frac{1}{\pi^2-1}+\frac{\pi^2}{4(\pi^2-1)}\right) - \frac{2\pi^2}{\pi^2-1}\\
    \approx & 0.069.
\end{align*}

\end{document}